\documentclass{article}

\usepackage{
    amsmath,
    amsthm,
    amssymb,
	wrapfig,
	cases,
	mathtools,
	thmtools,
	array,
	bbm,
	bm,
	subcaption,
	makecell,
	esvect,
	mathrsfs,
	breqn,
	booktabs,
	upgreek,
	changepage,
	dsfont,
	fixme,
	listings,
	multirow,
	xargs,
	xstring,
	xparse,
	multicol,
	graphicx,
	float,
	color,
	algpseudocode,
	enumitem,
	indentfirst,
	accents,
	ifthen,
	wasysym,
	tikz,
	pgf,
	pgfplots,
	lmodern,
	titlesec
	}

\usepackage[utf8]{inputenc}
\usepackage[linesnumbered,boxruled,lined,noend]{algorithm2e}
\SetAlFnt{\footnotesize}
\makeatletter
\newcommand{\nosemic}{\renewcommand{\@endalgocfline}{\relax}}% Drop semi-colon ;
\newcommand{\dosemic}{\renewcommand{\@endalgocfline}{\algocf@endline}}% Reinstate semi-colon ;
\newcommand{\pushline}{\Indp}% Indent
\newcommand{\popline}{\Indm\dosemic}% Undent
\let\oldnl\nl% Store \nl in \oldnl
\newcommand{\nonl}{\renewcommand{\nl}{\let\nl\oldnl}}% Remove line number for one line
\makeatother

\usepackage[hidelinks]{hyperref}

\usepackage{natbib}
\setcitestyle{authoryear,round,citesep={;},aysep={,},yysep={;}}
\usepackage{paralist}
\renewenvironment{enumerate}[1]{\begin{compactenum}#1}{\end{compactenum}}

\usetikzlibrary{shapes,arrows,trees,automata,positioning,matrix}
\pgfplotsset{compat=1.16}

\usepackage[capitalize,nameinlink]{cleveref}
\usepackage{crossreftools}
\hypersetup{%
    bookmarksnumbered, bookmarksopen=true, bookmarksopenlevel=1,%
}

\crefname{subsection}{Subsection}{Subsections}
\crefname{lemma}{Lemma}{Lemmas}
\crefname{corollary}{Corollary}{Corollaries}
\crefname{theorem}{Theorem}{Theorems}
\crefname{assumption}{Assumption}{Assumptions}

\declaretheorem[name=Theorem]{theorem}

\declaretheorem[sibling=theorem,name=Lemma]{lemma}

\declaretheorem[sibling=theorem,name=Definition]{definition}
\declaretheorem[sibling=theorem,name=Assumption]{assumption}
\declaretheorem[name=Assumption, numbered=no]{assumption*}

\makeatletter
\renewcommand{\maketag@@@}[1]{\hbox{\m@th\normalsize\normalfont#1}}%
\makeatother

\makeatletter
\let\reftagform@=\tagform@
\def\tagform@#1{\maketag@@@{\ignorespaces\textcolor{gray}{(#1)}\unskip\@@italiccorr}}
\renewcommand{\eqref}[1]{\textup{\reftagform@{\ref{#1}}}}
\makeatother

\newcommand{\EE}{\mathbb{E}}

\newcommand{\NN}{\mathbb{N}}

\newcommand{\PP}{\mathbb{P}}

\newcommand{\RR}{\mathbb{R}}

\newcommand{\Aa}{\mathcal{A}}

\newcommand{\Cc}{\mathcal{C}}

\newcommand{\Ee}{\mathcal{E}}
\newcommand{\Ff}{\mathcal{F}}

\newcommand{\Ii}{\mathcal{I}}

\newcommand{\Ll}{\mathcal{L}}

\newcommand{\Pp}{\mathcal{P}}

\newcommand{\Ss}{\mathcal{S}}
\newcommand{\Tt}{\mathcal{T}}
\newcommand{\Uu}{\mathcal{U}}

\newcommand{\Xx}{\mathcal{X}}

\newcommand{\Zz}{\mathcal{Z}}

\def\[#1\]{\begin{equation}\begin{aligned}#1\end{aligned}\end{equation}}
\def\*[#1\]{\begin{equation*}\begin{aligned}#1\end{aligned}\end{equation*}}
\def\s*[#1\s]{\small\begin{align*}#1\end{align*}\normalsize}

\newcommand{\lcrx}[4][{-1}]{ 
	\IfEq{#1}{-1}{\left #2 {{{{#3}}}} \right #4}{
   	\IfEq{#1}{0}{#2 {{{{#3}}}} #4}{
	\IfEq{#1}{1}{\bigl #2 {{{{#3}}}} \bigr #4}{
	\IfEq{#1}{2}{\Bigl #2 {{{{#3}}}} \Bigr #4}{
	\IfEq{#1}{3}{\biggl #2 {{{{#3}}}} \biggr #4}{
	\IfEq{#1}{4}{\Biggl #2 {{{{#3}}}} \Biggr #4}{
    \GenericWarning{"4th argument to lcrx must be -1, 0, 1, 2, 3, or 4"}
    }}}}}}} % specify size with {-1,...4} as optional argument

\DeclareMathOperator{\E}{\EE} % Expectation
\newcommand{\suchthat}{\;\ifnum\currentgrouptype=16 \middle\fi|\;} % Vertical bar
\newcommand{\Ind}{\mathds 1} % indicator function
\def\multiset#1#2{\ensuremath{\left(\kern-.3em\left(\genfrac{}{}{0pt}{}{#1}{#2}\right)\kern-.3em\right)}}

\DeclareMathOperator*{\esssup}{\ess\sup} % essential supremum
\DeclareMathOperator*{\newlim}{\mathrm{lim}\vphantom{\mathrm{infsup}}}
\DeclareMathOperator*{\newmin}{\mathrm{min}\vphantom{\mathrm{infsup}}}
\DeclareMathOperator*{\newmax}{\mathrm{max}\vphantom{\mathrm{infsup}}}

\DeclareMathOperator*{\newsup}{\mathrm{sup}\vphantom{\mathrm{infsup}}}
\DeclareMathOperator*{\newess}{\mathrm{ess}\vphantom{\mathrm{infsup}}}
\renewcommand{\lim}{\newlim}
\renewcommand{\min}{\newmin}
\renewcommand{\max}{\newmax}

\renewcommand{\sup}{\newsup}
\newcommand{\ess}{\newess}

\newcommand{\dee}{\mathrm{d}} % for integrals \int f(x) \dee x
\newcommand{\grad}{\nabla} % gradient
\DeclareDocumentCommand{\virtualDiff}{m m G{} G{}}{\frac{#1^{#4} \kern 1pt #3}{#1 \kern 1pt #2^{#4}}}
\DeclareDocumentCommand{\pdiff}{m G{} G{}}{\virtualDiff{\partial}{#1}{#2}{#3}} % Partial derivative \pdiff{denominator}{optional numerator}{optional exponant}
\DeclareDocumentCommand{\diff}{m G{} G{}}{\virtualDiff{\text d}{#1}{#2}{#3}} % Derivative

\newcommand{\normaldist}{\mathcal{N}}

\newcommand{\sbra}[2][{-1}]{\lcrx[#1] [ {#2} ] }

\newcommand{\abs}[2][{-1}]{\lcrx[#1] \vert {#2} \vert }

\newcommand{\norm}[2][{-1}]{\lcrx[#1] \Vert {#2} \Vert}

\newcommand{\Nats}{\NN}

\newcommand{\Reals}{\RR}

\newcommand{\PosReals}{\Reals_+}

\newcommand{\range}[2][{1}]{
	\IfEq{#1}{1}{\sbra{#2}}{\sbra{#2}_{#1}}}
\newcommand{\rangeO}[2][{0}]{
	\IfEq{#1}{0}{\sbra{#2}_0}{\sbra{#2}_{#1}}}

\newcommand{\RL}{RL} % reinforcement learning
\newcommand{\MDP}{MDP} % markov decision process
\newcommand{\DP}{DP} % dynamic programming
\newcommand{\DPP}{DPP} % dynamic programming principle
\newcommand{\DPE}{DPE} % dynamic programming equations
\newcommand{\ANN}{ANN} % artificial neural network
\newcommand{\BSDE}{BSDE} % backward stochastic differential equation
\newcommand{\PDE}{PDE} % partial differential equation
\newcommand{\CVaR}{\text{CVaR}} % conditional value-at-risk
\newcommand{\agentname}{agent} % agent, player
\newcommand{\agentnames}{\agentname{}s} % agents, players
\newcommand{\episodename}{episode} % episode, epoch
\newcommand{\episodenames}{\episodename{}s} % episodes, epochs
\newcommand{\transitionname}{transition} % transition, step
\newcommand{\transitionnames}{\transitionname{}s} % transitions, steps
\newcommand{\timename}{period} % round, period, time step
\newcommand{\timenames}{\timename{}s} % rounds, periods, time steps
\newcommand{\stratname}{policy} % strategy, policy
\newcommand{\stratnames}{policies} % strategies, policies
\newcommand{\agentstratname}{\agentname{}'s \stratname{}} % agent strategy, agent policy
\newcommand{\datastratname}{data-generating process} % data-generating mechanism, data-generating process, simulation engine, environment
\newcommand{\datastratnames}{\datastratname{}es} % data-generating mechanisms, data-generating processes, simulation engines

\newcommand{\statespace}{\Ss}
\newcommand{\state}{s}
\newcommand{\statedum}{\state'}
\newcommand{\nextstates}{\state_{\timeidx+1}^{\policyparams}}

\newcommand{\actionspace}{\Aa}
\newcommand{\action}{a}

\newcommand{\costspace}{\Cc}
\newcommand{\costfunc}{c}

\newcommand{\periodspace}{\Tt}

\newcommand{\policy}{\pi}
\newcommand{\policyparams}{\theta}

\newcommand{\valuefunc}{V}
\newcommand{\valueparams}{\phi}

\newcommand{\paramspace}{\Theta}

\newcommand{\inventory}{q}
\newcommand{\inventorymax}{\inventory_{\max}}
\newcommand{\price}{S} % S or x?
\newcommand{\pricespace}{\PosReals}
\newcommand{\trade}{\action} %\action or u?
\newcommand{\trademax}{\trade_{\max}}
\newcommand{\wealth}{y}
\newcommand{\wealthspace}{\Reals}
\newcommand{\clifflimit}{C}
\newcommand{\bank}{B}

\newcommand{\Lpspace}{\Zz}
\newcommand{\Lqspace}{\Zz^{*}}
\newcommand{\rv}{Z}
\newcommand{\rvdum}{W}

\newcommand{\eplength}{T}
\newcommand{\timeidx}{t}

\newcommand{\riskmeas}{\rho}
\newcommand{\dualriskmeas}{\riskmeas^{*}}
\newcommand{\dualdualriskmeas}{\riskmeas^{**}}
\newcommand{\riskenv}{\Uu}
\newcommand{\discount}{\gamma}

\newcommand{\lagrangian}{L}

\newcommand{\weight}{\xi}
\newcommand{\EEweight}{\EE^{\weight}}

\newcommand{\Ntrajectories}{N} %N
\newcommand{\Ntransitions}{M} %M
\newcommand{\transitionidx}{m} %m
\newcommand{\Nepochs}{K} %K
\newcommand{\epochidx}{\kappa} %\kappa
\newcommand{\NepochsV}{K} %K_1
\newcommand{\epochVidx}{k} %k
\newcommand{\NepochsPI}{K} %K_2
\newcommand{\epochPIidx}{k} %k
\newcommand{\NbatchsV}{B} %B_1
\newcommand{\batchVidx}{b} %b
\newcommand{\NbatchsPI}{B} %B_2
\newcommand{\batchPIidx}{b} %b

\renewcommand{\grad}[1]{\nabla_{#1}}
\definecolor{mgreen}{rgb}{0,0.455,0.247}
\definecolor{mblue}{rgb}{0.098,0.18,0.357}
\definecolor{mred}{rgb}{0.902,0.4157,0.0196}
\definecolor{mgrey}{rgb}{0.90196,0.90,0.90}
\definecolor{mblack}{rgb}{0,0,0}

\usepackage{arxiv}
\usepackage[T1]{fontenc}    % use 8-bit T1 fonts
\usepackage{nicefrac}       % compact symbols for 1/2, etc.
\usepackage{microtype}      % microtypography
\usepackage{todonotes}
\usepackage{setspace}
\title{Reinforcement Learning with\\Dynamic Convex Risk Measures}

\author{Anthony Coache\thanks{The authors acknowledge support from the Natural Sciences and Engineering Research Council of Canada (grants RGPIN-2018-05705, RGPAS-2018-522715, and CGSD3-2019-534435). \new{The authors also thank the anonymous referees for their thoughtful comments during the revision process.}} \\
	Department of Statistical Sciences\\
	University of Toronto\\
	\href{mailto:anthony.coache@mail.utoronto.ca}{anthony.coache@mail.utoronto.ca} \\
	\url{https://anthonycoache.ca/} \\
	\And
	Sebastian Jaimungal\footnotemark[1] \\
	Department of Statistical Sciences\\
	University of Toronto\\
	\href{mailto:sebastian.jaimungal@utoronto.ca}{sebastian.jaimungal@utoronto.ca} \\
	\url{http://sebastian.statistics.utoronto.ca/} \\
}

\renewcommand{\headeright}{}
\renewcommand{\undertitle}{}
\renewcommand{\shorttitle}{RL with Dynamic Convex Risk}

\hypersetup{
pdftitle={Reinforcement Learning with Dynamic Convex Risk Measures},
pdfsubject={cs.LG, q-fin.CP},
pdfauthor={Anthony ~Coache, Sebastian ~Jaimungal},
pdfkeywords={Reinforcement Learning, Dynamic Risk Measures, Time-Consistency, Actor-Critic Algorithm, Policy Gradient.},
}

\newcommand{\new}[1]{#1}

\begin{document}

\maketitle

\allowdisplaybreaks
%!TEX root = ../main.tex

% --------------------------------------------------------------
%                         Abstract
% --------------------------------------------------------------

\begin{abstract}%
    We develop an approach for solving time-consistent risk-sensitive stochastic optimization problems using model-free reinforcement learning (\RL{}).
    Specifically, we assume \agentnames{} assess the risk of a sequence of random variables using dynamic convex risk measures. We employ a time-consistent dynamic programming principle to determine the value of a particular  \stratname{}, and develop policy gradient update rules that aid in obtaining optimal policies. We further develop an actor-critic style algorithm using neural networks to optimize over \stratnames{}. Finally, we demonstrate the performance and flexibility of our approach by applying it to three optimization problems: statistical arbitrage trading strategies, \new{financial hedging}, and \new{obstacle avoidance robot control}.
\end{abstract}

% ARXIV
\keywords{Reinforcement Learning \and Dynamic Risk Measures \and Policy Gradient \and Actor-Critic Algorithm \and  Time-Consistency \and Trading Strategies \and Financial Hedging \and Robot Control}

% intro
%!TEX root = ../main.tex

% --------------------------------------------------------------
%                         Introduction
% --------------------------------------------------------------
\section{Introduction}
\label{sec:introduction}

Reinforcement learning (\RL{}) provides a (model-free) framework for learning-based control.
\RL{} problems aim at learning dynamics in the underlying data and finding optimal behaviors while collecting data via an interactive process.
It differs from \emph{supervised learning} that attempts to learn classification functions from labeled data, and \emph{unsupervised learning} that seeks hidden patterns in unlabeled data.
\new{During the training phase,} the \agentname{} makes a sequence of decisions while interacting with the \datastratname{} and observing feedback in the form of costs.
The \agentname{} aims to discover the best possible actions by interacting with the environment and consistently updating their actions based on their experience, while often, also taking random actions to assist in exploring the state space -- the classic exploration-exploitation trade-off \citep{sutton2018reinforcement}.

In \RL{}, \emph{uncertainty} in the \datastratname{} can have substantial effects on performance.
Indeed, the environmental randomness may result in algorithms optimized for ``on-average'' performance to have large variance across scenarios.
For example, consider a portfolio selection problem: a risk-neutral optimal strategy (where one optimizes the expected terminal return) focuses on assets with the highest returns while ignoring the risks associated with them.
As a second example, consider an autonomous car which should account for environmental uncertainties such as the weather and road conditions when learning the optimal strategy.
Such an \agentname{} may wish to account for variability in the environment and the results of its actions to avoid large potential ``losses''. For an overview and outlook on RL in financial mathematics see, e.g., \cite{jaimungal2022reinforcement}.

In the extant literature, there are numerous proposals for accounting for risk sensitivity, where most of them replace the expectation in the optimization problem by risk measures -- we provide an overview in \cref{sec:related-work}.
\emph{Risk-awareness} or \emph{risk-sensitivity} in \RL{} offers a remedy to the \datastratname{} uncertainty by quantifying low-probability but high-cost outcomes, provides strategies that are more robust to the environment, and allows more flexibility than traditional approaches as it is tuned to the \agentname{}'s risk preference. The specific choice of risk measure is a decision the \agentname{} makes considering their goals and risk tolerances. We do not address here how the \agentname{} makes this specific choice and instead refer the reader to, e.g., \cite{dhaene2006risk} for an overview.

An interesting approach to risk-aware learning stems from \cite{tamar2015policy,tamar2016sequential}, where they provide policy search algorithms to solve risk-aware \RL{} problems in a dynamic framework. 
Both studies investigate \emph{stationary policies}, restrict themselves to \emph{coherent} risk measures, and apply their methodology to simple financial engineering applications.
More specifically, they evaluate their algorithms when learning policies for (perpetual) American options and trading in static portfolio optimization problems.
Several real-world applications require a level of flexibility that these limitations preclude.

In this paper, we develop a model-free approach to solve a wide class of (non-stationary) risk-aware \RL{} problems in a time-consistent manner.
Our contributions may be summarized as follows:
(i) we extend \cite{tamar2015policy,tamar2016sequential,ahmadi2020constrained,kose2021risk,huang2021convergence} by focusing on the broad class of \emph{dynamic convex} risk measures and consider finite-horizon problems with \emph{non-stationary \stratnames{}}; 
(ii) we devise an \emph{actor-critic} algorithm to solve this class of \RL{} problems using neural networks to allow continuous state-action spaces;
(iii) we derive a recursive formula for efficiently computing the policy gradients; and (iv) we demonstrate the performance and flexibility of our proposed approach on three important applications: optimal trading for statistical arbitrage, \new{hedging financial options}, and \new{obstacle avoidance in robot control}. We demonstrate that our approach appropriately accounts for uncertainty and leads to strategies that mitigate risk.

The remainder of the paper is structured as follows.
In \cref{sec:related-work}, we discuss related work and introduce our \RL{} notation in \cref{sec:rl-notation}.
\cref{sec:risk-notation} formalizes the evaluation of risk in both static and dynamic frameworks.
We introduce the class of sequential decision making problems with dynamic convex risk measures in \cref{sec:problem-setup} and devise an actor-critic style algorithm to solve them in \cref{sec:algorithm}.
Finally, \cref{sec:experiments} illustrates the performance of our proposed algorithm, and we discuss our work's limitations and possible extensions in \cref{sec:conclusion}.

% related work / lit review
%!TEX root = ../main.tex

% --------------------------------------------------------------
%                         Related work
% --------------------------------------------------------------
\section{Related Work}
\label{sec:related-work}

% -------  Static vs dynamic framework  ------- %
The literature in risk evaluation for sequential decision making can be divided between those that apply a risk measure to a single cost random variable \new{(e.g. discounted sum of costs or terminal profit and loss)}, and those that apply risk measures recursively to a sequence of cost random variables \new{(e.g. cash-flows)}.
The former approach optimizes the risk of a single random variable, which does not account for the temporal structure in what generates it, while the latter optimizes the risk of sequences of random variables in a \emph{dynamic} framework as additional information becomes available.

% -------  Static framework  ------- %
Several authors address sequential decision making problems by minimizing the risk of a cost \emph{over a whole \episodename{}}.
For example, \cite{prashanth2013actor} focus on objective functions for variance related criteria, while \cite{chow2017risk} take a risk-constrained approach with an objective function that includes a penalty on the conditional value-at-risk (\CVaR{}).
Some lines of research look at risk-sensitive \RL{} problems using a broader classes of risk measures, such as comonotonic \citep{petrik2012approximate}, entropic \citep{nass2019entropic}, or spectral \citep{bauerle2020minimizing} risk measures.
\cite{di2019practical} consider a risk-neutral objective function, but includes an additional constraint on the risk of the cumulative discounted cost.

% -------  Dynamic framework  ------- %
Other works extend optimization in Markov decision processes (\MDP{}s) by evaluating the risk \emph{at each \timename{}}.
For instance, \cite{ruszczynski2010risk} evaluates the risk at each \timename{} using dynamic Markov coherent risk measures, while \cite{chu2014markov} and \cite{bauerle2021markov} propose iterated coherent risk measures, where they both derive risk-aware dynamic programming (\DP{}) equations and provide policy iteration algorithms.
While they focus on coherent risk measures, various classes of risk measures have already been extended to a dynamic framework, such as distribution-invariant risk measures \citep{weber2006distribution}, coherent risk measures \citep{riedel2004dynamic}, convex risk measures \citep{frittelli2004dynamic,detlefsen2005conditional}, and dynamic assessment indices \citep{bielecki2016dynamic}, among others -- however, these works do not look at how to perform model free optimization, i.e., they do not look at RL. For an overview of dynamic risk measures, see, e.g., \cite{acciaio2011dynamic}.

% -------  Robust optimization  ------- %
Another way to account for uncertainty in the \datastratname{} is by allowing the parameters of the model, or the entire distribution itself, to be unknown.
The class of \emph{robust \MDP{}s} \citep{delage2010percentile} focuses on optimizing the worst-case expectation when the parameters vary within a certain set.
There exists relationships between risk-aware and robust \MDP{}s, as shown in \cite{osogami2012robustness} and \cite{bauerle2020distributionally}.
Indeed, minimizing a Markov coherent risk measure in a risk-aware context is equivalent to minimizing a certain worst-case expectation where the uncertainty set is characterized by a concave function.
Several researchers have developed algorithms to solve robust \MDP{}s, for an overview see, e.g., \cite{rahimian2019distributionally}.

% -------  Model-based approaches  ------- %
While we consider a model-free approach, several researchers tackle related, but distinct, problems in a \emph{model-based framework}.
For instance, \cite{weinan2017deep} and \cite{han2018solving} use deep learning methods to solve non-linear partial differential equations (\PDE{}s). Using the non-linear Feynman-Kac representation, they reformulate the \PDE{}s as backward stochastic differential equations (\BSDE{}s), they then parametrize the co-adjoint process and initial condition using  an ensemble of neural networks, and use the mean squared  error in the terminal condition as the loss function. Despite there being an equivalence between \BSDE{}s and dynamic risk measures \citep{peng1997backward,drapeau2016dual}, the dual representation does not directly help when we aim to optimize a dynamic risk measure in a model-free manner.

% -------  Misc  ------- %
Other types of algorithms exist in the literature to find a solution to risk-aware \RL{} problems.
Among others, \cite{galichet2013exploration} use a \emph{multi-armed bandit approach}, \cite{shen2014risk} devise a risk-sensitive \emph{Q-learning method} when optimizing utility functions, \cite{bellemare2017distributional} use a \emph{distributional perspective} to learn the whole distribution of the value function, \cite{yu2018approximate} employ an approximate \DP{} approach to devise a \emph{value iteration} algorithm for Markov risk measures, and \cite{kalogerias2020better} address risk-aware problems from a \emph{Bayesian perspective}.
The shortcomings of these approaches are that they apply to a specific class of risk measures and the developed methodologies are tuned to them.

% our work
%!TEX root = ../main.tex

% --------------------------------------------------------------
%                         RL notation
% --------------------------------------------------------------
\section{Reinforcement Learning}
\label{sec:rl-notation}

\tikzstyle{reward}=[shape=circle,draw=mblue,fill=mblue!10]
\tikzstyle{action}=[shape=circle,draw=mgreen,fill=mgreen!10]
\tikzstyle{state}=[shape=circle,draw=mred,fill=mred!10]
\tikzstyle{gru}=[shape=rectangle,draw=black!50,fill=lime!10]
\tikzstyle{obs}=[shape=circle,draw=mblue,fill=mblue!10]
\tikzstyle{lightedge}=[<-,dotted]
\tikzstyle{mainstate}=[state,thick]
\tikzstyle{mainedge}=[<-,thick]			
\begin{wrapfigure}[9]{r}{0.5\textwidth}
    \centering
    \begin{tikzpicture}[scale=0.7,every node/.style={transform shape},minimum width=1.0cm]
    
    \node[reward] (r1) at (3,2) {$~\costfunc_{\timeidx}~$};
    \node[reward] (r2) at (6,2) {$\costfunc_{\timeidx+1}$};
    
    \node[] (s0) at (-2,0) {$\dots$};
    \node[state,scale=1] (s1) at (0,0) {$~\state_{\timeidx}~$};
    \node[state,scale=1] (s2) at (3,0) {$\state_{\timeidx+1}$};
    \node[state,scale=1] (s3) at (6,0) {$\state_{\timeidx+2}$};
    \node[] (s4) at (8,0) {$\dots$};
    
    \node[action] (a1) at (1.5,1.25) {$~\action_{\timeidx}~$};
    \node[action] (a2) at (4.5,1.25) {$\action_{\timeidx+1}$};
    
    \draw [->] (s0) to (s1);
    \draw [->] (s1) to (s2);
    \draw [->] (s2) to (s3);
    \draw [->] (s3) to (s4);
    
    \draw [->] (s1) to (a1);
    \draw [->] (a1) to (r1);
    \draw [->] (s1) to [out=90, in=165] (r1.west);
    \draw [->] (a1) to  (s2);
    
    \draw [->] (s2) to (a2);
    \draw [->] (a2) to (r2);
    \draw [->] (s2) to [out=90, in=165] (r2.west);
    \draw [->] (a2) to  (s3);
    
    \draw [->] (s2) to (r1);
    \draw [->] (s3) to (r2);
    
    \end{tikzpicture}
    \caption{Directed graph representation of an MDP.}
    \label{fig:tikz-MDP}
    
\end{wrapfigure}
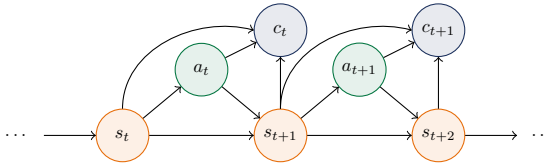
In this section, we introduce the necessary theoretical background for the \RL{} problems we study.
We describe each problem as an \emph{\agentname{}} who tries to learn an optimal behavior, or \emph{\agentstratname{}}, by interacting with a certain \emph{\datastratname{}}.

Let $\statespace$ and $\actionspace$ be arbitrary state and action spaces respectively, and let $\costspace \subset \Reals$ be a cost space.
The \datastratname{} is often represented as a \emph{\MDP{}} with the tuple $(\statespace, \actionspace, \costfunc, \PP)$, where $\costfunc(\state, \action, \statedum) \in \costspace$ is a deterministic, state-action dependent cost function and $\PP$ characterizes the transition probabilities $\PP(\state_{\timeidx+1} = \statedum | \state_{\timeidx} = \state, \action_{\timeidx} = \action)$, \new{unknown to the \agentname{}}.
The transition probability is assumed stationary, although time may be a component of the state, e.g. we usually assume that time is embedded in the state space without loss of generality.
A single \emph{\episodename{}} consists of $\eplength$ \emph{\timenames{}}, where $\eplength \in \Nats$ is known and finite.
At each \timename{}, the \agentname{} begins in state $\state_{\timeidx} \in \statespace$, takes an action $\action_{\timeidx} \in \actionspace$, moves to the next state $\state_{\timeidx+1} \in \statespace$, and receives a cost $\costfunc_{\timeidx} = \costfunc(\state_{\timeidx}, \action_{\timeidx}, \state_{\timeidx+1}) \in \costspace$.
A directed graph representation of the described \MDP{} is shown in \cref{fig:tikz-MDP}.
A \emph{trajectory} consists of all states, actions and costs that occur during a single \episodename{} and we denote it by the tuple $\tau := (\state_0, \action_0, \costfunc_0, \ldots, \state_{\eplength-1}, \action_{\eplength-1}, \costfunc_{\eplength-1}, \state_{\eplength})$.

The \agentname{} follows a strategy described by a randomized (also known as exploratory control) \emph{\stratname{}} $\policy: \statespace \rightarrow \Pp(\actionspace)$, where $\Pp(\actionspace)$ is the space of probability measures on $\sigma(\actionspace)$.
More specifically when in state $\state$ at time $\timeidx$, the \agentname{} selects the action $\action$ with probability $\policy(\action | \state_{\timeidx} = \state)$.
This also  allows for non-stationary \stratnames{} when dealing with finite-horizon problems.

Standard \RL{} usually deals with \emph{risk-neutral} objective functions of the ($\discount$-discounted) cost of a trajectory induced by a \MDP{} with a policy $\policy$, for instance
\begin{equation}
    \min_{\policy} \, \E \left[\; \sum_{\timeidx=0}^{\eplength-1} \gamma^{\timeidx} \costfunc(\state_{\timeidx}, \action_{\timeidx}, \state_{\timeidx+1}) \right],\label{eq:rl-obj}
\end{equation}
where $\discount \in (0,1]$ is a discount factor.
At \timename{} $\eplength$, there is no action and hence any cost based on the terminal state is encapsulated in $\costfunc_{\eplength-1}$.
Contrastingly, risk-sensitive \RL{} considers problems in which the objective takes into account the \emph{variability} of the cost with a risk measure $\riskmeas$, whether optimizing
\begin{equation}
    \min_{\policy} \, \riskmeas \left( \sum_{\timeidx=0}^{\eplength-1} \gamma^{\timeidx} \costfunc(\state_{\timeidx}, \action_{\timeidx}, \state_{\timeidx+1}) \right), \label{eq:risk-rl-obj}
\end{equation}
or \cref{eq:rl-obj} with an additional constraint on the risk measure of the trajectory cost.
We discuss thoroughly risk measures and their properties in \cref{sec:risk-notation} next.

In all cases, the goal of \RL{} is to learn the optimal \stratname{} $\policy$ that attains the minimum in the corresponding objective function, and do so in a manner that makes no explicit assumptions on the \datastratname{}.

%!TEX root = ../main.tex

% --------------------------------------------------------------
%                         Risk measure notation
% --------------------------------------------------------------
\section{Risk Measures}
\label{sec:risk-notation}

In this section, we formalize how we quantify the risk of random variables.
We start by providing a review of static risk measures, and then continue by reviewing the framework of \cite{ruszczynski2010risk} for building time-consistent dynamic risk measures.
Static risk measures evaluate the immediate risk of a financial position, while dynamic risk measures allow its monitoring at different times, and lead to time-consistent optimal strategies.
\new{In our work, we make use of static risk measures on conditional distributions of random variables to construct dynamic risk measures.}

%!TEX root = ../main.tex

% --------------------------------------------------------------
%                         Static risk measures
% --------------------------------------------------------------
\subsection{Static Setting}
\label{ssec:static-risk}

Let $(\Omega, \Ff, \PP)$ be a probability space, \new{$\bar{\Reals}:=\Reals \cup \{ -\infty, \infty \}$ the extended real line,} and define $\Lpspace := \Ll_p(\Omega, \Ff, \PP)$ and $\Lqspace := \Ll_q(\Omega, \Ff, \PP)$ with $p,q \in [1, \infty]$ \new{such that $1/p + 1/q = 1$}.
They represent the space of $p$-integrable, respectively $q$-integrable, $\Ff$-measurable random variables.
A \emph{risk measure} is a mapping \new{$\riskmeas: \Xx \rightarrow \bar{\Reals}$}, where $\Xx$ is a space of random variables\new{, that satisfies additional properties}.
In what follows we assume that $\Xx = \Lpspace$ and $\rv \in \Lpspace$ is interpreted as a random cost.
We next enumerate some properties of various risk measures.
\begin{definition}
    Let $m \in \Reals$, $\beta > 0$ and $\lambda \in [0,1]$. A \new{map $\riskmeas : \Xx \rightarrow \bar{\Reals}$} is said to be
    \begin{enumerate}
    	\setlength{\itemsep}{0pt}
    	\item \emph{law-invariant} if $\riskmeas(\rv_1) = \riskmeas(\rv_2)$ if $\rv_1\stackrel{d}{=}\rv_2$;
    
    	\item \emph{monotone} if $\rv_1 \leq \rv_2$ implies $\riskmeas(\rv_1) \leq \riskmeas(\rv_2)$;
    	
    	\item \emph{translation invariant} if $\riskmeas(\rv+m) = \riskmeas(\rv)+m$;
    	
    	\item \emph{positive homogeneous} if $\riskmeas(\beta \rv) = \beta\, \riskmeas(\rv)$;
    	
    	\item \emph{comonotonic additive} if $\riskmeas(\rv_1 + \rv_2) = \riskmeas(\rv_1) + \riskmeas(\rv_2)$ for all comonotonic pairs $(\rv_1,\rv_2)$;
    	
    	\item \emph{subadditive} if $\riskmeas(\rv_1 + \rv_2) \leq \riskmeas(\rv_1) + \riskmeas(\rv_2)$;
    	
    	\item \emph{convex} if $\riskmeas (\lambda \rv_1 + (1-\lambda) \rv_2) \leq \lambda \riskmeas(\rv_1) + (1-\lambda) \riskmeas(\rv_2)$.
    \end{enumerate}    
\end{definition}

There is a consensus in the literature that any risk measure should satisfy the monotonicity and translation invariance properties.
More specifically, a portfolio with a higher cost for every possible scenario is indeed riskier, and the deterministic part of a portfolio does not contribute to its risk.
\begin{definition}
	A \new{map} $\riskmeas$ is said to be a \emph{monetary \new{risk measure}} \citep{follmer2004stochastic} if and only if it is monotone and translation invariant.
\end{definition}

In addition to being monetary, additional requirements may be assumed so that the risk measure reflects \emph{observed investor behavior}.
\begin{definition}
	A \new{map} $\riskmeas$ is said to be a \emph{coherent \new{risk measure}} \citep{artzner1999coherent} if and only if it is monotone, translation invariant, positive homogeneous, and subadditive.
	\label{def:coherent_riskmeasure}
\end{definition}

The additional properties for coherent risk measures guarantees that doubling a position doubles its risk, and diversifying a portfolio reduces its risk.
The CVaR \citep{rockafellar2000optimization} is a well-known example of a coherent risk measure commonly used in the literature.
Criticisms of positive homogeneity and subadditivity led to the study of a broader class of risk measures, where these axioms are weakened and replaced by the notion of convexity.
Indeed, the risk might increase in a nonlinear way, which is not permitted under coherent risk measures.
\begin{definition}
	A \new{map} $\riskmeas$ is said to be  a \emph{convex \new{risk measure}} \citep{follmer2002convex} if and only if it is monotone, translation invariant, and convex.
	\label{def:convex_riskmeasure}
\end{definition}

Another advantage of convex risk measures is that we can combine several risk measures into a linear combination to create a trade-off between different risk-aware objectives. 
Indeed, one can easily show that given two convex risk measures $\riskmeas_1,\riskmeas_2$ and nonnegative coefficients $\beta_1,\beta_2 > 0$, then $\rho:=\beta_1\, \riskmeas_1 + \beta_2\, \riskmeas_2$ \new{and $\rho:=\max\{\riskmeas_1,\riskmeas_2\}$ are both} also convex. The set of coherent risk measures is a strict subset of the set of convex risk measures.

A dual representation of convex (and, as a special case, coherent) risk measures provides us with a key result for developing our practical algorithm.
The dual representation requires us to introduce the expectation under what is effectively a distorted probability measure and denote $\EEweight [\rv] := \sum_{\omega} \rv(\omega) \weight(\omega) \PP(\omega)$ with $\rv \in \Lpspace$ and $\weight \in \Lqspace$.
\begin{definition}
    The \emph{conjugate} \citep{shapiro2014lectures} of the risk measure $\riskmeas$, denoted \new{$\dualriskmeas: \Lqspace \rightarrow \bar{\Reals}$}, is given by
    \begin{equation}
    	\dualriskmeas(\weight) = \sup_{\rv \in \Lpspace} \left\{ \EEweight[\rv] - \riskmeas(\rv) \right\}. \label{eq:conjugate}
    \end{equation}
\end{definition}
\begin{definition}
    The \emph{biconjugate} \citep{shapiro2014lectures}, or conjugate of the conjugate, is a mapping \new{$\dualdualriskmeas: \Lpspace \rightarrow \bar{\Reals}$} with
\begin{equation}
	\dualdualriskmeas(\rv) = \sup_{\weight \in \Lqspace} \left\{ \EEweight\left[ \rv \right] - \dualriskmeas(\weight) \right\}. \label{eq:biconjugate}
\end{equation}
\end{definition}
The dual representation of the risk measure is given in the following theorem.
\begin{theorem}[Representation Theorem \new{{\citep[see Theorem 6.4 in][]{shapiro2014lectures}}}]
	\label{thm:representation-thm}
	A convex risk measure $\riskmeas$ is proper (i.e. $\riskmeas(\rv) > -\infty$ and its domain is nonempty) and lower semicontinuous (i.e. $\riskmeas(\rvdum) \leq \liminf_{\rv \rightarrow \rvdum} \riskmeas(\rv)$) iff there exists a set
	$$
	\riskenv(P) \subset \Big\{ \weight \in \Lqspace \, : \, \textstyle{\sum}_{\omega} \, \weight(\omega) \PP(\omega) = 1, \ \weight \geq 0 \Big\},
	$$
	often referred to as the \emph{risk envelope}, such that 
	\begin{equation}
		\riskmeas(\rv) = \sup_{\weight \in \, \riskenv(\PP)} \left\{ \EEweight \left[ \rv \right] - \dualriskmeas(\weight) \right\}. \label{eq:representation-thm}
	\end{equation}
	Moreover, we have that $\riskmeas$ is coherent (e.g. satisfies also the positive homogeneity) iff \begin{equation*}
		\riskmeas(\rv) = \sup_{\weight \in \, \riskenv(\PP)} \left\{ \EEweight \left[ \rv \right] \right\}.
	\end{equation*}
\end{theorem}

If we assume $\riskmeas$ is a convex, proper, lower semicontinuous risk measure, then by \cref{thm:representation-thm}, it may be written as an optimization problem where the distortion $\weight$ is chosen adversarially from a subset of the set of all probability densities.
Moreover, the notable difference between coherent and convex risk measures is the conjugate term that appears in its dual representation.
Coherent risk measures are \emph{uniquely} characterized by their risk envelope.
%!TEX root = ../main.tex

% --------------------------------------------------------------
%                         Dynamic risk measures
% --------------------------------------------------------------
\subsection{Dynamic Setting}
\label{ssec:dynamic-risk}

Optimizing a controlled performance criteria using static risk measures is known to lead to time-inconsistent solutions -- we discuss the precise definition below.
Adapting risk measures to properly account for the flow of information requires  additional care to ensure that the risk evaluation is done in a time-consistent manner, especially with \DP{} models for \MDP{}s.
There are multiple extensions of static risk measure to the dynamic case.
Indeed, various classes of risk measures have already been extended to a dynamic framework, such as distribution-invariant \citep{weber2006distribution}, coherent \citep{riedel2004dynamic}, convex risk measures \citep{frittelli2004dynamic,detlefsen2005conditional}, and dynamic assessment indices \cite{bielecki2016dynamic},  among others.
Here we closely follow the work of \cite{ruszczynski2010risk}.

To this end, let $\periodspace := \{0, \ldots, \eplength\}$ denote the sequence of \timenames{} in an \episodename{}.
Consider a filtration $\Ff_{0} \subseteq \Ff_{1} \subseteq \ldots \subseteq \Ff_{\eplength} \subseteq \Ff$ on a \new{filtered} probability space \new{$(\Omega, \Ff, \{\Ff_{\timeidx}\}_{\timeidx \in \periodspace}, \PP)$} and $\{\Lpspace_{\timeidx}\}_{\timeidx \in \periodspace}$ with $\Lpspace_{\timeidx} := \Ll_{p}(\Omega, \Ff_{\timeidx}, \PP)$.
Define $\Lpspace_{\timeidx,\eplength} := \Lpspace_{\timeidx} \times \cdots \times \Lpspace_{\eplength}$.
\begin{definition}
    A \emph{conditional risk measure} 
    %\citep{ruszczynski2010risk} 
    is a map $\riskmeas_{\timeidx,\eplength}: \Lpspace_{\timeidx,\eplength} \rightarrow \Lpspace_{\timeidx}$ which satisfies the monotonicity property, i.e. $\riskmeas_{\timeidx,\eplength}(\rv) \leq \riskmeas_{\timeidx,\eplength} (\rvdum)$ for all $\rv,\rvdum \in \Lpspace_{\timeidx,\eplength}$ such that $\rv \leq \rvdum$ a.s..
\end{definition}
\begin{definition}
	A \emph{dynamic risk measure} 
	%\citep{ruszczynski2010risk} 
	is a sequence of conditional risk measures $\{\riskmeas_{\timeidx,\eplength}\}_{\timeidx \in \periodspace}$.
\end{definition}

We may interpret $\riskmeas_{\timeidx,\eplength} (\rv)$, for $\rv\in\Lpspace_{\timeidx,\eplength}$, as a $\Ff_{\timeidx}$-measurable charge the \agentname{} would be willing to incur at time $\timeidx$ instead of the sequence of costs $\rv$.
Developing a dynamic programming principle (\DPP{}) for dynamic risk measures crucially depends on the property of \emph{time-consistency}.
One wishes to evaluate the risk of future outcomes, but it must not lead to inconsistencies at different points in time.
\begin{definition}
	\label{def:time-consistency}
	$\{\riskmeas_{\timeidx,\eplength}\}_{\timeidx \in \periodspace}$ is said to be \new{\emph{time-consistent} \citep[see Definition 3 in][]{ruszczynski2010risk}} iff for any sequence $\rv,\rvdum \in \Lpspace_{\timeidx_1,\eplength}$ and any $\timeidx_1,\timeidx_2 \in \periodspace$ such that ($0 \leq \timeidx_1 < \timeidx_2 \leq \eplength$),
	\begin{equation*}
		\rv_k = \rvdum_k, \; \forall k = \timeidx_1, \ldots, \timeidx_2 \quad \text{and} \quad \riskmeas_{\timeidx_2,\eplength}(\rv_{\timeidx_2}, \ldots, \rv_{\eplength}) \leq \riskmeas_{\timeidx_2,\eplength}(\rvdum_{\timeidx_2}, \ldots, \rvdum_{\eplength})
	\end{equation*}
	implies that $\riskmeas_{\timeidx_1,\eplength}(\rv_{\timeidx_1}, \ldots, \rv_{\eplength}) \leq \riskmeas_{\timeidx_1,\eplength}(\rvdum_{\timeidx_1}, \ldots, \rvdum_{\eplength})$.
\end{definition}

\cref{def:time-consistency} may be interpreted as follows: if $\rv$ will be at least as good as $\rvdum$ at time $\timeidx_2$ and they are identical between $\timeidx_1$ and $\timeidx_2$, then $\rv$ should not be worse than $\rvdum$ at time $\timeidx_1$.

Furthermore, we introduce one additional concept that aids in developing a recursive relationship for dynamic risk measures.
\begin{definition}
    \label{def:one-step-cond}
    A \emph{one-step conditional risk measure}
    %\cite{ruszczynski2010risk}
    is a map $\riskmeas_{\timeidx}: \Lpspace_{\timeidx+1} \rightarrow \Lpspace_{\timeidx}$ which satisfies $\riskmeas_{\timeidx} (\rv) = \riskmeas_{\timeidx, \timeidx+1} (0, \rv)$ for any $Z\in\Lpspace_{\timeidx+1}$.
\end{definition}

One may assume even stronger properties for conditional risk measures, e.g., one may assume the risk measures are static (across time) and convex. In the next section, we do precisely this.
Therefore, the one-step conditional risk measure $\riskmeas_{\timeidx} (\cdot \, | \, \Ff_{\timeidx})$ outputs an $\Ff_{\timeidx}$-measurable random variable obtained when conditioning on $\Ff_{\timeidx}$.
As a consequence of \cref{def:time-consistency,def:one-step-cond}, we have the following recursive relationship for time-consistent dynamic risk measures.
\begin{theorem}[Recursive relationship \new{{\citep[see Theorem 1 and Eq. (10) in][]{ruszczynski2010risk}}}]
    \label{thm:recursive-relation}
    Let $\{\riskmeas_{\timeidx,\eplength}\}_{\timeidx \in \periodspace}$ be a time-consistent, dynamic risk measure.
    \new{Suppose that it satisfies $\riskmeas_{\timeidx,\eplength}(\rv_{\timeidx}, \rv_{\timeidx+1}, \ldots, \rv_{\eplength}) = \rv_{\timeidx} + \riskmeas_{\timeidx,\eplength}(0, \rv_{\timeidx+1}, \ldots, \rv_{\eplength})$ and $\riskmeas_{\timeidx,\eplength} (0,\ldots,0) = 0$ for any $\rv\in\Lpspace_{\timeidx,\eplength}$, $\timeidx\in\periodspace$.}
    Then $\{\riskmeas_{\timeidx,\eplength}\}_{\timeidx \in \periodspace}$ can be expressed as
    \begin{equation}
    	\riskmeas_{\timeidx,\eplength} (\rv_{\timeidx}, \ldots, \rv_{\eplength}) = \rv_{\timeidx} +
    	\riskmeas_{\timeidx} \Bigg( \rv_{\timeidx+1} +
    	\riskmeas_{\timeidx+1} \bigg( \rv_{\timeidx+2} +
    	\cdots +
    	\riskmeas_{\eplength-2} \Big( \rv_{\eplength-1} +
    	\riskmeas_{\eplength-1} \big( \rv_{\eplength} \big) \Big) \cdots \bigg) \Bigg). \label{eq:dynamic-risk}
    \end{equation}
\end{theorem}

\new{As we are working with \MDP{}s, the \agentname{} observes at each \timename{} $\timeidx\in\periodspace$ the current state $\state_{\timeidx}$, that is $\Ff_{\timeidx}$-measurable.}
We thus define dynamic \emph{Markov} risk measures\new{, in a similar manner to \cite{ruszczynski2010risk},} where $\riskmeas_{\timeidx}$ are \new{Markov one-step conditional risk measures that are $\sigma(\state_{\timeidx})$-measurable}.
Those are \new{originally} obtained from \new{\emph{risk transition mappings} with respect to a \emph{controlled Markov process}, which our framework with \MDP{}s satisfies \citep[for a thorough exploration, see][especially Section 4]{ruszczynski2010risk}.}
%!TEX root = ../main.tex

% --------------------------------------------------------------
%                         Problem setup
% --------------------------------------------------------------
\section{Problem Setup}
\label{sec:problem-setup}

In this section, we formally introduce the optimization problems that we face.
Briefly, we are interested in \RL{} problems where the \agentname{} wants to minimize a dynamic convex risk measure in order to obtain a time-consistent solution.

Let $(\statespace, \actionspace, \costfunc, \PP)$ be a \MDP{}, $\periodspace := \{0, \ldots, \eplength-1\}$ be the sequence of \timenames{} in an \episodename{}, and suppose that the \agentstratname{} $\policy$ is parametrized by some parameters $\policyparams \in \paramspace$.
We consider a time-consistent, Markov, dynamic convex risk measure $\{\riskmeas_{\timeidx,\eplength}\}_{\timeidx \in \periodspace}$\new{, i.e. the one-step conditional risk measures $\riskmeas_{\timeidx}$ are static convex risk measures, real-valued and lower semicontinuous}.
Using \cref{eq:dynamic-risk} from \cref{thm:recursive-relation}, we aim to solve the following $\eplength$-period risk-sensitive \RL{} problem:
\begin{equation}
	\min_{\policyparams} \, \riskmeas_{0,\eplength}(\rv^{\policyparams}) = \min_{\policyparams} \, 
	\riskmeas_{0} \Bigg(\costfunc^{\policyparams}_{0} +
	\riskmeas_{1} \bigg(\costfunc^{\policyparams}_{1} +
	\cdots +
	\riskmeas_{\eplength-2} \Big(\costfunc^{\policyparams}_{\eplength-2} +
	\riskmeas_{\eplength-1} \big(\costfunc^{\policyparams}_{\eplength-1})
	\big) \Big) \cdots \bigg) \Bigg), \label{eq:optim-problem1} \tag{P1}
\end{equation}
where $\costfunc^{\policyparams}_{\timeidx} = \costfunc(\state_{\timeidx}, \action^{\policyparams}_{\timeidx}, \state_{\timeidx+1}^{\policyparams})$ is a $\Ff_{\timeidx+1}$-measurable random cost and the trajectory may be modulated by the \stratname{} $\policy^{\policyparams}$ -- that is why we include a $\policyparams$ index to actions and states.
In the sequel, we may omit the superscript $\policyparams$ when obvious for readability purposes.

Let us denote transition probabilities by $\PP^{\policyparams}(\action,\statedum | \state_{\timeidx}=\state) := \PP(\statedum | \state, \action) \policy^{\policyparams}(\action | \state_{\timeidx} = \state)$ and $\weight$-weighted conditional expectations by $\EEweight_{\timeidx} [\rv] := \sum_{(\action,\statedum)} \weight(\action,\statedum) \PP^{\policyparams}(\action,\statedum | \state_{\timeidx}) \rv(\action,\statedum)$ for any $\timeidx \in \periodspace$.

\new{All one-step conditional risk measures $\riskmeas_{\timeidx}$ of the dynamic risk measure $\{\riskmeas_{\timeidx,\eplength}\}_{\timeidx \in \periodspace}$ satisfy the assumptions in \cref{thm:representation-thm}, and thus} using the dual representation from \cref{thm:representation-thm}, the problem in \cref{eq:optim-problem1} may be written equivalently as
\begin{equation}
\begin{split}
	\min_{\policyparams} \
	&\max_{\weight_{0} \in \riskenv(\PP^{\policyparams}(\cdot,\cdot | \state_0 = \state_0))} \Bigg\{
	\EE^{\weight_{0}}_{0} \Bigg[ \costfunc_0^{\policyparams}
	+ \max_{\weight_{1} \in \riskenv(\PP^{\policyparams}(\cdot,\cdot | \state_1 = \state_1))} \bigg\{
	\EE^{\weight_{1}}_{1} \bigg[ \costfunc_1^{\policyparams} + \qquad \cdots \\
	&\quad 
	+ \max_{\weight_{\eplength-1} \in \riskenv(\PP^{\policyparams}(\cdot,\cdot | \state_{\eplength-1} = \state_{\eplength-1}))} \Big\{
	\EE^{\weight_{\eplength-1}}_{\eplength-1} \Big[ \costfunc_{\eplength-1}^{\policyparams}
	\Big] - \dualriskmeas_{\eplength-1}(\weight_{\eplength-1}) \Big\} \cdots \bigg] - \dualriskmeas_{1}(\weight_{1})
	\bigg\} \Bigg] - \dualriskmeas_{0}(\weight_{0}) \Bigg\}.
\end{split} \label{eq:optim-problem2} \tag{P2}
\end{equation}
\begin{assumption}
    \label{assump:risk-envelope}
    We restrict to convex risk measures $\rho$ such that the risk envelope $\riskenv$ may be written as
    \begin{equation}
    	\riskenv(\PP^{\policyparams}(\cdot,\cdot | \state)) = \Bigg\{ 
    	\weight 
    	%\PP^{\policyparams} 
    	: \sum_{(\action,\statedum)} \weight(\action,\statedum) \PP^{\policyparams}(\action,\statedum | \state) = 1, \; \weight \geq 0, \;
    	g_e(\weight, \PP^{\policyparams}) = 0, \, \forall e \in \Ee, \; f_i(\weight, \PP^{\policyparams}) \leq 0, \, \forall i \in \Ii \Bigg\}, \label{eq:risk-envelope}
    \end{equation}
    where $g_e(\weight, \PP)$ are affine functions wrt $\weight$, $f_i(\weight, \PP)$ are convex functions wrt $\weight$, and $\Ee$ (resp. $\Ii$) denotes the finite set of equality (resp. inequality) constraints.
    Furthermore, for any given $\weight \in \{ \weight : \sum_{(\action,\statedum)} \weight(\action,\statedum) = 1, \ \weight \geq 0 \}$, $g_e(\weight, p)$ and $f_i(\weight, p)$ are twice differentiable in $p$, and there exists $M>0$ such that for all $(\action,\statedum) \in \actionspace \times \statespace$ we have
    $$\max \, \Bigg\{ \max_{i\in\Ii} \abs{\diff{p(\action,\statedum)}{f_i(\weight,p)}}, \, \max_{e\in\Ee} \abs{\diff{p(\action,\statedum)}{g_e(\weight,p)}} \Bigg\} \leq M.$$
\end{assumption}

As noted by \cite{tamar2016sequential}, ``all coherent risk measures we are aware of in the literature are already captured by [that] risk envelope''. Note, however, here we use convex (which subsumes coherent) risk measures, while still keeping the structure of the risk envelope induced by this observation for coherent risk measures. \new{This risk envelope still covers most cases of convex risk measures typically of interest in the literature such as entropic risk measures, with Shannon entropy as the convex penalty $\dualriskmeas$, and more generally expected utility risk measures \citep[see e.g.][]{rockafellar2013fundamental}. Hence, we view the restrictions of \cref{assump:risk-envelope} on the explicit form of $\riskenv$ as being not too restrictive.}

We now wish to derive \DP{} equations with a view of solving problems of the form \eqref{eq:optim-problem2}.
To this end, define the \emph{value function} $\valuefunc$ as the running risk-to-go
\begin{equation}
    \valuefunc_{\timeidx}(\state;\policyparams) :=
    \riskmeas_{\timeidx} \Bigg(\costfunc_{\timeidx}^{\policyparams} +
    \riskmeas_{\timeidx+1} \bigg(\costfunc_{\timeidx+1}^{\policyparams} +
    \dots +
    \riskmeas_{\eplength-1} \big(\costfunc_{\eplength-1}^{\policyparams} \big) 
    \bigg) \Biggm| \state_{\timeidx}=\state \Bigg),
    \label{eq:value-func}
\end{equation}
for all $\state \in \statespace$ and $\timeidx \in \periodspace$.
\new{Here, as is standard in the RL literature, the value function is to be understood as a tool to evaluate the quality of a given policy at any state of the \MDP{}. In our case, it} represents the risk at a certain time when being in a specific state and following the policy $\policy^{\policyparams}$.
The \DP{} equations (\DPE{}) for a specific \stratname{} $\policy^{\policyparams}$ are
\begin{subequations}
\begin{align}
	\valuefunc_{\eplength-1}(\state;\policyparams) &=
    \riskmeas_{\eplength-1} \Big(\costfunc_{\eplength-1}^{\policyparams} \Bigm|\state_{\eplength-1}=\state \Big),
    \qquad \text{and}
    \label{eq:value-func0-1} 
    \\
	\valuefunc_{\timeidx}(\state;\policyparams) &=
    \riskmeas_{\timeidx} \Big(\costfunc_{\timeidx}^{\policyparams} +
    \valuefunc_{\timeidx+1}(\state_{\timeidx+1}^{\policyparams};\policyparams)
    \Bigm|\state_{\timeidx}=\state \Big), \label{eq:value-func0-2}
\end{align}%
\end{subequations}%
for any $\state \in \statespace$ and $\timeidx \in \periodspace \setminus \{\eplength-1\}$.
\new{The \DPE{} allows us to recursively assess the full extant of the risk associated with a fixed policy $\policy_{\policyparams}$, and in particular the recursion can be seen to include the risk associated with the current (random) cost and the one-step ahead running risk-to-go, both of which depend explicitly on the next state.}
Using the dual representation in \cref{thm:representation-thm}, the \DPE{} in \cref{eq:value-func0-1,eq:value-func0-2} may be written as
\begin{subequations}
\begin{align}
	\valuefunc_{\eplength-1} (\state; \policyparams) &= \max_{\weight \in \riskenv(\PP^{\policyparams}(\cdot,\cdot | \state_{\eplength-1} = \state))} \Bigg\{ 
	\EEweight_{\eplength-1,s} \Big[
	\costfunc_{\eplength-1}^{\policyparams}
	\Big] - \dualriskmeas_{\eplength-1}(\weight) \Bigg\}, 
	\quad \text{and}
	\label{eq:value-func1} 
	\\
	\valuefunc_{\timeidx} (\state; \policyparams)
	&= \max_{\weight \in \riskenv(\PP^{\policyparams}(\cdot,\cdot | \state_{\timeidx} = \state))} \Bigg\{ \EEweight_{\timeidx,s} \Big[
	%\underbrace{\quad
	\costfunc_{\timeidx}^{\policyparams}
	%\quad}_{\text{cost for present state}}
	+
	%\underbrace{\quad
	\valuefunc_{\timeidx+1}(\nextstates; \policyparams)
	%\quad}_{\text{risk for next state}}
	\Big] - \dualriskmeas_{\timeidx}(\weight) \Bigg\}, \label{eq:value-func2}
\end{align}
\end{subequations}
where $\EEweight_{\timeidx,\state}[\cdot]$ denotes the conditional expectation  $\EEweight[\cdot\, | \state_\timeidx=\state]$.

We aim to optimize the value function $\valuefunc$ over \stratnames{} $\policy^{\policyparams}$ using a \emph{policy gradient approach} \citep{sutton2000policy}.
Policy gradient proposes to optimize by updating parameters of the \stratname{} using the update rule $\policyparams \leftarrow \policyparams + \eta \grad{\policyparams} \valuefunc (\cdot; \policyparams)$, which requires an estimation of the gradient.
In order to obtain the gradient of $\valuefunc$, we need an additional assumption on the transition probabilities, more specifically on the \stratname{}.
\begin{assumption}
    \label{assump:log-policy}
    We suppose the logarithm of transition probabilities $\log \PP^{\policyparams}(\action,\statedum | \state)$ is a differentiable function in $\policyparams$ when $\PP^{\policyparams}(\action,\statedum | \state) \neq 0$, and its gradient wrt $\policyparams$ is bounded for any $(\action, \state) \in \actionspace \times \statespace$.
\end{assumption}

\new{\cref{assump:log-policy} is written for completeness, as it is a common restriction with policy gradient methods \citep[see e.g.][]{sutton2000policy}. It ensures that the \agentname{} chooses a differentiable policy $\policy^{\policyparams}$ so that policy gradient can be applied. Indeed, the transition probabilities depend on the \stratname{}'s parameters $\policyparams$ only through the policy $\policy^{\policyparams}$ itself -- therefore, the gradient of the log-probability of a transition $\grad{\policyparams} \log \PP^{\policyparams}$ may be represented as
\begin{equation}
\begin{split}
    \grad{\policyparams} \log \PP^{\policyparams}(\action,\statedum|  \state_{\timeidx-1}=\state) &= \grad{\policyparams} \left( \log \PP(\statedum | \state, \action)  + \log \policy^{\policyparams}(\action | \state_{\timeidx} = \state) \right) \\
    &= \grad{\policyparams} \log \policy^{\policyparams}(\action | \state_{\timeidx} = \state).
\end{split} \label{eq:logprof-MDP}
\end{equation}}

\cref{thm:grad-valuefunc} provides the gradient formulae in our proposed methodology.
\begin{theorem}[Gradient of $\valuefunc$]
	\label{thm:grad-valuefunc}
	Let \cref{assump:risk-envelope,assump:log-policy} hold. For any state $\state \in \statespace$, the gradient of the value function at \timename{} $\eplength-1$ is then
    \begin{subequations}
        \begin{equation}
        \begin{split}
        % &\grad{\policyparams} \valuefunc_{\eplength-1} (\state; \policyparams) 
        % \\
        % \!\! &=
        %\overbrace{\;
        \grad{\policyparams} \valuefunc_{\eplength-1} (\state; \policyparams) 
        &=\E^{\weight^{*}}_{\eplength-1} \Bigg[ \left(\costfunc(\state, \action_{\eplength-1}^{\policyparams}, \state_{\eplength}^{\policyparams}) - \lambda^{*} \right) \grad{\policyparams} \log \policy^{\policyparams}(\action_{\eplength-1}^{\policyparams} | \state_{\eplength-1} = \state) \Bigg]
        \\
        &\qquad
        %\;}^{\text{transition}}
        -
        %\overbrace{\;
        \Bigg. \grad{\policyparams} \dualriskmeas_{\eplength-1}(\weight^{*})
        %\;}^{\text{conjugate}}
        -
        %\underbrace{\;
        \sum_{e \in \Ee}  \Bigg( \lambda^{*,\Ee}(e)  \grad{\policyparams} g_e(\weight^{*}, \PP^{\policyparams}) \Bigg)
        %\;}_{\text{equality constraints}}
        -
        %\underbrace{\;
        \sum_{i \in \Ii}  \Bigg( \lambda^{*,\Ii}(i)  \grad{\policyparams} f_i(\weight^{*}, \PP^{\policyparams}) \Bigg),
        %\;}_{\text{inequality constraints}}.
        \end{split} \label{eq:gradient-last-period}
        \end{equation}
        and the gradient of the value function at \timenames{} $\timeidx \in \periodspace \setminus \{\eplength-1\}$ 
        %$\timeidx=\eplength-2,\ldots,0$
        is
        \begin{equation}
        \begin{split}
        % & \grad{\policyparams} \valuefunc_{\timeidx} (\state; \policyparams) \\
        % \!\! &=
        %\overbrace{\;
        \grad{\policyparams} \valuefunc_{\timeidx} (\state; \policyparams) &= \E^{\weight^{*}}_{\timeidx} \Bigg[ \left(\costfunc(\state,\action_{\timeidx}^{\policyparams}, \state_{\timeidx+1}^{\policyparams}) + \valuefunc_{\timeidx+1}(\nextstates; \policyparams) - \lambda^{*} \right) \grad{\policyparams} \log \policy^{\policyparams}(\action_{\timeidx}^{\policyparams}|\state_{\timeidx} = \state)
        %\Bigg]
        %\;}^{\text{transition}}
        +
        %\overbrace{\;
        %\E^{\weight^{*}} \Bigg[
        \grad{\policyparams} \valuefunc_{\timeidx+1}(\nextstates; \policyparams) \Bigg]
        %\;}^{\text{gradient of $\valuefunc$}}
        \\
        &\qquad
        %\;}_{\text{inequality constraints}}
        -
        %\underbrace{\;
        \Bigg. \grad{\policyparams} \dualriskmeas_{\timeidx}(\weight^{*})
        %\;}_{\text{conjugate}}
        -
        %\underbrace{\;
        \sum_{e \in \Ee}  \Bigg( \lambda^{*,\Ee}(e)  \grad{\policyparams} g_e(\weight^{*}, \PP^{\policyparams}) \Bigg)
        %\;}_{\text{equality constraints}}
        -
        %\underbrace{\;
        \sum_{i \in \Ii}  \Bigg( \lambda^{*,\Ii}(i)  \grad{\policyparams} f_i(\weight^{*}, \PP^{\policyparams}) \Bigg),
        \end{split} \label{eq:gradient-other-period}
        \end{equation}
    \end{subequations}%
    where $(\weight^{*}, \lambda^{*}, \lambda^{*,\Ee}, \lambda^{*,\Ii})$ is any saddle-point of the Lagrangian function of \cref{eq:value-func1,eq:value-func2} respectively.
\end{theorem}

\begin{proof}[Proof of \cref{thm:grad-valuefunc}]
In order to have an expression for the gradient of the value function, we first compute the gradient of the last \timename{} and then obtain the recursive relation for subsequent \timenames{}.
Using \cref{thm:representation-thm} and \cref{assump:risk-envelope}, the Lagrangian of the maximization problem in \cref{eq:value-func1} (with $\timeidx = \eplength-1$) for any state $\state \in \statespace$ can be written as 
\begin{equation}
\begin{split}
	\lagrangian^{\policyparams}(\weight, \lambda,  \lambda^{\Ee},  \lambda^{\Ii})
	&=
	\sum_{(\action,\statedum)} \weight(\action,\statedum) \PP^{\policyparams}(\action,\statedum | \state_{\eplength-1} = \state) \costfunc_{\eplength-1}(\state,\action, \statedum) - \dualriskmeas_{\eplength-1}(\weight)
	\\&\qquad
	- \lambda \left( \sum_{(\action,\statedum)} \weight(\action,\statedum) \PP^{\policyparams}(\action,\statedum |  \state_{\eplength-1} = \state) - 1 \right)
	\\&\qquad
	- 
	%\underbrace{\;
	\sum_{e \in \Ee}  \left( \lambda^{\Ee}(e)  g_e(\weight, \PP^{\policyparams}) \right)
	%\;}_{\text{equality constraints}}
	-
	%\underbrace{\;
	\sum_{i \in \Ii}  \left( \lambda^{\Ii}(i)  f_i(\weight, \PP^{\policyparams}) \right).
	%\;}_{\text{inequality constraints}}.
\end{split} \label{eq:lagrangian-last-period}
\end{equation}

By the convexity of \cref{eq:value-func1} and \cref{assump:risk-envelope}, $\lagrangian^{\policyparams}$ in \cref{eq:lagrangian-last-period} has at least one saddle-point.
We emphasize here that the saddle-points $(\weight^{*}, \lambda^{*}, \lambda^{*,\Ee}, \lambda^{*,\Ii})$ depend on the state $\state$.
\new{Using Slater's condition and the convexity of the problem, strong duality holds, i.e.
\begin{equation}
    \valuefunc_{\eplength-1}(\state; \policyparams) = \max_{\weight \geq 0} \, \min_{\lambda, \lambda^{\Ee}, \lambda^{\Ii}} \, \lagrangian^{\policyparams}(\weight, \lambda,  \lambda^{\Ee},  \lambda^{\Ii}). \label{eq:strong-duality}
\end{equation}}

We \new{then} recall a what is known in the ML literature as the ``likelihood-ratio'' trick, which states that
\begin{equation}
    p(x; \policyparams) \; \grad{\policyparams} \log \left( p(x; \policyparams) \right) = \grad{\policyparams} \, p(x;\policyparams).
    \label{eq:likelihood-trick}
\end{equation}

Next, we apply the \emph{\new{Envelope} theorem for saddle-point problems} \citep[see Theorem 4 and Corollary 5 in][]{milgrom2002envelope} -- which relies on the equidifferentiability of the objective function and the absolute continuity of its gradient. These properties are satisfied as we work under \cref{assump:risk-envelope}. Under these assumptions, \new{we have
\begin{equation*}
    \max_{\weight \geq 0} \, \min_{\lambda, \lambda^{\Ee}, \lambda^{\Ii}} \, \lagrangian^{\policyparams}(\weight, \lambda,  \lambda^{\Ee},  \lambda^{\Ii}) = \max_{\weight \geq 0} \, \min_{\lambda, \lambda^{\Ee}, \lambda^{\Ii}} \, \lagrangian^{0}(\weight, \lambda,  \lambda^{\Ee},  \lambda^{\Ii}) + \int_0^{\policyparams} \grad{\policyparams} \lagrangian^{\policyparams}(\weight, \lambda,  \lambda^{\Ee},  \lambda^{\Ii})\Big|_{\policyparams=u} \, \dee u,
\end{equation*}
which implies that} the gradient \new{wrt $\policyparams$ of the objective function $\valuefunc_{\eplength-1}(\state; \policyparams)$} equals to the gradient \new{wrt $\policyparams$} of the Lagrangian evaluated at one of its saddle-points.

Using \cref{assump:log-policy}, the \new{Envelope} theorem, the Lagrangian in \cref{eq:lagrangian-last-period}\new{, the strong duality in \cref{eq:strong-duality}} and the ``likelihood-ratio'' trick in \cref{eq:likelihood-trick},
%and the gradient in \cref{eq:logprof-MDP},
we obtain that
\begin{align}
\grad{\policyparams} \valuefunc_{\eplength-1} (\state; \policyparams) &= \grad{\policyparams} \max_{\weight \geq 0} \, \min_{\lambda, \lambda^{\Ee}, \lambda^{\Ii}} \, \lagrangian^{\policyparams}(\weight, \lambda,  \lambda^{\Ee},  \lambda^{\Ii}) \nonumber\\
&= \grad{\policyparams} \lagrangian^{\policyparams}(\weight, \lambda,  \lambda^{\Ee},  \lambda^{\Ii})
\Big|_{\weight^{*}, \lambda^{*}, \lambda^{*,\Ee}, \lambda^{*,\Ii}} \nonumber\\
&=
\sum_{(\action,\statedum)} \weight^{*}(\action,\statedum) \costfunc_{\eplength-1}(\state,\action, \statedum) \grad{\policyparams} \PP^{\policyparams}(\action,\statedum|  \state_{\eplength-1}=\state)
- \grad{\policyparams} \dualriskmeas_{\eplength-1}(\weight^{*})
\nonumber\\
&\qquad
- \sum_{(\action,\statedum)} \lambda^{*} \weight^{*}(\action,\statedum) \grad{\policyparams} \PP^{\policyparams}(\action,\statedum| \state_{\eplength-1}=\state) \nonumber\\
&\qquad
- \sum_{e \in \Ee}  \left( \lambda^{*,\Ee}(e)  \grad{\policyparams} g_e(\weight^{*}, \PP^{\policyparams}) \right)
- \sum_{i \in \Ii}  \left( \lambda^{*,\Ii}(i)  \grad{\policyparams} f_i(\weight^{*}, \PP^{\policyparams}) \right) \nonumber\\
\begin{split}
&=
%\overbrace{\;
\E^{\weight^{*}}_{\eplength-1} \Bigg[ \left(\costfunc_{\eplength-1}(\state, \action_{\eplength-1}^{\policyparams}, \state_{\eplength}^{\policyparams}) - \lambda^{*} \right) \grad{\policyparams} \log \policy^{\policyparams}(\action_{\eplength-1}^{\policyparams} | \state_{\eplength-1} = \state) \Bigg]
%\;}^{\text{transition}}
\\
&\qquad
-
%\overbrace{\;
\grad{\policyparams} \dualriskmeas_{\eplength-1}(\weight^{*})
%\;}^{\text{conjugate}}
-
%\underbrace{\;
\sum_{e \in \Ee}  \Bigg( \lambda^{*,\Ee}(e)  \grad{\policyparams} g_e(\weight^{*}, \PP^{\policyparams}) \Bigg)
%\;}_{\text{equality constraints}}
-
%\underbrace{\;
\sum_{i \in \Ii}  \Bigg( \lambda^{*,\Ii}(i)  \grad{\policyparams} f_i(\weight^{*}, \PP^{\policyparams}) \Bigg).
%\;}_{\text{inequality constraints}}.
\end{split}
%\label{eq:gradient-last-period}
\end{align}

The gradient for others \timenames{} is obtained in a similar manner.
The Lagrangian of the problem in \cref{eq:value-func2} (with $\timeidx=\eplength-2,\cdots,0$) is
\begin{equation}
\begin{split}
	\lagrangian^{\policyparams}(\weight, \lambda, \lambda^{\Ee}, \lambda^{\Ii})
	&= \sum_{(\action,\statedum)} \weight(\action,\statedum) \PP^{\policyparams}(\action,\statedum | \state_{\timeidx} = \state) \left( \costfunc_{\timeidx}(\state, \action, \statedum)
	+
	\valuefunc_{\timeidx+1}(\statedum; \policyparams)\right) - \dualriskmeas_{\timeidx}(\weight)
	\\&\qquad
	- \lambda \left( \sum_{(\action,\statedum)} \weight(\action,\statedum) \PP^{\policyparams}(\action,\statedum | \state_{\timeidx} = \state) - 1 \right)
	\\&\qquad
	- 
	%\underbrace{\;
	\sum_{e \in \Ee}  \left( \lambda^{\Ee}(e)  g_e(\weight, \PP^{\policyparams}) \right)
	%\;}_{\text{equality constraints}}
	-
	%\underbrace{\;
	\sum_{i \in \Ii}  \left( \lambda^{\Ii}(i)  f_i(\weight, \PP^{\policyparams}) \right).
	%\;}_{\text{inequality constraints}}.
\end{split} \label{eq:lagrangian-other-period}
\end{equation}
As the value function depends on the \stratname{} $\policy^{\policyparams}$, the gradient formula will have an additional term. We obtain
\begin{align}
\grad{\policyparams} \valuefunc_{\timeidx} (\state; \policyparams) &= \grad{\policyparams} \lagrangian^{\policyparams}(\weight, \lambda,  \lambda^{\Ee},  \lambda^{\Ii})
\Big|_{\weight^{*}, \lambda^{*}, \lambda^{*,\Ee}, \lambda^{*,\Ii}} \nonumber 
\\
&=
\sum_{(\action,\statedum)} \weight^{*}(\action,\statedum) \PP^{\policyparams}(\action,\statedum | \state_{\timeidx} = \state) \grad{\policyparams} \valuefunc_{\timeidx+1} (\statedum; \policyparams)
- \sum_{(\action,\statedum)} \lambda^{*} \weight^{*}(\action,\statedum) \grad{\policyparams} \PP^{\policyparams}(\action,\statedum | \state_{\timeidx} = \state)
\nonumber\\
&\qquad
+ \sum_{(\action,\statedum)} \weight^{*}(\action,\statedum) \left( \costfunc_{\timeidx}(\state, \action, \statedum)
+ \valuefunc_{\timeidx+1}(\statedum; \policyparams) \right)
\grad{\policyparams} \PP^{\policyparams}(\action,\statedum | \state_{\timeidx} = \state)
\nonumber\\
&\qquad
- \grad{\policyparams} \dualriskmeas_{\timeidx}(\weight^{*})
- \sum_{e \in \Ee}  \left( \lambda^{*,\Ee}(e)  \grad{\policyparams} g_e(\weight^{*}, \PP^{\policyparams}) \right)
- \sum_{i \in \Ii}  \left( \lambda^{*,\Ii}(i)  \grad{\policyparams} f_i(\weight^{*}, \PP^{\policyparams}) \right) \nonumber\\
\begin{split}
&=
%\overbrace{\;
\E^{\weight^{*}}_{\timeidx} \Bigg[ \left(\costfunc_{\timeidx}(\state,\action_{\timeidx}^{\policyparams}, \state_{\timeidx+1}^{\policyparams}) + \valuefunc_{\timeidx+1}(\nextstates; \policyparams) - \lambda^{*} \right) \grad{\policyparams} \log \policy^{\policyparams}(\action_{\timeidx}^{\policyparams} |\state_{\timeidx} = \state) 
%\Bigg]
%\;}^{\text{transition}}
+
%\overbrace{\;
%\E^{\weight^{*}} \Bigg[
\grad{\policyparams} \valuefunc_{\timeidx+1}(\nextstates; \policyparams) \Bigg]
%\;}^{\text{gradient of $\valuefunc$}}
\\
&\qquad
-
%\underbrace{\;
\Bigg. \grad{\policyparams} \dualriskmeas_{\timeidx}(\weight^{*})
%\;}_{\text{conjugate}}
-
%\underbrace{\;
\sum_{e \in \Ee}  \Bigg( \lambda^{*,\Ee}(e)  \grad{\policyparams} g_e(\weight^{*}, \PP^{\policyparams}) \Bigg)
%\;}_{\text{equality constraints}}
-
%\underbrace{\;
\sum_{i \in \Ii}  \Bigg( \lambda^{*,\Ii}(i)  \grad{\policyparams} f_i(\weight^{*}, \PP^{\policyparams}) \Bigg).
%\;}_{\text{inequality constraints}}
\end{split}
%\label{eq:gradient-other-period}
\end{align}

This concludes the proof.
\end{proof}

While the term $\dualriskmeas_{\timeidx}(\weight^{*})$ appears at first not to depend on the policy, and therefore the term $\grad{\policyparams}\dualriskmeas_{\timeidx}(\weight^{*})$ in \cref{thm:grad-valuefunc} appears to vanish, this is not necessarily so. To see why, let us consider convex penalties of the form $\dualriskmeas_{\timeidx}(\weight) = \EEweight_{\timeidx}[f_{\timeidx}(\weight)]$ for convex functions $f_{\timeidx}:\Lqspace \rightarrow \Reals$. In this case, using the Envelope theorem \citep{milgrom2002envelope} and the ``likelihood-ratio'' trick -- the derivation is similar to the proof of \cref{thm:grad-valuefunc} -- the gradient wrt the policy is given by
\begin{equation*}
    \grad{\policyparams} \dualriskmeas_{\timeidx}(\weight^{*}) = \E^{\weight^{*}}_{\timeidx} [f_{\timeidx}(\weight^{*}) \grad{\policyparams} \log \policy^{\policyparams}(\action_{\timeidx}^{\policyparams}|\state_{\timeidx} = \state)].
\end{equation*}
While we do not restrict to the above specific form for $\dualriskmeas_{\timeidx}$, this result illustrates why, in general, this contribution to the gradient cannot be ignored.

%!TEX root = ../main.tex

% --------------------------------------------------------------
%                         Algorithm
% --------------------------------------------------------------
\section{Actor-Critic Algorithm}
\label{sec:algorithm}

 \begin{wrapfigure}[17]{r}{0.5\textwidth}
%\begin{figure}[H]
\begin{algorithm}[H]
	\caption{Main steps}
	\label{algo:main-steps}
	\KwIn{\new{\ANN{}s for} value function $\valuefunc^{\valueparams}$ and \stratname{} $\policy^{\policyparams}$, \new{number of \episodenames{} $\Ntrajectories$, \transitionnames{} $\Ntransitions$, epochs $\NepochsPI$, mini-batch size $\NbatchsPI$}}
	\new{Set initial guesses for $\valuefunc^{\valueparams}$, $\policy^{\policyparams}$\;} 
	Initialize environment and optimizers\;
	\For{each epoch $\epochidx = 1, \ldots, \Nepochs$}{
		Simulate trajectories under $\policy^{\policyparams}$\;
		Freeze $\tilde{\policy} = \policy^{\policyparams}$\;
		\nosemic \emph{Critic}: Estimate $\valuefunc^{\valueparams}$ using $\tilde{\policy}$\;
		\pushline\dosemic\nonl \new{\cref{algo:estimate-value}$(\valuefunc^{\valueparams}, \, \tilde{\policy}, \, \Ntrajectories, \, \Ntransitions, \, \NepochsV, \, \NbatchsV)$}\;
		\popline Freeze $\tilde{\valuefunc} = \valuefunc^{\valueparams}$\;
		\nosemic \emph{Actor}: Update $\policy^{\policyparams}$ using $\tilde{\valuefunc}$\; \pushline\dosemic\nonl \new{\cref{algo:update-pi}$(\tilde{\valuefunc}, \, \policy^{\policyparams}, \, \Ntrajectories, \, \Ntransitions, \, \NepochsV, \, \NbatchsV)$}\;
		\popline Store results\;
	}
	\KwOut{Optimal
	%\stratname{}
	$\policy^{\policyparams} \approx \policy^{\policyparams^{*}}$ and $\valuefunc^{\valueparams} \approx \valuefunc(\state; \policyparams^{*})$}
\end{algorithm}
%\end{figure}
\end{wrapfigure}
In this section, we provide details on our proposed algorithm and the architecture of the implemented objects.
Our policy gradient algorithm has an actor-critic style \citep{konda2000actor}, in the sense that we must learn two functions, and we do so in an alternating fashion: (i) a value function $\valuefunc$, which plays the role of the \emph{critic}; and (ii) the \stratname{} $\policy$, which plays the role of the \emph{actor}.
Actor-critic algorithms are on-policy policy search methods that maintain an estimate of a value function, which is then used to update the \agentname{}'s \stratname{} parameters.
Such algorithms have been developed in the \RL{} community for their ability to find optimal policies using low variance gradient estimates \citep{grondman2012survey}.
To obtain an approximation of the optimal \stratname{}, we perform the steps described in \cref{algo:main-steps}.

We propose to use function approximations, more specifically \emph{neural network structures}, as they are useful when dealing with continuous state-action spaces and are known to be universal approximators.
In recent years, deep neural network modeling has shown remarkable success in approximating complex functions \citep[see e.g.][]{lecun2015deep,silver2016mastering,goodfellow2016deep}, especially in financial mathematics \citep[see e.g.][]{al2018solving,hu2019deep,casgrain2019deep,cuchiero2020generative,horvath2021deep,campbell2021deep,carmona2021deep,ning2021arbitrage, hambly2021policy}.
The use of compositions of simple functions (usually referred to as propagation and activation functions) through several layers approximates complicated functions to arbitrary accuracy (with arbitrarily large structures). Neural networks also avoid the curse of dimensionality issue of representing nonlinear functions in high dimensions.

There are several approaches for modeling the \stratname{} and value function with neural network structures.
We consider a \stratname{} $\policy$ characterized by a single (fully-connected, multi-layered feed forward) artificial neural network (\ANN{}) with parameters $\policyparams$, which takes a state $\state$ and time $\timeidx$ as inputs and outputs a distribution over the space of actions $\actionspace$.
We also suppose that the value function $\valuefunc$ is characterized by another single (fully-connected, multi-layered feed forward) \ANN{} with parameters $\valueparams$.
$\valuefunc^{\valueparams}_{\timeidx}(\state; \policyparams)$ approximates the value function when the system is in state $\state$ during the \timename{} $\timeidx$ under \stratname{} $\policy^{\policyparams}$, previously denoted $\valuefunc_{\timeidx}(\state; \policyparams)$ in \cref{eq:value-func1,eq:value-func2}.
We therefore refer to the \stratname{} and value function respectively by $\policy^{\policyparams}$ and $\valuefunc^{\valueparams}$.

\begin{wrapfigure}{r}{0.4\textwidth}
% \begin{figure}[htbp]
    \vspace{-11pt}
    \centering
    \includegraphics[width=0.4\textwidth]{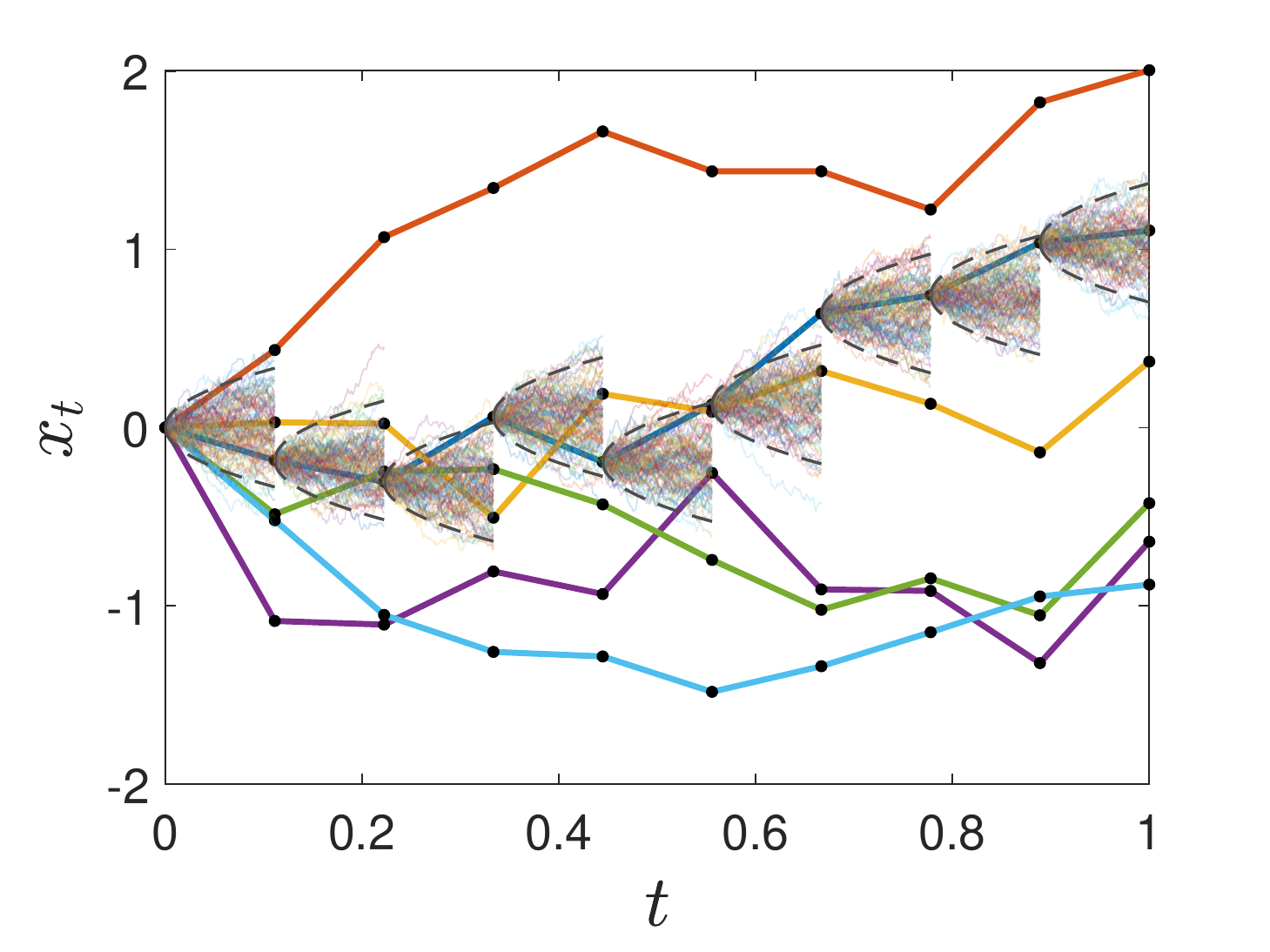}
    \caption{Representation of (outer) \episodenames{} and (inner) \transitionnames{} for the simulation upon simulation framework.}
    \label{fig:sim-upon-sim}
%  \end{figure}
\vspace*{-2em}
\end{wrapfigure}
Our proposed algorithm uses a \emph{nested simulation} or \emph{simulation upon simulation} approach, where we simulate \transitionnames{} at each visited state.
This simulation upon simulation approach results in full (outer) \episodenames{} and batches of (inner) \transitionnames{} for every state, as illustrated in \cref{fig:sim-upon-sim}.
This allows us to easily compute and estimate quantities of interest for each state, e.g. saddle-points or one-step conditional risk measures.
Such nested simulation approaches are computationally expensive -- part of our future work aims to develop a more efficient algorithm when simulations are costly.

We next provide additional details on our actor-critic algorithm in \cref{algo:main-steps}, i.e. how to (i) estimate the value function in \cref{ssec:estimate-value} (step 5), and (ii) update the \stratname{} in \cref{ssec:update-pi} (step 7).
Additional implementation details are provided in \cref{sec:implementation}.

% \input{section-files/estimate-value}
%!TEX root = ../main.tex

% --------------------------------------------------------------
%                         Estimation of the value function
% --------------------------------------------------------------
\subsection{Value Function Estimation}
\label{ssec:estimate-value}

The value function 
%$\valuefunc_{\timeidx}(\state; \policyparams)$ 
$\valuefunc^\phi$
may be estimated by using the recursion in \cref{eq:value-func1,eq:value-func2}.
We use the simulation upon simulation approach mentioned previously.
More precisely, when sampling a certain state $\state_{\timeidx}$, we also generate $\Ntransitions$ (inner) \transitionnames{} from the \stratname{} $\policy^{\policyparams}$ to obtain the tuples $(\state_{\timeidx}, \action_{\timeidx}^{(\transitionidx)}, \state_{\timeidx+1}^{(\transitionidx)}, \costfunc_{\timeidx}^{(\transitionidx)})$, $\transitionidx=1,\ldots,\Ntransitions$.
We can then estimate the one-step conditional risk measures for any given state using those additional \transitionnames{}.
To update the value function, we perform this simulation process for a mini-batch of states, compute the predicted (i.e. output of $\valuefunc^{\valueparams}$) and target values (i.e. \cref{eq:value-func1,eq:value-func2}), and calculate the expected square loss between these values.
We update the parameters using the Adam optimizer \citep{kingma2014adam} and repeat this process for several epochs in order to provide a good approximation of the value function.
We recall that the policy $\policy^{\policyparams}$ is fixed while we optimize $\valuefunc^{\valueparams}$.
The algorithm is provided in \cref{algo:estimate-value}.
\begin{algorithm}[htbp]
	\caption{Estimation of the value function $\valuefunc$}
	\label{algo:estimate-value}
	\KwIn{\new{Value function} $\valuefunc^{\valueparams}$, \new{policy} $\policy^{\policyparams}$,
	\new{number of \episodenames{} $\Ntrajectories$, \transitionnames{} $\Ntransitions$, epochs $\NepochsV$, batch size $\NbatchsV$}}
	%Freeze $\policy^{\policyparams}$ \;
	\For{each epoch $\epochVidx = 1, \ldots, \NepochsV$}{
		Set the gradients \new{of $\valuefunc^{\valueparams}$} to zero\;
		Sample $\NbatchsV$ states $\state^{(\batchVidx)}_{\timeidx}, \ \batchVidx=1,\ldots,\NbatchsV, \, \timeidx \in \periodspace$\;
		Obtain from $\policy^{\policyparams}$ the associated \transitionnames{} $(\action^{(\batchVidx,\transitionidx)}_{\timeidx},  \state^{(\batchVidx,\transitionidx)}_{\timeidx+1}, c_{\timeidx}^{(\batchVidx,\transitionidx)}), \ \transitionidx=1,\ldots,\Ntransitions$\;
		\For{each state $\batchVidx=1,\ldots,\NbatchsV$, $\timeidx\in \periodspace$}{
			Compute the \emph{predicted values} $\hat{v}^{\batchVidx}_{\timeidx} = \valuefunc^{\valueparams}_{\timeidx}(\state^{(\batchVidx)}_{\timeidx}; \policyparams)$\;
			\uIf{$\timeidx = \eplength-1$}{
				Set the \emph{target value} as 
				$$v^{\batchVidx}_{\eplength-1} =
				\max_{\weight \in \riskenv(\PP^{\policyparams}(\cdot,\cdot | \state_{\eplength-1} = \state^{(\batchVidx)}_{\eplength-1}))} \left\{ \EEweight_{\eplength-1,\state_{\eplength-1}^{(\batchVidx)}} \left[ \costfunc_{\eplength-1}^{(\batchVidx,\transitionidx)} \right] + \dualriskmeas_{\eplength-1}(\weight) \right\};$$
				\\
			}
			\uElse{
				Set the \emph{target value} as
				$$v^{\batchVidx}_{\timeidx} =
				\max_{\weight \in \riskenv(\PP^{\policyparams}(\cdot,\cdot | \state_{\timeidx} = \state^{(\batchVidx)}_{\timeidx}))} \left\{ \EEweight_{\timeidx,\state_{\timeidx}^{(\batchVidx)}} \left[ \costfunc_{\timeidx}^{(\batchVidx,\transitionidx)} + \valuefunc^{\valueparams}_{\timeidx+1}(\state^{(\batchVidx,\transitionidx)}_{\timeidx+1}; \policyparams) \right] + \dualriskmeas_{\timeidx}(\weight) \right\};$$
				\\				
			}
		}
		Compute the expected square loss between $v^{\batchVidx}_{\timeidx}$ and $\hat{v}^{\batchVidx}_{\timeidx}$\;
		Update $\valueparams$ by performing an Adam optimizer step\;
	}
	\KwOut{An estimate of the value function $\valuefunc^{\valueparams}_{\timeidx}(\state; \policyparams) \approx \valuefunc_{\timeidx}(\state; \policyparams)$}
\end{algorithm}

A powerful result for neural networks structures is the \emph{universal approximation theorem} -- see e.g. \cite{cybenko1989approximation,hornik1991approximation,leshno1993multilayer,pinkus1999approximation}.
\begin{theorem}[Universal Approximation \citep{cybenko1989approximation}]
    \label{thm:universal-approx}
    Let $d_1, d_2\in \Nats$ and $\sigma$ be an activation function.
    Then $\sigma$ is not a polynomial iff for any continuous function $f:\Reals^{d_1} \rightarrow \Reals^{d_2}$, any compact subset $K \subset \Reals^{d_1}$ and any $\epsilon > 0$, there exists a neural network $\hat{f}_{\epsilon}:\Reals^{d_1} \rightarrow \Reals^{d_2}$ with representation $\hat{f}_{\epsilon} = W_2 \circ \sigma \circ W_1$ such that $\sup_{x \in K} \norm[0]{ f(x) - \hat{f}_{\epsilon}(x) } < \epsilon$. 
\end{theorem}

We prove next that for a fixed \stratname{} $\policy^{\policyparams}$, we can approximate its corresponding value function $\valuefunc_{\timeidx}(\state; \policyparams)$ with an \ANN{} using the procedure devised in \cref{algo:estimate-value}.
\cref{thm:approx-valuefunc} follows from the universal approximation theorem.
\begin{theorem}[Approximation of $\valuefunc$]
    \label{thm:approx-valuefunc}
    Let $\policy^{\policyparams}$ denote a fixed \stratname{}, with corresponding value function as defined in \cref{eq:value-func}, which we denote $\valuefunc_{\timeidx}(\state; \policyparams)$.
    Then for any $\epsilon^{*} > 0$, there exists an \new{\ANN{}} $\valuefunc_{\timeidx}^{\valueparams}: \statespace \rightarrow \Reals$ such that $\esssup_{\state \in \statespace} \norm[0]{ \valuefunc_{\timeidx}(\state; \policyparams) - \valuefunc_{\timeidx}^{\valueparams}(\state; \policyparams) } < \epsilon^{*}$, for any $\timeidx \in \periodspace$.
\end{theorem}

\begin{proof}[Proof of \cref{thm:approx-valuefunc}]
First, we prove a lemma that states convex risk measures are \new{absolutely} continuous since they are in fact monetary.
\new{In what follows, the norm is to be understood as the $\infty$-norm.}
%This comes in handy when dealing with the universal approximation theorem.
\begin{lemma}
    \label{lemma:continuity}
    \new{Monetary} one-step conditional risk measures $\riskmeas_{\timeidx}$ are absolutely continuous a.s..
\end{lemma}
\begin{proof}[Proof of \cref{lemma:continuity}]
    Indeed, starting from the inequality $\rv \leq \rvdum + \norm{ \rv - \rvdum }$, where $\rv,\rvdum \in \Lpspace_{\timeidx+1}$, and using the monotonicity and translation invariance properties, we have
    \begin{equation}
        \riskmeas_{\timeidx}(\rv) \leq \riskmeas_{\timeidx}(\rvdum + \norm{\rv-\rvdum}) \implies \riskmeas_{\timeidx}(\rv) - \riskmeas_{\timeidx}(\rvdum) \leq \norm{\rv - \rvdum}.
    \end{equation}
    Repeating this with $\rvdum \leq \rv + \norm{\rv-\rvdum}$ yields to 
    \begin{equation}
        \esssup \norm{\riskmeas_{\timeidx}(\rv) - \riskmeas_{\timeidx}(\rvdum)} \, \leq \norm{\rv-\rvdum}.
    \end{equation}
    Therefore convex risk measures are Lipschitz continuous a.s. wrt the essential supremum norm, and hence they are absolutely continuous a.s..
\end{proof}
Next, recall that the value function given in \cref{eq:value-func} is a dynamic convex risk measure, and therefore may be written recursively with the \DPE{} in \cref{eq:value-func0-1,eq:value-func0-2} as
\begin{equation}
    \valuefunc_{\timeidx}(\state;\policyparams) =
    \riskmeas_{\timeidx} \Big(\costfunc_{\timeidx}^{\policyparams} +
    \valuefunc_{\timeidx+1}(\state_{\timeidx+1}^{\policyparams};\policyparams)
    \Bigm|\state_{\timeidx}=\state \Big).
\end{equation}

\new{We prove that the \ANN{} $\valueparams$ approximate the value function by induction.}
Without loss of generality, let us consider the case where the first dimension of the state space $\statespace$ corresponds to the time $\timeidx \in \periodspace$.
At the \timename{} $\eplength-1$, we have that
\begin{equation}
    \valuefunc_{\eplength-1}(\state;\policyparams) =
    \riskmeas_{\eplength-1} \Big(\costfunc_{\eplength-1}^{\policyparams} \Bigm|\state_{\eplength-1}=\state \Big).
\end{equation}
This is a convex risk measure which is absolutely continuous a.s. by \cref{lemma:continuity}.
Using the universal approximation theorem given in \cref{thm:universal-approx}, we obtain that \new{for any $\epsilon_{\eplength-1}>0$,} there exists a neural net $\valueparams_{\eplength-1}$ such that
\begin{equation}
    \esssup_{\state \in \statespace} \norm{ \valuefunc_{\eplength-1}(\state; \policyparams) - \valuefunc_{\eplength-1}^{\valueparams_{\eplength-1}}(\state; \policyparams) } < \epsilon_{\eplength-1}.
    \label{eq:induction-base}
\end{equation}
\new{This proves the base case of our proof by induction.}

For \new{the induction step}, the \ANN{} \new{$\valueparams_{\timeidx}$} approximates the value function at \timename{} $\timeidx$ as long as the value function at the next \timename{} $\timeidx+1$ is adequately approximated.
\new{Assume that for any $\epsilon_{\timeidx+1} > 0$, there exists an \ANN{} \new{$\valueparams_{\timeidx+1}$} such that
\begin{equation}
    \esssup_{\state \in \statespace} \norm{ \valuefunc_{\timeidx+1}(\state; \policyparams) - \valuefunc_{\timeidx+1}^{\valueparams_{\timeidx+1}}(\state; \policyparams) } < \epsilon_{\timeidx+1}.\label{eq:induction-step}
\end{equation}}
Using the translation invariance, \new{the triangle inequality, \cref{lemma:continuity} and \cref{eq:induction-step}}, we have
\new{\begin{align}
    &\esssup_{\state \in \statespace} \norm{ \valuefunc_{\timeidx}(\state; \policyparams) - \valuefunc_{\timeidx}^{\valueparams_{\timeidx}}(\state; \policyparams) } \nonumber\\
    &\quad= \esssup_{\state \in \statespace} \norm{ \riskmeas_{\timeidx}\Big( \costfunc_{\timeidx}^{\policyparams} + \valuefunc_{\timeidx+1}(\state_{\timeidx+1}^{\policyparams}; \policyparams) \Bigm| \state_{\timeidx} = \state \Big)- \valuefunc_{\timeidx}^{\valueparams_{\timeidx}}(\state; \policyparams) } \nonumber\\
    &\quad\leq \esssup_{\state \in \statespace} \norm{ \riskmeas_{\timeidx}\Big( \costfunc_{\timeidx} + \valuefunc^{\valueparams_{\timeidx+1}}_{\timeidx+1}(\state_{\timeidx+1}^{\policyparams}; \policyparams) \Bigm| \state_{\timeidx} = \state \Big)- \valuefunc_{\timeidx}^{\valueparams_{\timeidx}}(\state; \policyparams) } \nonumber\\
    &\qquad\quad + \esssup_{\state \in \statespace} \norm{ \riskmeas_{\timeidx}\Big( \costfunc_{\timeidx}^{\policyparams} + \valuefunc_{\timeidx+1}(\state_{\timeidx+1}^{\policyparams}; \policyparams) \Bigm| \state_{\timeidx} = \state \Big)- \riskmeas_{\timeidx}\Big( \costfunc_{\timeidx} + \valuefunc^{\valueparams_{\timeidx+1}}_{\timeidx+1}(\state_{\timeidx+1}^{\policyparams}; \policyparams) \Bigm| \state_{\timeidx} = \state \Big) } \nonumber\\
    &\quad\leq \esssup_{\state \in \statespace} \norm{ \riskmeas_{\timeidx}\Big( \costfunc_{\timeidx} + \valuefunc^{\valueparams_{\timeidx+1}}_{\timeidx+1}(\state_{\timeidx+1}^{\policyparams}; \policyparams) \Bigm| \state_{\timeidx} = \state \Big)- \valuefunc_{\timeidx}^{\valueparams_{\timeidx}}(\state; \policyparams) } \nonumber\\
    &\qquad\quad + \esssup_{\state \in \statespace} \norm{  \valuefunc_{\timeidx+1}(\state; \policyparams) -  \valuefunc^{\valueparams_{\timeidx+1}}_{\timeidx+1}(\state; \policyparams) } \nonumber\\ 
    &\quad< \esssup_{\state \in \statespace} \norm{ \riskmeas_{\timeidx}\Big( \costfunc_{\timeidx} + \valuefunc^{\valueparams_{\timeidx+1}}_{\timeidx+1}(\state_{\timeidx+1}^{\policyparams}; \policyparams) \Bigm| \state_{\timeidx} = \state \Big) - \valuefunc_{\timeidx}^{\valueparams_{\timeidx}}(\state; \policyparams) } + \epsilon_{\timeidx+1}. \label{eq:proof-induction1}
    %&\quad< \epsilon_{\timeidx} + \epsilon_{\timeidx+1}. \nonumber 
\end{align}Using the universal approximation theorem on \cref{eq:proof-induction1}, we get for any $\epsilon_{\timeidx}>0$, there exists an \ANN{} $\valueparams_{\timeidx}$ such that
\begin{equation}
    \esssup_{\state \in \statespace} \norm{ \valuefunc_{\timeidx}(\state; \policyparams) - \valuefunc_{\timeidx}^{\valueparams_{\timeidx}}(\state; \policyparams) } < \epsilon_{\timeidx} + \epsilon_{\timeidx+1}
\end{equation}We apply this argument recursively for any \timename{} $\timeidx \in \periodspace$, which completes the proof by induction.}

\new{Finally, we need an additional lemma to prove one can approximate the ensemble of \ANN{}s $\{\valuefunc_{\timeidx}^{\valueparams_{\timeidx}}(\state; \policyparams)\}_{\timeidx\in\periodspace}$ with a single \ANN{} $\valuefunc^{\valueparams}_{\timeidx}(\state;\policyparams)$.}
\new{\begin{lemma}
    Suppose $\{\hat{f}_{\timeidx}(x)\}_{\timeidx\in\periodspace}$, $x \in K \subseteq \Reals^{d_1}$ is an ensemble of a finite number of ANNs.
    Then for any $\epsilon > 0$, there exists an ANN $\hat{g}_{\timeidx}(x)$ such that $\esssup_{x \in K} \; \Vert \hat{f}_{\timeidx}(x) - \hat{g}_{\timeidx}(x) \Vert < \epsilon$, $\forall \, \timeidx \in \periodspace$.
    \label{lemma:ensemble}
\end{lemma}}
\begin{proof}[Proof of \cref{lemma:ensemble}]
\new{Throughout the proof, we label the ensemble using $\periodspace:=\{0,1,\ldots,\eplength\}$ without loss of generality.
We create an extension of the ensemble $\{\hat{f}_{\timeidx}(x)\}_{\timeidx\in\periodspace}$ to obtain a function $\tilde{f}_{\timeidx}(x)$ that is absolutely continuous a.s. wrt both $x$ and $\timeidx$.}

\new{First, note that, by construction, $\hat{f}_{\timeidx}(x)$'s are absolutely continuous a.s. wrt $x$.
Between each pair of indices $\timeidx$ and $\timeidx+1$, we extend the function via a polynomial interpolation on the compact set $[\timeidx, \timeidx+1]$.
Since we work on a closed interval and polynomials are continuously differentiable, the interpolation $\tilde{f}_{\timeidx}(x)$ must be Lipschitz, and thus absolutely continuous.
We then use this argument on all pairs of indices, which shows that there exists a function $\tilde{f}_{\timeidx}(x)$ such that (i) $\tilde{f}_{\timeidx}(x) = \hat{f}_{\timeidx}(x)$ for any $x \in K$ and $\timeidx \in \periodspace$, and (ii) $\tilde{f}$ is absolutely continuous on $K \times [0, \eplength]$.}

\new{Using the universal approximation theorem, for any $\epsilon > 0$, there exists an ANN $\hat{g}_{\timeidx}(x)$ such that
\begin{equation}
    \esssup_{(x,\timeidx) \in K \times [0,\eplength]} \; \Big\Vert \tilde{f}_{\timeidx}(x) - \hat{g}_{\timeidx}(x) \Big\Vert < \epsilon.
\end{equation}
This also holds for $\hat{f}_{\timeidx}(x)$, which yields the desired result.}
\end{proof}

\new{Using the triangle inequality and \cref{lemma:ensemble}, for any $\hat{\epsilon}>0$, there exists an \ANN{} $\valueparams$ such that
\begin{equation}
    \esssup_{\state \in \statespace} \norm{ \valuefunc_{\timeidx}(\state; \policyparams) - \valuefunc_{\timeidx}^{\valueparams}(\state; \policyparams) } < \hat{\epsilon} + \epsilon_{\timeidx}, \quad \forall \timeidx \in \periodspace.
\end{equation}}As \cref{thm:universal-approx} is valid for any $\epsilon>0$, we can perform the training procedure in order to construct a sequence of $\epsilon_{\timeidx}, \ \timeidx \in \periodspace$ that satisfies a global error $\epsilon^{*}$, such as \new{$\hat{\epsilon} + \sum \epsilon_{\timeidx} < \epsilon^{*}$}.
\end{proof}
%!TEX root = ../main.tex

% --------------------------------------------------------------
%                         Update of the policy
% --------------------------------------------------------------
\subsection{Update of the Policy}
\label{ssec:update-pi}

The update of the policy is done using the gradients provided in \cref{eq:gradient-last-period,eq:gradient-other-period} of \cref{thm:grad-valuefunc}.
Some points worth mentioning concern the \stratname{}, the saddle-points,
%of the risk envelope
and the gradient formula.
When implementing the algorithm, we ensure the \stratname{} uses the so-called \emph{reparametrization trick}, which allows the existence of pathwise gradient estimators from random samples.
Usually there are three basic approaches to perform a reparametrization:
\begin{enumerate}[(i)]
	\item Use a location-scale transformation -- we can view the standard random variable as an auxiliary variable $\rv$ (such as $\normaldist (0, 1)$) and simulate $\mu^{\policyparams} + \rv \sigma^{\policyparams}$ (distributed as $\normaldist (\mu^{\policyparams}, \sigma^{\policyparams})$);
	\item Use the inverse cumulative distribution function -- if it is tractable, we can use the inverse transform sampling method to simulate realizations from uniform random variables;
	\item Use a transformation of auxiliary variables -- common examples are the log-normal distribution, which can be expressed by exponentiation of a Gaussian distribution, and the gamma distribution, which can be rewritten as a sum of exponentially distributed random variables.
\end{enumerate}

Also since we assume the form of the risk envelope is known in an explicit form in \cref{assump:risk-envelope}, we can obtain a saddle-point $(\weight^{*}, \lambda^{*}, \lambda^{*,\Ee}, \lambda^{*,\Ii})$ of the Lagrangian of \cref{eq:value-func1,eq:value-func2} for any given risk measure, either analytically or using a sample average approximation \new{\citep[see Chapter 5 of][]{shapiro2014lectures}}.
The approach to obtain these saddle-points is illustrated for common risk measures in \cref{sec:experiments}.

\begin{algorithm}[htbp]
	\caption{Update of the \stratname{} $\policy$}
	\label{algo:update-pi}
	\KwIn{\new{Value function} $\valuefunc^{\valueparams}$, \new{policy} $\policy^{\policyparams}$,
	\new{number of \episodenames{} $\Ntrajectories$, \transitionnames{} $\Ntransitions$, epochs $\NepochsPI$, batch size $\NbatchsPI$}}
	%Freeze $\valuefunc^{\valueparams}$ \;
	\For{each epoch $\epochPIidx = 1, \ldots, \NepochsPI$}{
		Set the gradients \new{of $\policy^{\policyparams}$}to zero \;
		Sample $\NbatchsPI$ states $\state^{(\batchPIidx)}_{\timeidx}, \ \batchPIidx=1,\ldots,\NbatchsPI, \, \timeidx\in\periodspace$\;
		Obtain from $\policy^{\policyparams}$ the associated \transitionnames{} $(\action^{(\batchPIidx,\transitionidx)}_{\timeidx}, \state^{(\batchPIidx,\transitionidx)}_{\timeidx+1}, \costfunc_{\timeidx}^{(\batchVidx,\transitionidx)}), \ \transitionidx=1,\ldots,\Ntransitions$\;
		\For{each state $\batchPIidx=1,\ldots,\NbatchsPI$, $\timeidx\in\periodspace$}{
			Obtain $\hat{z}_{\timeidx}^{(\batchVidx,\transitionidx)} = \grad{\policyparams} \log \policy^{\policyparams}(\action^{(\batchPIidx,\transitionidx)}_{\timeidx} | \state^{(\batchPIidx)}_{\timeidx})$ with the reparametrization trick\;
			Obtain $\hat{v}_{\timeidx+1}^{(\batchVidx,\transitionidx)} = \valuefunc^{\valueparams}_{\timeidx+1}(\state^{(\batchPIidx,\transitionidx)}_{\timeidx+1}; \policyparams)$\;
			Get a saddle-point $(\weight^{*}, \lambda^{*}, \lambda^{*,\Ee}, \lambda^{*,\Ii})$ \new{and compute $\weight_{\timeidx}^{*,(b,m)} = \weight^{*}(\action_{\timeidx}^{(b,m)}, \state_{\timeidx+1}^{(b,m)})$}\;
			Obtain $\hat{\riskmeas}_{\timeidx}^{(\batchPIidx)} = \grad{\policyparams} \dualriskmeas_{\timeidx}(\weight^{*})$, $\hat{g}_{e, \timeidx}^{(\batchPIidx)} = \grad{\policyparams} g_e(\weight^{*}, \PP^{\policyparams})$, and $\hat{f}_{i, \timeidx}^{(\batchPIidx)} =  \grad{\policyparams} f_i(\weight^{*}, \PP^{\policyparams})$\;
% 			Obtain $\hat{\riskmeas}_{\timeidx}^{(\batchPIidx)} = \grad{\policyparams} \dualriskmeas_{\timeidx}(\weight^{*})$\;
			\uIf{$\timeidx = \eplength-1$}{
				Calculate the \emph{gradient} $\grad{\policyparams} \valuefunc_{\timeidx} (\state^{(\batchPIidx)}_{\timeidx}; \policyparams)$ from \cref{eq:gradient-last-period} 
				\new{$$\ell^{(\batchPIidx)}_{\timeidx} = \frac{1}{\Ntransitions} \sum_{\transitionidx=1}^{\Ntransitions} \Bigg(
    			\weight_{\timeidx}^{*,(b,m)} \Big(\costfunc_{\timeidx}^{(\batchVidx,\transitionidx)}
    			- \lambda^{*}\Big) \hat{z}_{\timeidx}^{(\batchVidx,\transitionidx)}
    			- \hat{\riskmeas}_{\timeidx}^{(\batchPIidx)}
    			- \sum_{e \in \Ee} \lambda^{*,\Ee}(e)  \hat{g}_{e, \timeidx}^{(\batchPIidx)}
    			- \sum_{i \in \Ii} \lambda^{*,\Ii}(i) \hat{f}_{i, \timeidx}^{(\batchPIidx)}
    			\Bigg);$$}
    			\\
			}
			\uElse{
				Calculate the \emph{gradient} $\grad{\policyparams} \valuefunc_{\timeidx} (\state^{(\batchPIidx)}_{\timeidx}; \policyparams)$ from \cref{eq:gradient-other-period}
				\new{$$\ell^{(\batchPIidx)}_{\timeidx} = \frac{1}{\Ntransitions} \sum_{\transitionidx=1}^{\Ntransitions} \Bigg(
			    \weight_{\timeidx}^{*,(b,m)} \Big(\costfunc_{\timeidx}^{(\batchVidx,\transitionidx)}
    			+ \hat{v}_{\timeidx+1}^{(\batchVidx,\transitionidx)} 
    			- \lambda^{*}\Big) \hat{z}_{\timeidx}^{(\batchVidx,\transitionidx)}
    			- \hat{\riskmeas}_{\timeidx}^{(\batchPIidx)}
    			- \sum_{e \in \Ee} \lambda^{*,\Ee}(e)  \hat{g}_{e, \timeidx}^{(\batchPIidx)}
    			- \sum_{i \in \Ii} \lambda^{*,\Ii}(i) \hat{f}_{i, \timeidx}^{(\batchPIidx)}
    			\Bigg);$$}
			    \\				
			}
		}
		Take the average $\ell = \frac{1}{\NbatchsPI \eplength} \sum_{\batchPIidx=1}^{\NbatchsPI} \sum_{\timeidx=0}^{\eplength-1} \ell^{(\batchPIidx)}_{\timeidx}$\;
		Update $\policyparams$ by performing an Adam optimizer step\;
	}
	\KwOut{An updated \stratname{} $\policy^{\policyparams}$}
\end{algorithm}

We recall that the value function $\valuefunc^{\valueparams}$ is fixed while we optimize $\policy^{\policyparams}$.
When we compute the gradient of the value function to optimize the \stratname{}, we fix the parameters of the value function $\valueparams$.
This can be interpreted as taking a copy of the \ANN{} structure, which implies that the value function used in the actor part of the algorithm does not depend explicitly on $\policyparams$.
Therefore, the additional expectation of the gradient of the value function at $\timeidx+1$ in \cref{eq:gradient-other-period} vanishes.
The value function gradient is then estimated averaging over a batch of states and different \timenames{}.
The algorithm is given in \cref{algo:update-pi}.

% experiments
%!TEX root = ../main.tex

% --------------------------------------------------------------
%                         Experiments
% --------------------------------------------------------------
\section{Experiments}
\label{sec:experiments}

In this section, we provide three illustrative examples to understand the potential gain of using dynamic risk measures in \RL{}, and more specifically the advantages of our proposed approach on several examples. In our experiments, we consider several risk measures in order to compare their performance and highlight their differences.\footnote{We implemented more dynamic convex risk measures in our code available on Github, and users can easily add their own in the Python files -- see \cref{sec:implementation} for a description of the code architecture.}

The first risk measure we consider is the \emph{\new{dynamic} expectation}, \new{where the one-step conditional risk measures are}  $\riskmeas_{\E}(\rv) = \E [\rv]$, which serves as a benchmark for the risk-neutral approach.
It is a convex risk measure, and its saddle-point $(\weight^{*},\lambda^{*})$ is given by $\weight^{*}(\omega) = 1$ and $\lambda^{*} = 0$.

The second risk measure is the \emph{\new{dynamic} conditional value-at-risk} (\CVaR{}) with threshold $\alpha \in (0,1)$\new{, where the one-step conditional risk maps are}
\begin{equation}
	\label{eq:CVaR}
	\riskmeas_{\CVaR}(\rv; \alpha) = \sup_{\weight \in \riskenv(\PP)} \left\{ \EEweight \left[ \rv \right] \right\},
\end{equation}
where
\begin{equation}
	\label{eq:risk-envelope-CVaR}
	\riskenv(\PP) = \left\{ \weight : \sum_{\omega} \weight(\omega) \PP(\omega) = 1, \ \weight \in \left[ 0,\frac{1}{\alpha} \right] \right\}.
\end{equation}
The \CVaR{} is a coherent risk measure widely used in the financial mathematics literature \citep{rockafellar2000optimization}.
As shown in \cite{shapiro2014lectures}, any saddle-point $(\weight^{*}, \lambda^{*})$ satisfies $\weight^{*}(\omega) = \frac{1}{\alpha}$ if $\rv (\omega) > \lambda^{*}$ and $\weight^{*}(\omega) = 0$ otherwise, where $\lambda^{*}$ is any $(1-\alpha)$-quantile of $\rv$.
\new{Here, note that despite the static \CVaR{} not being time-inconsistent (see \cite{cheridito2009time}), the dynamic \CVaR{} is time-consistent by construction.}

The third risk measure is a \emph{\new{dynamic} penalized \CVaR{}} where we add a relative entropy term with respect to the uniform distribution. This is a convex but not coherent \new{dynamic} risk measure. \new{The one-step conditional risk measures are} given by
\begin{equation}
	\label{eq:CVaR-penalized}
	\riskmeas_{\CVaR-p}(\rv; \alpha, \beta) = \sup_{\weight \in \riskenv(\PP)} \left\{ \EEweight \left[ \rv \right] - \beta\, \EEweight \left[ \log \weight \right]  \right\}, \quad \beta > 0,
\end{equation}
with the same risk envelope given in \cref{eq:risk-envelope-CVaR}.
Obtaining saddle-points is not as straightforward as with the \CVaR{}, since it requires solving a convex optimization problem.
The Lagrangian with the risk envelope constraints is
\begin{equation}
    \Ll (\weight, \lambda, \eta) = \sum_{\omega} \weight(\omega) \PP(\omega) \left( \rv(\omega) - \beta \log \weight(\omega) \right)
    - \lambda \left( \sum_{\omega} \weight(\omega) \PP(\omega) - 1\right)
    - \sum_{\omega}  \eta(\omega) \left( \weight(\omega) - \frac{1}{\alpha} \right), \label{eq:lagr-cvar-p}
\end{equation}
with $\lambda \in \Reals$ and \new{$\eta(\omega) \geq 0$} for all $\omega$.
Setting the derivative of \cref{eq:lagr-cvar-p} wrt $\weight(\omega_i)$ to zero leads to
\begin{equation}
    \label{eq:weights-1}
    \weight(\omega_i) = \exp \left( \frac{\rv(\omega_i) - \lambda - \beta}{\beta} - \frac{\eta(\omega_i)}{\beta\, \PP(\omega_i)} \right).
\end{equation}
When imposing the constraint on the $\eta$'s on \cref{eq:weights-1}, we obtain the following expression
\begin{equation}
    \label{eq:weights-2}
    \weight(\omega_i) = 
     \begin{cases}
     \exp \left( \frac{\rv(\omega_i) - \lambda - \beta}{\beta} \right) & \text{if } \eta(\omega_i) = 0 \\
     \frac{1}{\alpha} & \text{if } \eta(\omega_i) > 0
     \end{cases},
\end{equation}
and that constraint is active when $\rv(\omega_i) > -\beta \log(\alpha) + \beta + \lambda$.
We combine \cref{eq:weights-2} with the constraint on $\lambda$ to get
\begin{equation}
    \label{eq:weights-3-root}
    \sum_{i \, : \, \rv(\omega_i) \leq -\beta \log(\alpha) + \beta + \lambda}
    \PP(\omega_i) \left( \exp \left( \frac{\rv(\omega_i) - \lambda - \beta}{\beta} \right) - \frac{1}{\alpha} \right) = 1 - \frac{1}{\alpha}.
\end{equation}
Any saddle-point $(\weight^{*},\lambda^{*})$ then  satisfies $\weight^{*}(\omega) = \max(1/\alpha, \, e^{(\rv(\omega_i) - \lambda^{*} - \beta)/\beta})$, where $\lambda^{*}$ is a root of \cref{eq:weights-3-root}.

We note here that the penalized \CVaR{} contains both risk measures as special cases.
Indeed, for $\beta = 0$, it reduces to the \CVaR{}, while we recover the expectation as $\beta$ tends to $\infty$.

%!TEX root = ../main.tex

% --------------------------------------------------------------
%                         Trading problem
% --------------------------------------------------------------
\subsection{Statistical Arbitrage Example}
\label{ssec:trading-problem}

This collection of experiments is performed on an algorithmic trading environment problem.
The \agentname{} begins each \episodename{} with zero inventory, and on each \timename{} the \agentname{} wishes to trade quantities of an asset, whose price fluctuates according to some \datastratnames{}.
For each \timename{} $\timeidx \in \periodspace$, the \agentname{} observes the asset's price  $\price_{\timeidx} \in \pricespace$ and their inventory $\inventory_{\timeidx} \in (-\inventorymax, \inventorymax)$, performs a trade $\trade_{\timeidx}^{\policyparams} \in (-\trademax, \trademax)$, resulting in wealth $\wealth_{\timeidx} \in \wealthspace$ according to
%the set of equations \cref{eq:wealth}:
\begin{equation}
\label{eq:wealth}
\begin{cases}
\begin{aligned}
	\wealth_0 &= 0,
	\\
	\wealth_{\timeidx} &= \wealth_{\timeidx-1}
	- \trade_{\timeidx-1}^{\policyparams} \price_{\timeidx-1}
	- \varphi (\trade_{\timeidx-1}^{\policyparams})^2, \qquad \timeidx=1, \ldots, \eplength-1
	\\
	\wealth_{\eplength} &= \wealth_{\eplength-1}
	- \trade_{\eplength-1}^{\policyparams} \price_{\eplength-1}
	- \varphi (\trade_{\eplength-1}^{\policyparams})^2
	+ \inventory_{\eplength} \price_{\eplength}
	- \psi \inventory_{\eplength}^2,
\end{aligned},
\end{cases}
\end{equation}
with coefficients $\varphi=0.005$ and $\psi=0.5$ for the cost transactions and terminal penalty imposed by the market respectively.
We suppose that $\eplength=5$, $\inventorymax=5$, $\trademax=2$, and the asset price follows an Ornstein-Uhlenbeck process, and hence mean-reverts:
\begin{equation}
	\label{eq:price-process}
	\dee \price_{\timeidx} = \kappa (\mu - \price_{\timeidx}) \dee \timeidx + \sigma \dee W_{\timeidx},
\end{equation}
where $\kappa = 2$, $\mu = 1$, $\sigma = 0.2$ and $W_{\timeidx}$ is a standard $\PP$-Brownian motion.\footnote{Our approach is model-free, which implies that we can easily replace the asset price dynamics with more complex models, for instance including a stochastic volatility.}
\new{The risk-aware agent tries to optimize the \RL{} problem stated in \cref{eq:optim-problem1}, where} for all \timenames{} $\timeidx \in \periodspace$, the actions are determined by the trades $\trade_{\timeidx}$, the costs by the differences in wealth \new{$\costfunc_{\timeidx} = \wealth_{\timeidx} - \wealth_{\timeidx+1}$}, and the states by the tuples $(\timeidx, \price_{\timeidx}, \inventory_{\timeidx})$.

%!TEX root = ../main.tex

\begin{figure}[htbp]
	\centering
	\begin{subfigure}[b]{0.24\textwidth}
		\centering
		\includegraphics[width=0.95\textwidth]{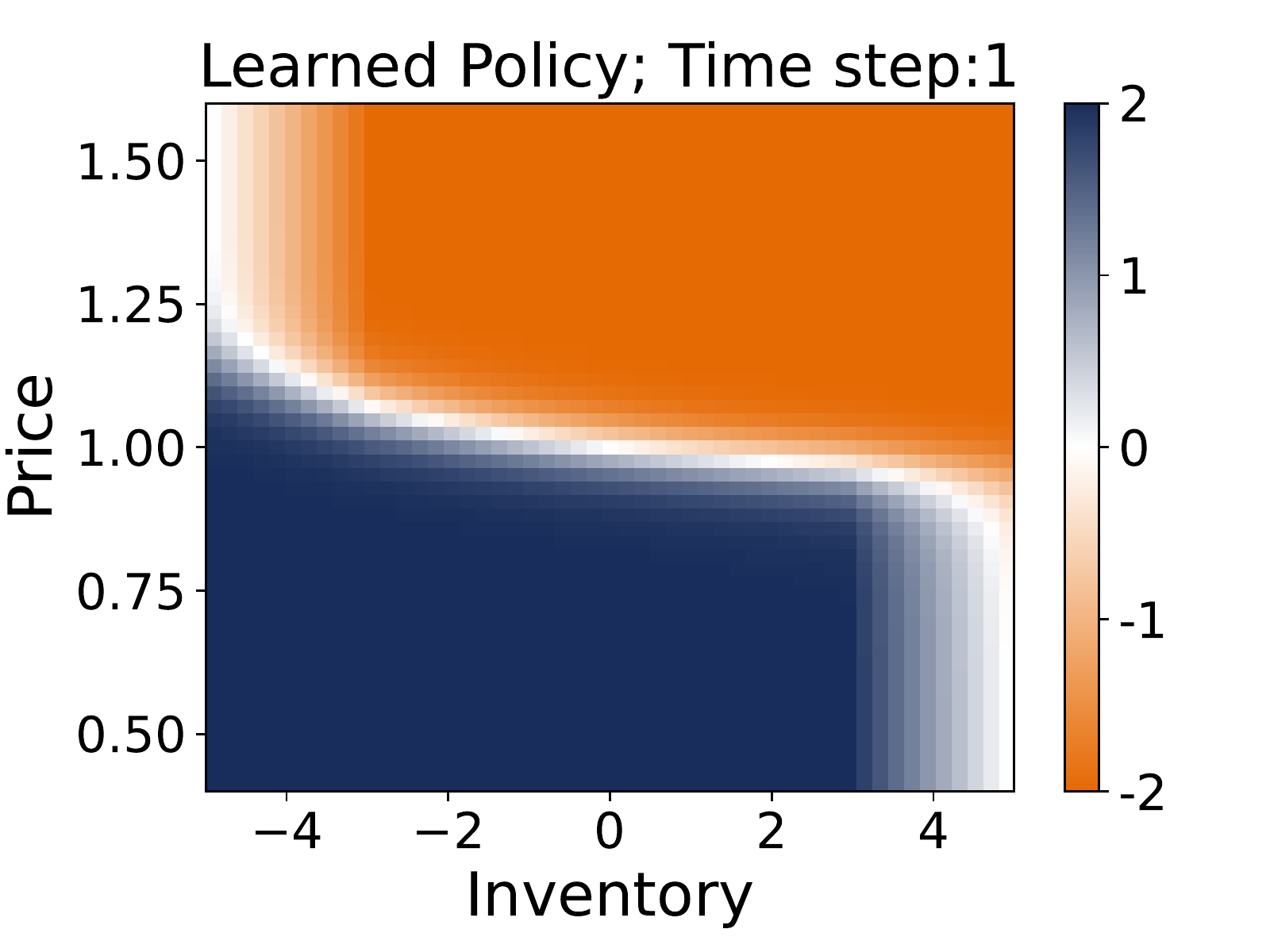}
		\caption{Mean, $t=1$}
		\label{subfig:mean_time1_learned}
	\end{subfigure}
	\hfill
	\begin{subfigure}[b]{0.24\textwidth}
		\centering
		\includegraphics[width=0.95\textwidth]{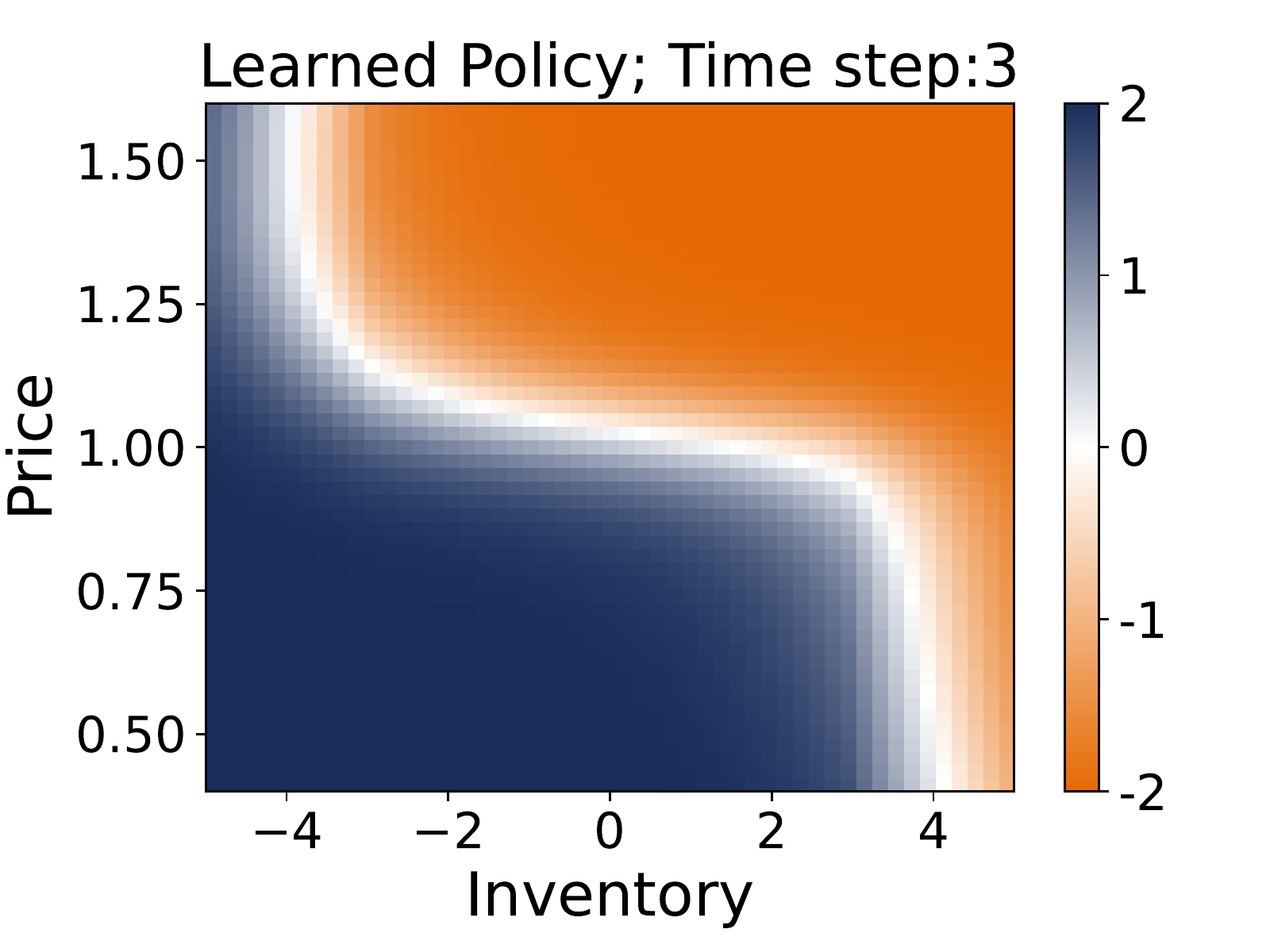}
		\caption{Mean, $t=3$}
		\label{subfig:mean_time3_learned}
	\end{subfigure}
	\hfill
	\begin{subfigure}[b]{0.24\textwidth}
		\centering
		\includegraphics[width=0.95\textwidth]{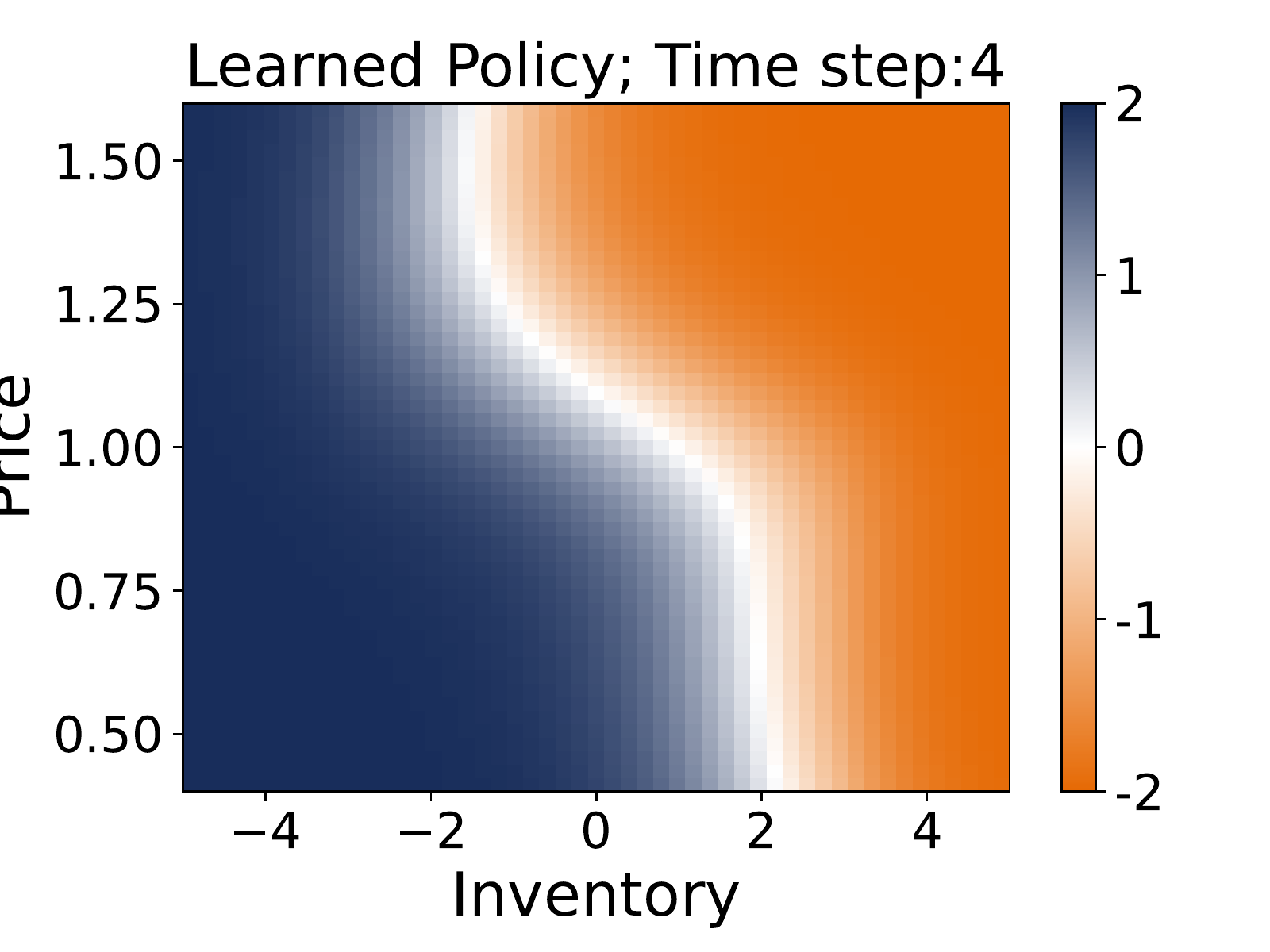}
		\caption{Mean, $t=4$}
		\label{subfig:mean_time4_learned}
	\end{subfigure}
	\hfill
	\begin{subfigure}[b]{0.24\textwidth}
		\centering
		\includegraphics[width=0.95\textwidth]{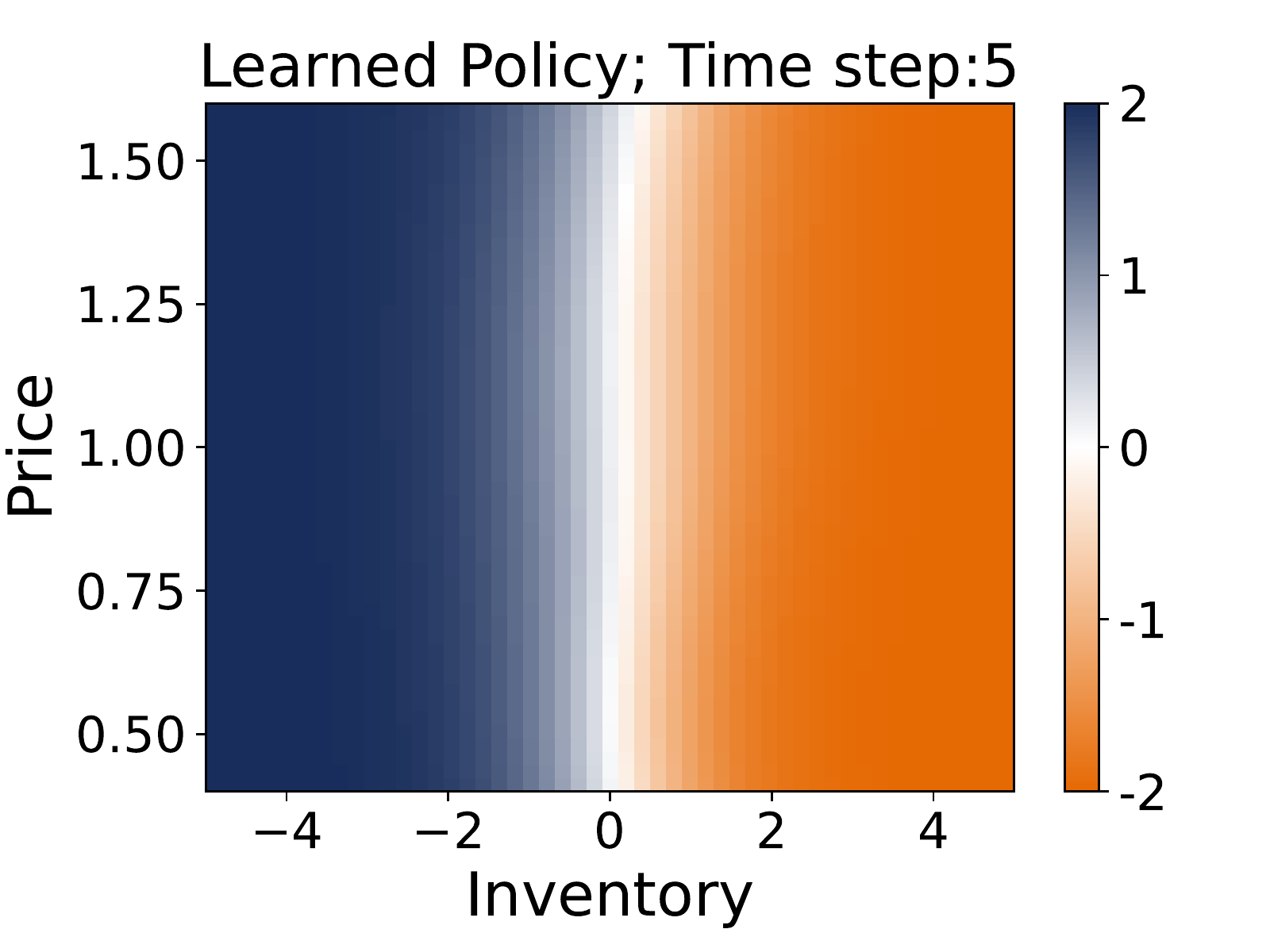}
		\caption{Mean, $t=5$}
		\label{subfig:mean_time5_learned}
	\end{subfigure}
	\vskip\baselineskip
	\begin{subfigure}[b]{0.24\textwidth}
		\centering
		\includegraphics[width=0.95\textwidth]{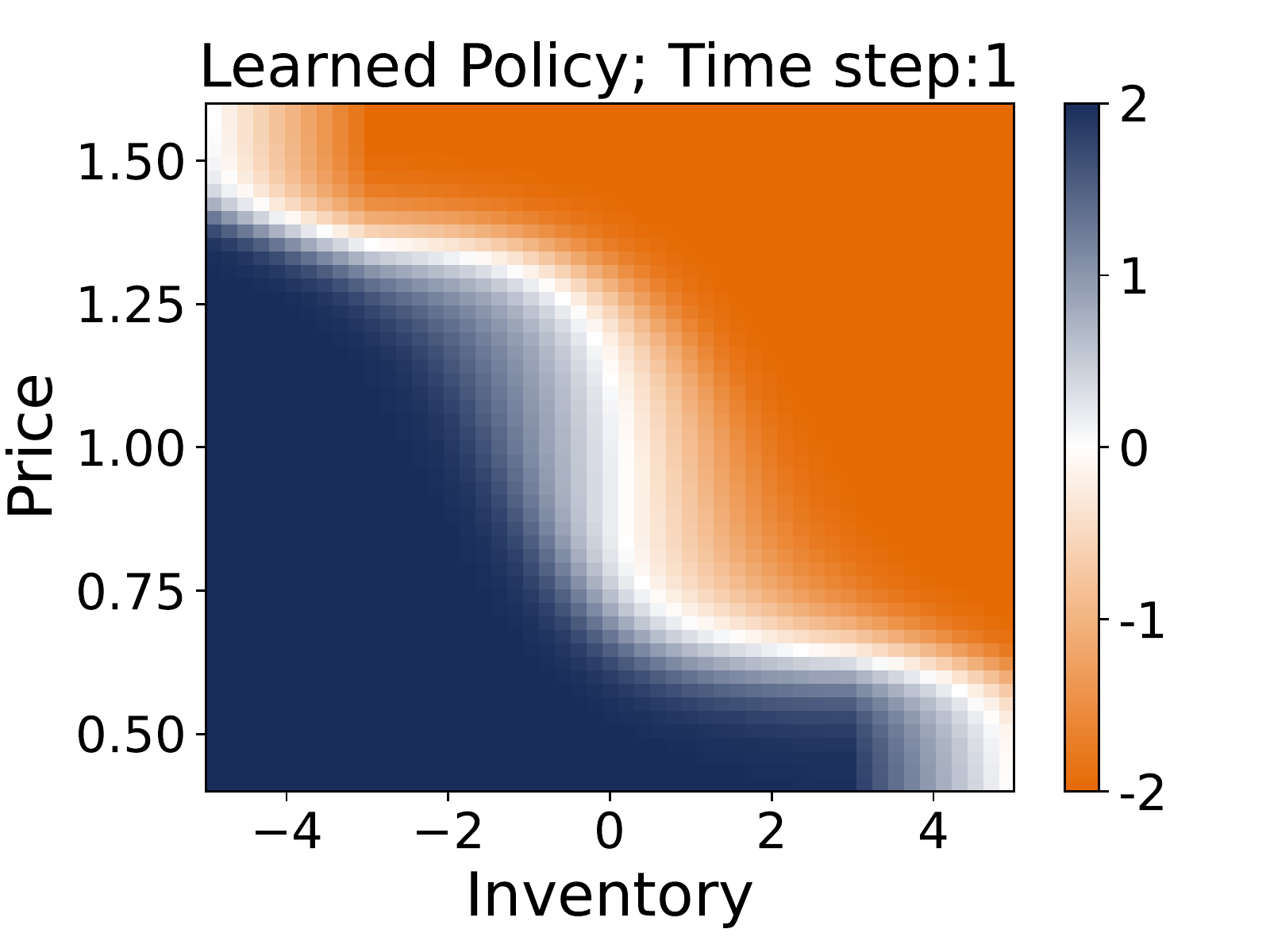}
		\caption{\CVaR{}, $t=1$}
		\label{subfig:cvar_time1_learned}
	\end{subfigure}
	\hfill
	\begin{subfigure}[b]{0.24\textwidth}
		\centering
		\includegraphics[width=0.95\textwidth]{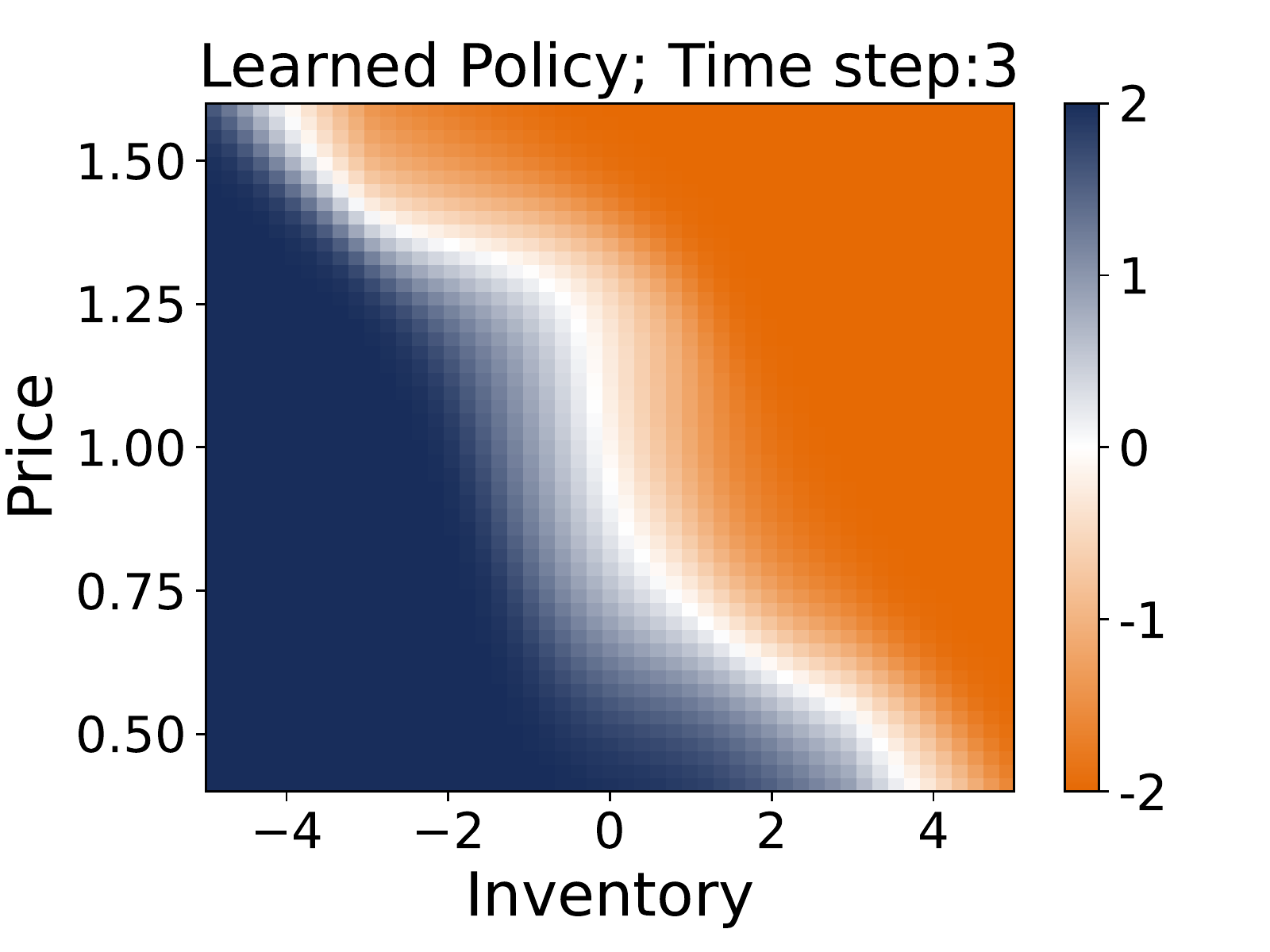}
		\caption{\CVaR{}, $t=3$}
		\label{subfig:cvar_time3_learned}
	\end{subfigure}
	\hfill
	\begin{subfigure}[b]{0.24\textwidth}
		\centering
		\includegraphics[width=0.95\textwidth]{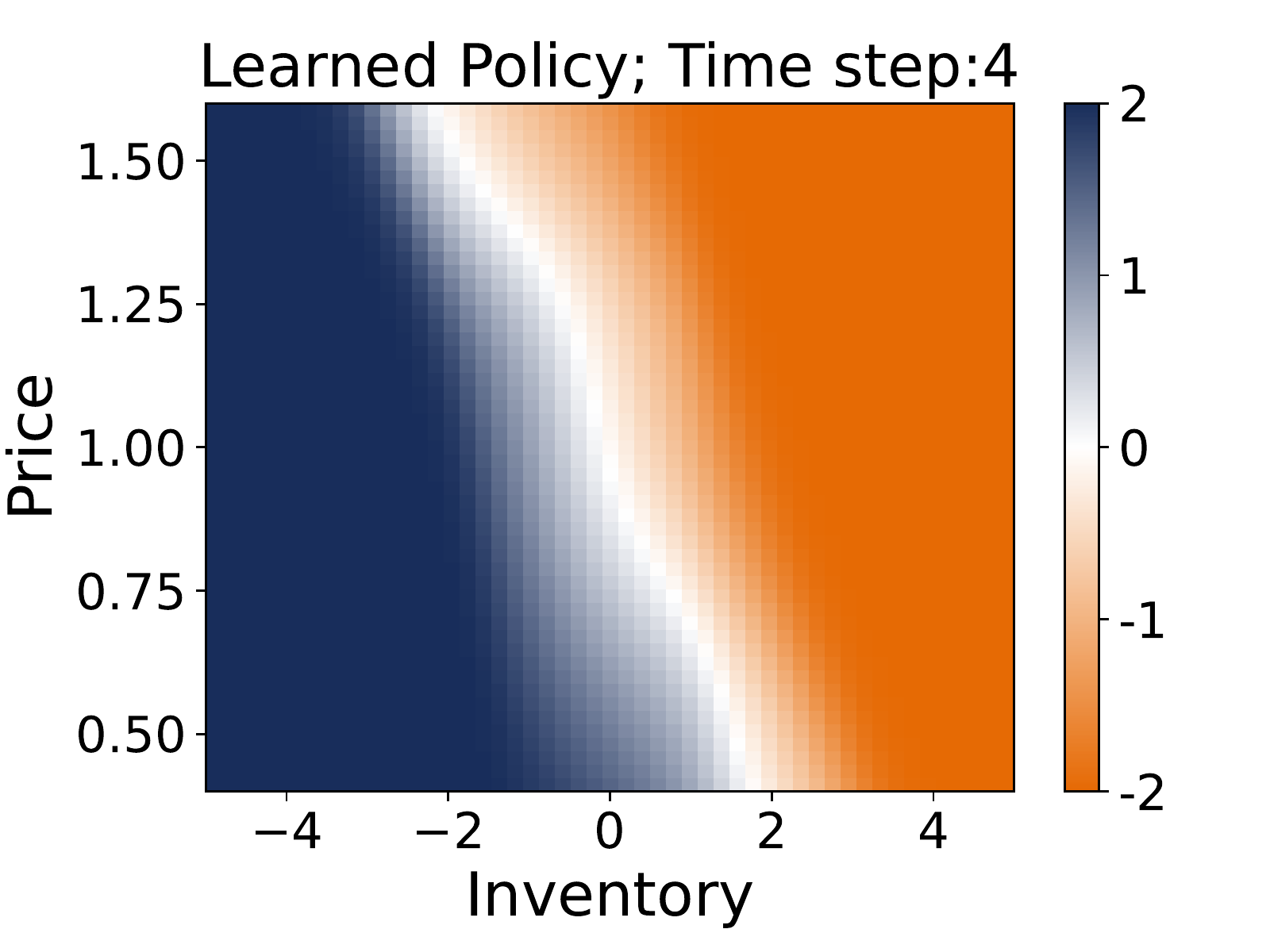}
		\caption{\CVaR{}, $t=4$}
		\label{subfig:cvar_time4_learned}
	\end{subfigure}
	\hfill
	\begin{subfigure}[b]{0.24\textwidth}
		\centering
		\includegraphics[width=0.95\textwidth]{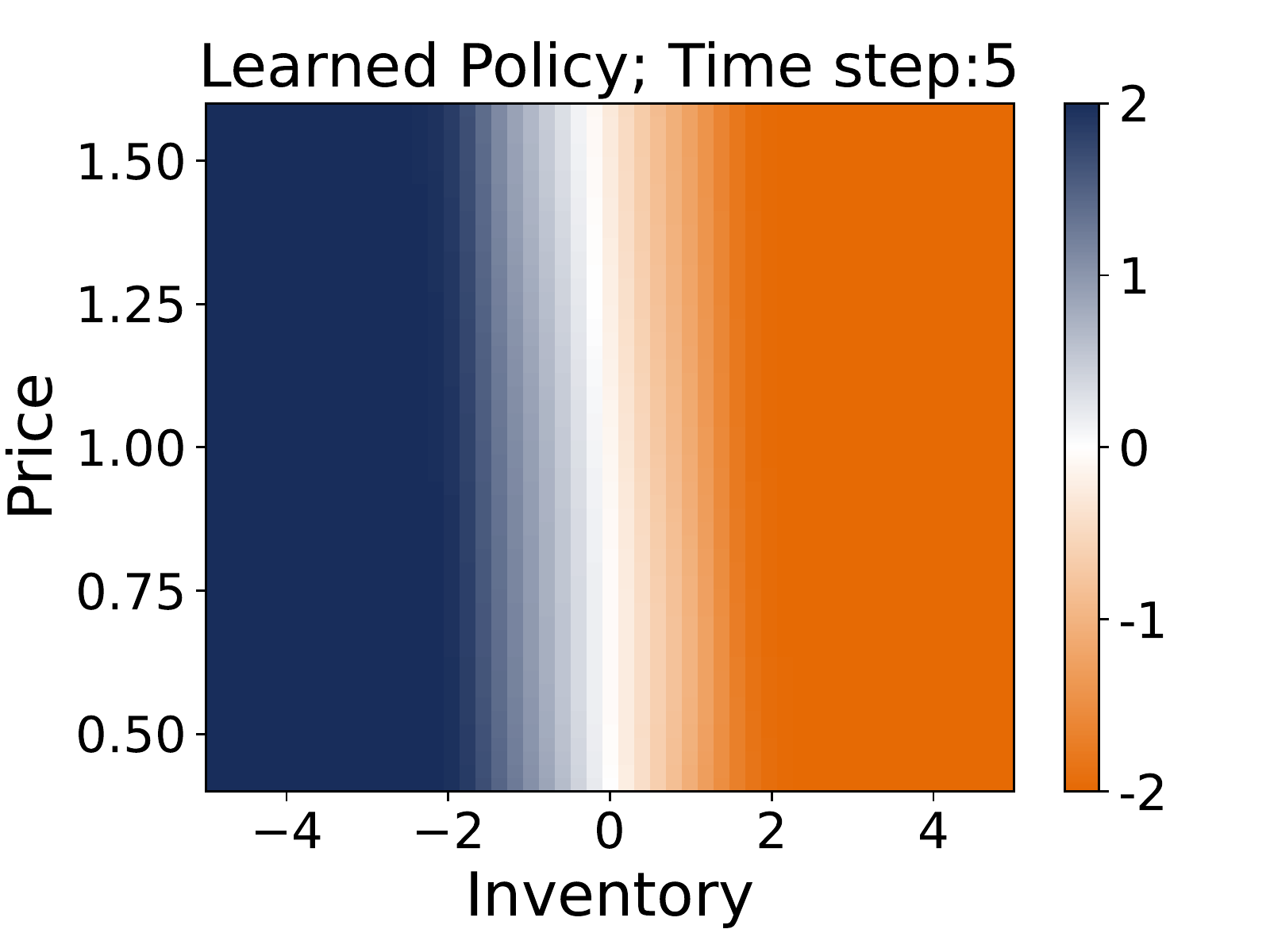}
		\caption{\CVaR{}, $t=5$}
		\label{subfig:cvar_time5_learned}
	\end{subfigure}
	\vskip\baselineskip
	\begin{subfigure}[b]{0.24\textwidth}
		\centering
		\includegraphics[width=0.95\textwidth]{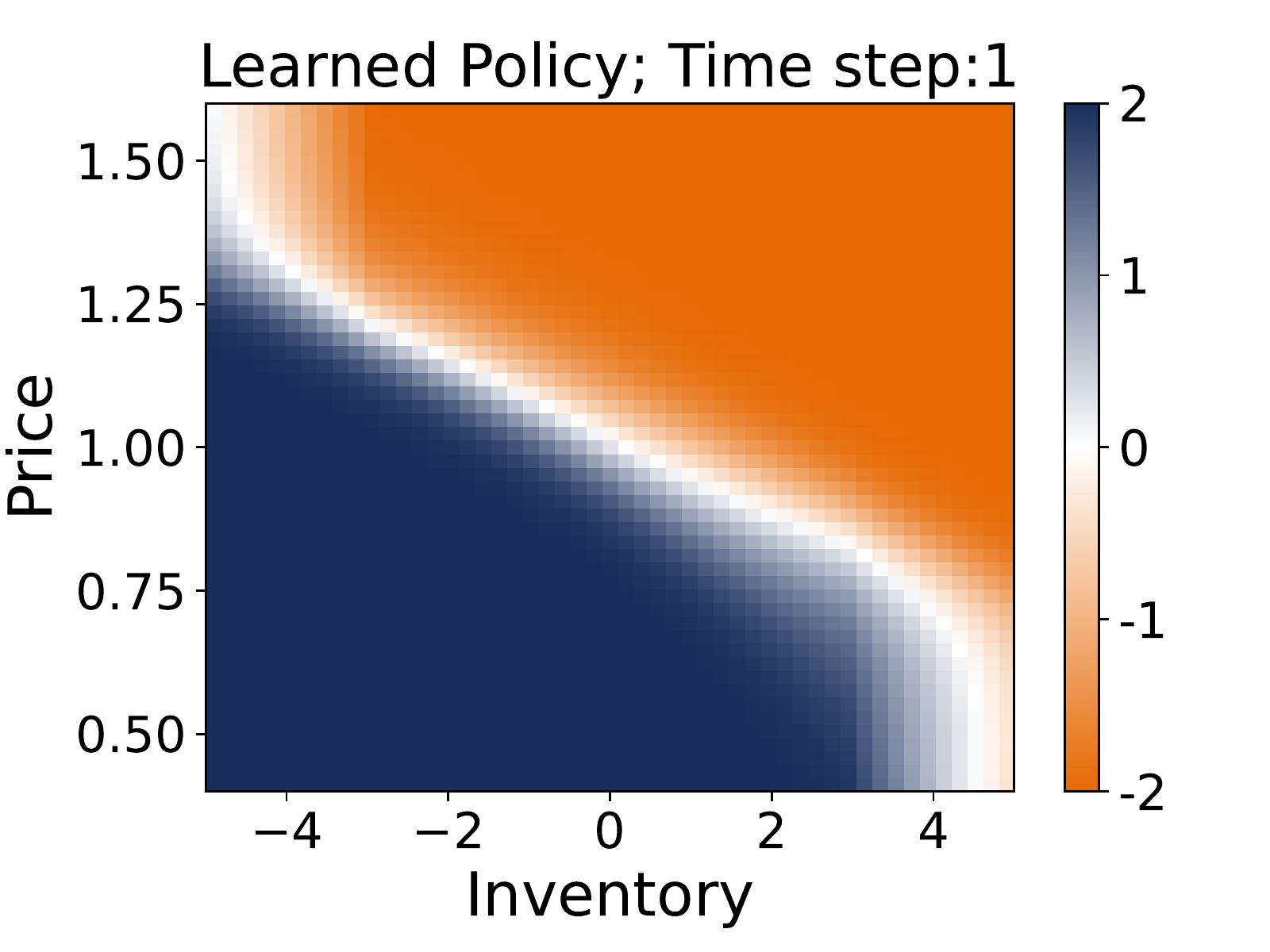}
		\caption{\CVaR{}-p, $t=1$}
		\label{subfig:cvar_pen_time1_learned}
	\end{subfigure}
	\hfill
	\begin{subfigure}[b]{0.24\textwidth}
		\centering
		\includegraphics[width=0.95\textwidth]{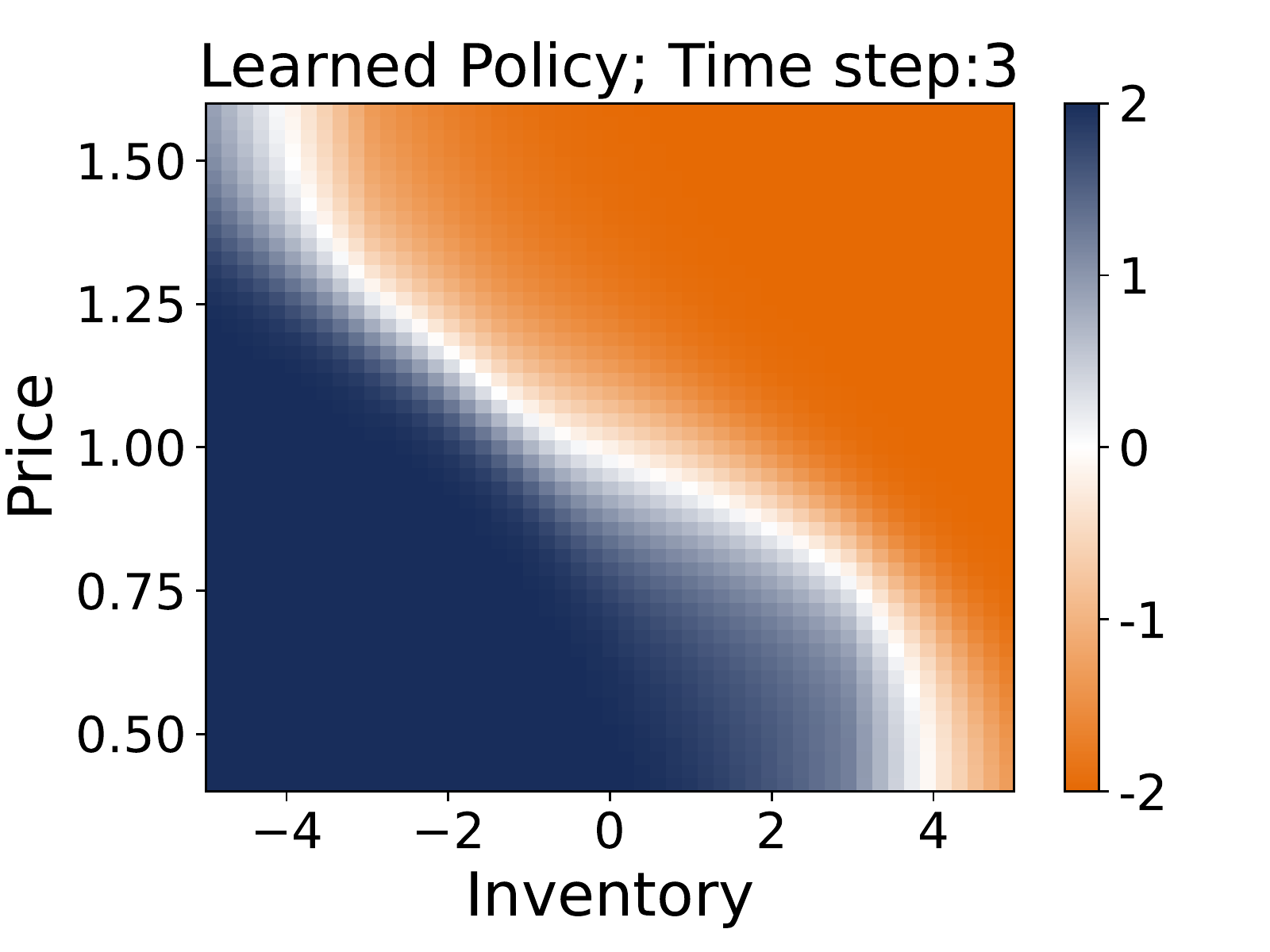}
		\caption{\CVaR{}-p, $t=3$}
		\label{subfig:cvar_pen_time3_learned}
	\end{subfigure}
	\hfill
	\begin{subfigure}[b]{0.24\textwidth}
		\centering
		\includegraphics[width=0.95\textwidth]{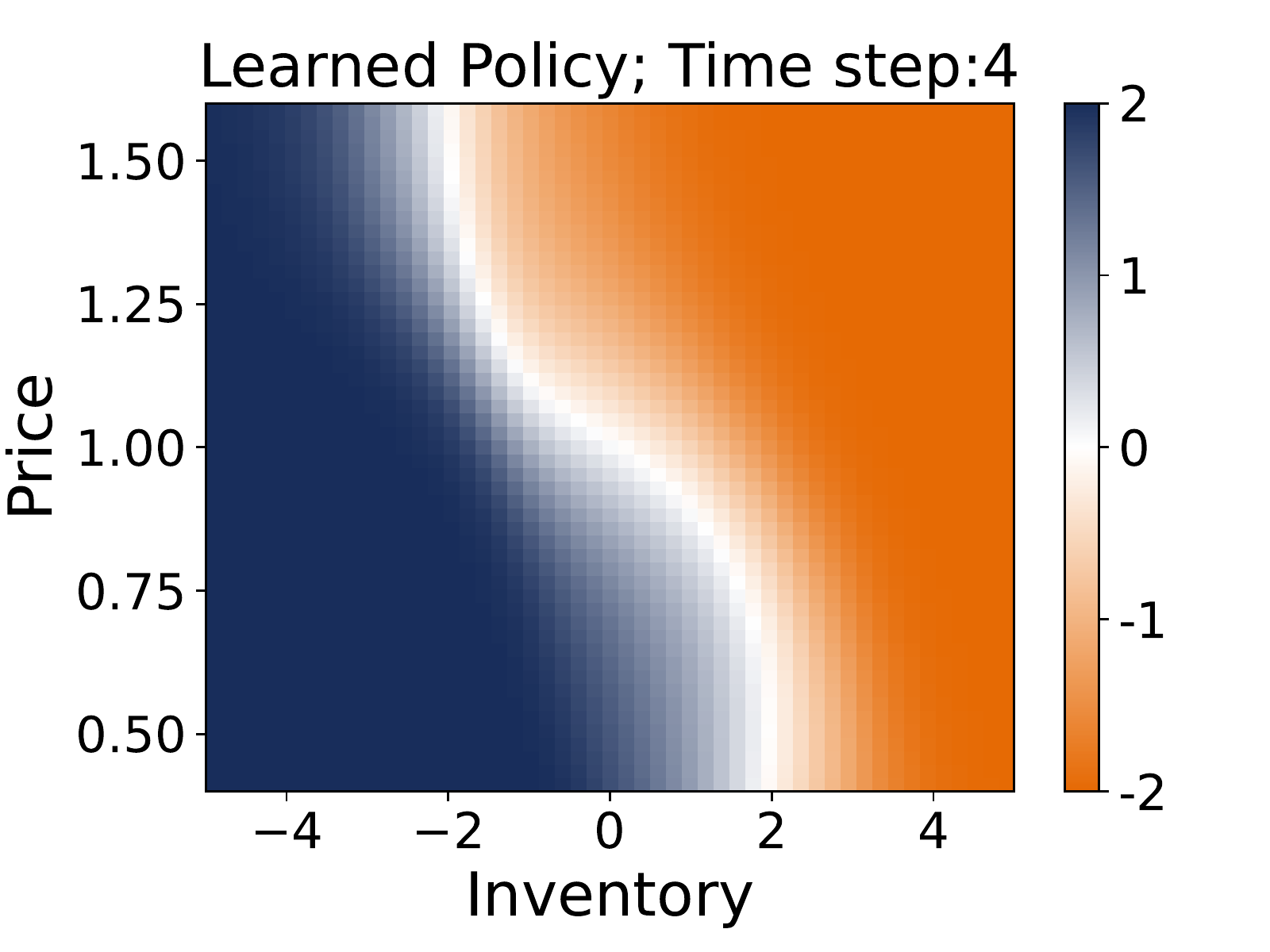}
		\caption{\CVaR{}-p, $t=4$}
		\label{subfig:cvar_pen_time4_learned}
	\end{subfigure}
	\hfill
	\begin{subfigure}[b]{0.24\textwidth}
		\centering
		\includegraphics[width=0.95\textwidth]{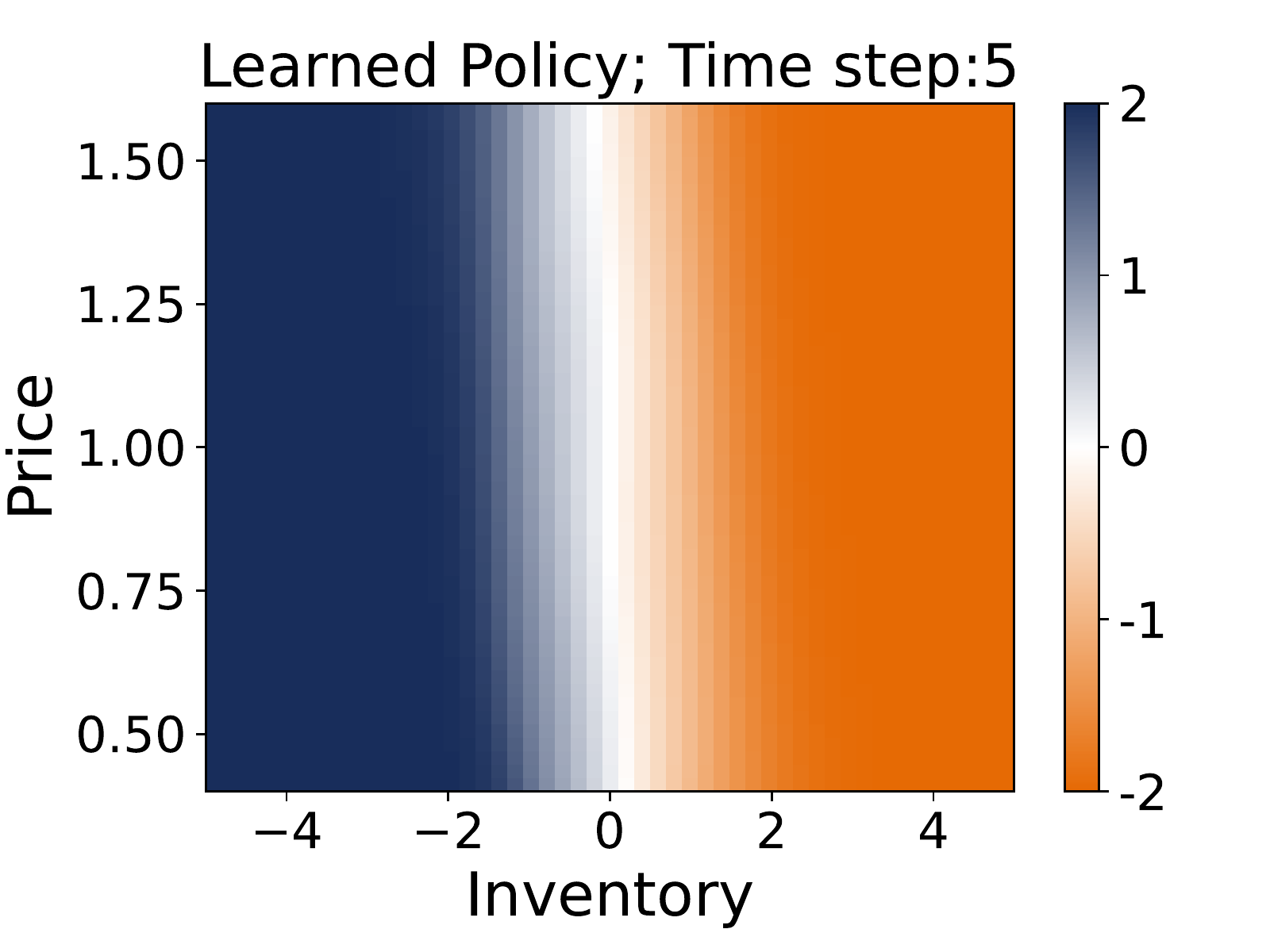}
		\caption{\CVaR{}-p, $t=5$}
		\label{subfig:cvar_pen_time5_learned}
	\end{subfigure}
	\caption{Learned \stratname{} by the actor-critic algorithm as a function of time (from left to right) when optimizing the \new{dynamic} expectation (top), the \new{dynamic} \CVaR{} with  $\alpha=0.2$ (middle) and the \new{dynamic} penalized \CVaR{} with  $\alpha=0.2$ and $\beta=0.1$ (bottom) in the statistical arbitrage example.}
% 	\caption{Learned \stratname{} by the actor-critic algorithm when optimizing the expectation $\riskmeas_{\E}(\cdot)$ (top), the \CVaR{} $\riskmeas_{\CVaR}(\cdot; \alpha=0.2)$ (middle) and the penalized \CVaR{} $\riskmeas_{\CVaR-p}(\cdot; \alpha=0.2, \beta=0.1)$ (bottom) in the trading problem.}
	\label{fig:opt-actions-trading}
\end{figure}

Figure \ref{fig:opt-actions-trading} shows a comparison of the learned \stratname{} between the \new{dynamic expectation}, the \new{dynamic} \CVaR{} \new{with $\alpha=0.2$} and the \new{dynamic} penalized \CVaR{} \new{with $\alpha=0.2$ and $\beta=0.1$}.
When optimizing a risk-neutral objective function (see \cref{subfig:mean_time1_learned,subfig:mean_time3_learned,subfig:mean_time4_learned,subfig:mean_time5_learned}), in the beginning of the \episodename{}, the \agentname{} aims to sell quantities of the asset when its price is higher than the mean-reversion level, and buy it when its price is lower.
As \timenames{} evolve, the pattern shifts to ensure that the \agentname{} concludes the \episodename{} with zero inventory to avoid the terminal penalty.
With other dynamic risk measures (see \cref{subfig:cvar_time1_learned,subfig:cvar_time3_learned,subfig:cvar_time4_learned,subfig:cvar_time5_learned,subfig:cvar_pen_time1_learned,subfig:cvar_pen_time3_learned,subfig:cvar_pen_time4_learned,subfig:cvar_pen_time5_learned}), we observe that the agent is less aggressive, and instead waits until there are more significant price deviations from the mean-reversion level before taking actions that would benefit from the price reverting back to its mean. This reflects the \emph{risk-sensitive behavior} of the \agentname{}.

The distribution of the terminal reward when the \agentname{} follows the learned \stratname{} for the \new{dynamic} mean, \CVaR{}, and penalized \CVaR{} is illustrated in \cref{fig:final-cost-dist}.
In general, risk-sensitive approaches lead to a distribution with a \emph{smaller variance} and \emph{fewer large losses}.
The \agentname{}'s tolerance to risk can be adjusted by modifying the threshold $\alpha$ of the dynamic \CVaR{} (see \cref{subfig:final-cost-cvars}) or the relative entropy penalty constant $\beta$ \new{of the dynamic penalized \CVaR{}} (see \cref{subfig:final-cost-cvar_pen1}).
When increasing $\alpha$ or $\beta$, the distribution of the terminal wealth converges to the distribution for the risk-neutral objective function, as expected.

%!TEX root = ../main.tex

\begin{figure}[htbp]
    \centering
    \begin{subfigure}[b]{0.48\textwidth}
		\centering
		\includegraphics[width=0.95\textwidth]{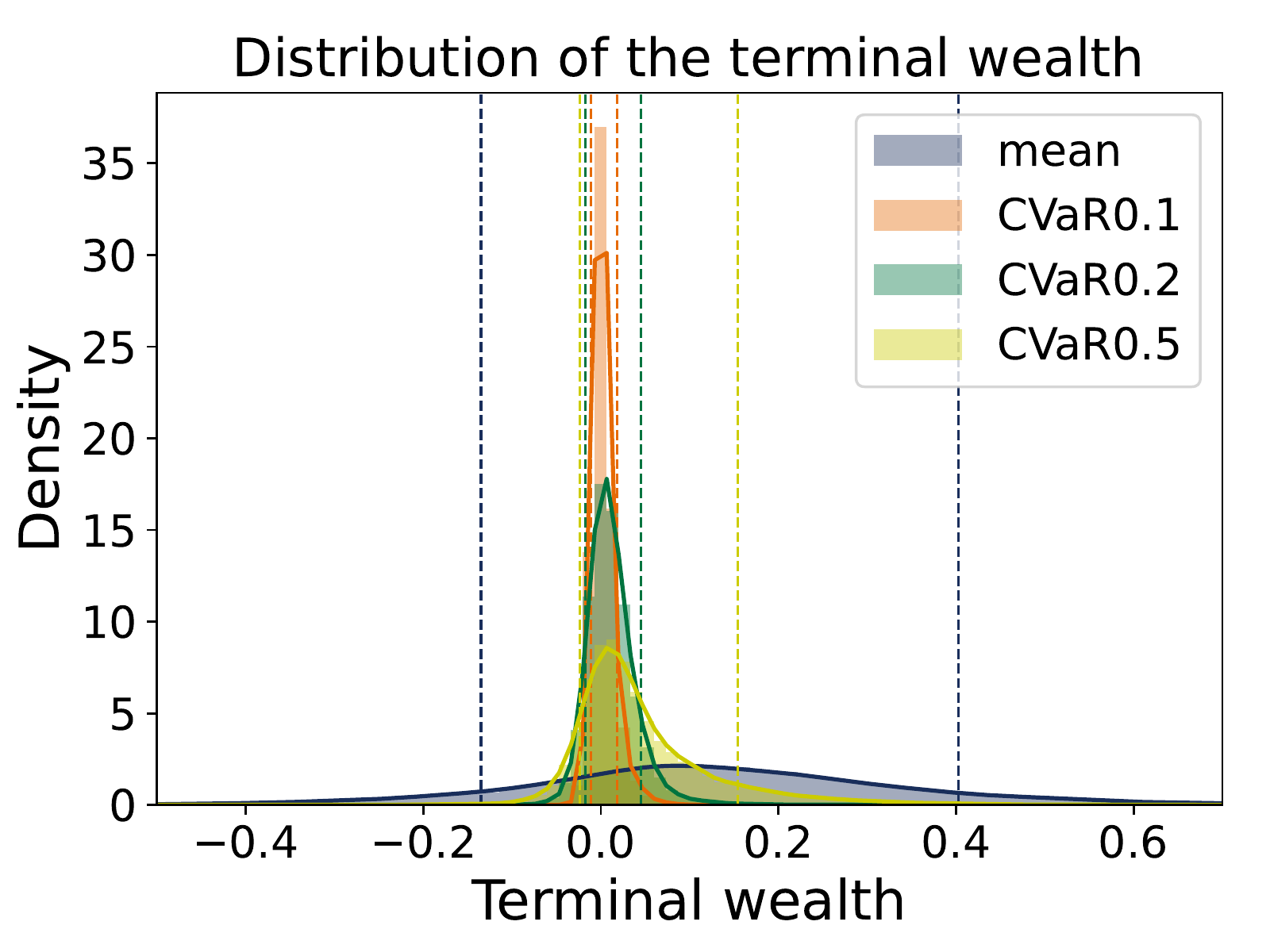}
		\caption{\new{Dynamic} \CVaR{} with various $\alpha$'s}
		\label{subfig:final-cost-cvars}
	\end{subfigure}
	\hfill
	\begin{subfigure}[b]{0.48\textwidth}
		\centering
		\includegraphics[width=0.95\textwidth]{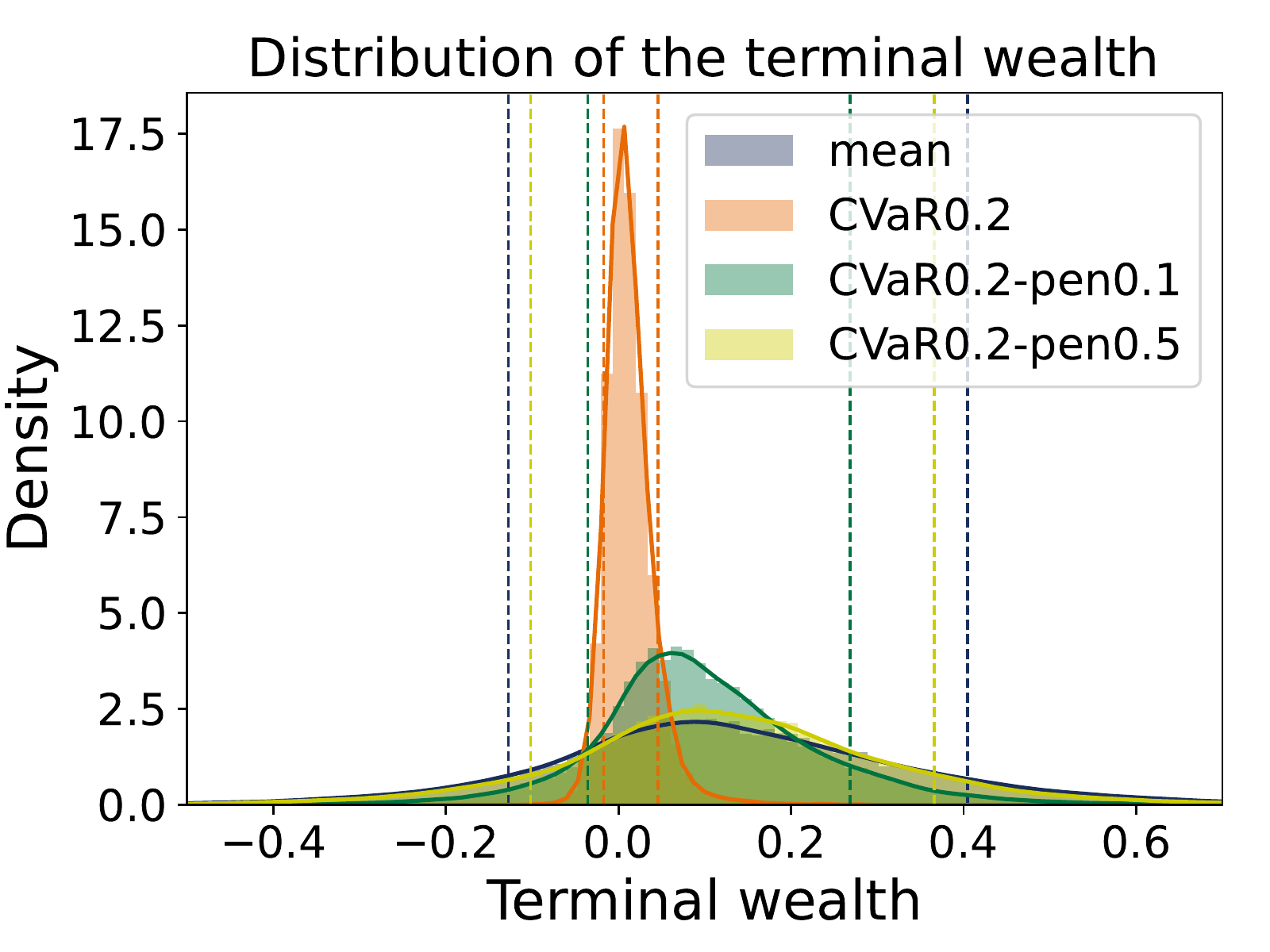}
		\caption{\new{Dynamic} penalized \CVaR{} with $\alpha=0.2$ and various $\beta$'s}
		\label{subfig:final-cost-cvar_pen1}
	\end{subfigure}
	\caption{Estimated distribution of the terminal wealth when following the learned \stratname{} in the statistical arbitrage example for several risk measures over 30,000 \episodenames{}. Vertical dashed lines indicate the $10\%$ and $90\%$ quantiles.}
	\label{fig:final-cost-dist}
\end{figure}

%!TEX root = ../main.tex

% --------------------------------------------------------------
%                         Hedging problem
% --------------------------------------------------------------

\subsection{Hedging with Friction Example}
\label{ssec:hedging-problem}

A corner stone problem of mathematical finance is the question of how to hedge the exposure to financial options.
In this section, we illustrate how our approach may be applied to hedging a call option in an environment where the underlying asset dynamics  follows the \cite{heston1993closed} model in a market with trading frictions.
In principle, one can swap out the specific stochastic volatility model for other models and/or use historical sample paths.

We denote the price of an underlying asset by $(\price_{\timeidx})_{\timeidx\in \periodspace{}}$ and the  call option's strike price by $K=10$.
We consider the case where an \agentname{} sells the call option and aims to dynamically hedge it trading solely in the underlying asset and the bank account.
We assume
%,
there are $\eplength=10$ \timenames{} (corresponding to one month).
Hence, the \agentname{} must pay $(\price_{\eplength} - K)_{+}$ at the terminal time.

We denote the stochastic variance process by $(\nu_\timeidx)_{\timeidx\in \periodspace{}}$, and, as a reminder, in the Heston model the price evolution of the underlying asset is given by
\begin{subequations}
\begin{align}
    \dee \price_{\timeidx} &= \mu \,\price_{\timeidx} \dee \timeidx + \sqrt{\nu_{\timeidx}} \,\price_{\timeidx} \,\dee W^{\price}_{\timeidx}, \\
    \dee \nu_{\timeidx} &= \kappa\, \left( \vartheta - \nu_{\timeidx} \right) \,\dee \timeidx + \varsigma \,\sqrt{\nu_{\timeidx}} \,\dee W^{\nu}_{\timeidx},
\end{align}
\end{subequations}
where $\mu=0.1$ is the drift, $\kappa=9$ the mean-reversion rate, $\vartheta=(0.25)^2$ the mean-reversion level, $\varsigma=1$ the volatility of the volatility (often referred to as vol-vol), and $(W^{\price}_{\timeidx})_{\timeidx\in \periodspace{}}, (W^{\nu}_{\timeidx})_{\timeidx\in \periodspace{}}$ are two $\PP$-Brownian motions with correlation $\rho=-0.5$ (i.e. $d[W^{\price},W^{V}]_{\timeidx} = \rho \dee \timeidx$).

For generating sample paths, we use the Milstein discretization scheme \citep{mil1975approximate} to simulate the dynamics of the stock price and volatility, where for each $\timeidx \in \periodspace$,
\begin{subequations}
\begin{align}
    \price_{\timeidx+1} &= \price_{\timeidx} \exp \left\{
        \left(\mu - \frac{1}{2} \nu_{\timeidx}^{+} \right) \Delta \timeidx
        + \sqrt{\nu_{\timeidx}^{+} \Delta \timeidx} W^{\price}_{\timeidx}
        \right\}, \\
    \nu_{\timeidx+1} &= \nu_{\timeidx} 
        + \kappa (\vartheta - \nu_{\timeidx}^{+}) \Delta \timeidx 
        + \varsigma \sqrt{\nu_{\timeidx}^{+} \Delta \timeidx} W^{\nu}_{\timeidx}
        + \Ind(\nu_{\timeidx} \geq 0) \left(\frac{1}{4} \varsigma^{2} \Delta \timeidx ((W^{\nu}_{\timeidx})^{2} - 1)\right),
\end{align}
\end{subequations}
with $\nu_{\timeidx}^{+} = \max(\nu_{\timeidx}, 0)$, $\Delta \timeidx = 1/(12\eplength)$, and the initial price and volatility respectively of $\price_0=10$ and $\nu_0=(0.2)^2$.

At each \timename{} $\timeidx \in \periodspace{}$, the \agentname{}'s wealth $\wealth_{\timeidx}$ is determined by its hedges $\action_{\timeidx}$ and bank account $\bank_{\timeidx}$.
\new{In what follows,} we assume: (i) there are market frictions in the form of transaction costs of $\varepsilon = 0.005$ (per share); (ii) the interest rate of the bank account\new{, denoted $r$,} is zero; and (iii) the \agentname{} starts with an initial wealth of $\wealth_0=B_0$.

We next describe the dynamic hedging procedure we employ and the cash-flow it induces.
For each $\timeidx \in \periodspace$, the \agentname{} takes an action $\action_{\timeidx}^{\policyparams}$, which corresponds to the number of shares to hold over the next time interval, based on its \stratname{}, and this action influences its bank account and wealth in the following manner:
\begin{subequations}
\begin{align}
    \bank_{\timeidx^{+}} &= \bank_{\timeidx} -
                        \left( \action_{\timeidx}^{\policyparams} - \action_{\timeidx-1}^{\policyparams} \right) \price_{\timeidx}
                        -
                    \left| \action_{\timeidx}^{\policyparams} - \action_{\timeidx-1}^{\policyparams} \right| \varepsilon, \\
    \wealth_{\timeidx^{+}} &= \bank_{\timeidx^{+}} + \action_{\timeidx}^{\policyparams} \price_{\timeidx}.
\end{align}
\end{subequations}
The second term in $\bank_{\timeidx^{+}}$ represents the rebalancing costs, while the third term represents a cost due to trading frictions. 
The asset price, and bank account, evolves over the next period and induces a change in the \agentname{}'s wealth process.
Thus, for each $\timeidx \in \periodspace \setminus \{\eplength-1\}$, we have the following relationships:
\begin{subequations}
\begin{align}
    \bank_{\timeidx+1} &= e^{r \Delta t} \bank_{\timeidx^{+}}, \\
    \wealth_{\timeidx+1} &= \bank_{\timeidx+1}
                + 
                \action_{\timeidx}^{\policyparams} \price_{\timeidx+1}.
\end{align}
\end{subequations}
At the end of the investment horizon, the \agentname{} must pay the call option and liquidate its inventory.
Therefore, we have
\begin{subequations}
\begin{align}
    \bank_{\eplength} &= e^{r \Delta t}                  \bank_{(\eplength-1)^{+}}
                +
                \action_{\eplength-1}^{\policyparams} \price_{\eplength} - \left| \action_{\eplength-1}^{\policyparams} \right| \varepsilon
                -
                (\price_{\eplength} - K)_{+}, \\
    \wealth_{\eplength} &= \bank_{\eplength}.
\end{align}
\end{subequations}

In our \RL{} notation, \new{the agent minimizes a dynamic convex risk measure of the costs induced by its policy, where} for all \timenames{} $\timeidx \in \periodspace$, the costs $\costfunc \in \costspace$ are given by the change in the \agentname{} 's wealth process \new{$\costfunc_{\timeidx} = \wealth_{\timeidx} - \wealth_{\timeidx+1}$}, and the states
%by the tuples $( \timeidx, \price_{\timeidx}, \action_{\timeidx-1}, \bank_{\timeidx})$.
represent all the information the \agentname{} possesses before making its hedging decision.
Depending on the option and the market assumptions, one may easily include several features into the state space, such as more asset price history, market volatility, inventories of other assets the  \agentname{} holds, and so on.
To keep the experiments brief, here, we focus on just a few features.

We consider the case where the states that determine the \agentname{}'s \stratname{} are tuples of the form $(\timeidx, \price_{\timeidx}, \action_{\timeidx-1})$.
Typically, in continuous time financial modeling for the Heston model, one assumes that the volatility process is observed and used as a feature in obtaining the optimal hedge.
Here, however, we exclude it from the \stratname{} as in real-world trading, volatility itself is not observable. 
The initial wealth $\bank_{0}$ is obtained by forcing the dynamic \new{risk (whether the dynamic \CVaR{} at level $\alpha$ or dynamic penalized \CVaR{} at level $\alpha$ with relative entropy penalty constant $\beta$)} to be $0.2$.
This specific choice is up to the agent to choose and represents the minimal reservation value the \agentname{} is willing to take. In implementation, we achieve this by initially setting the price to the Black-Scholes price using the long-run volatility $\sqrt{\vartheta}$, and then adjusting the price until the dynamic \new{risk} is $0.2$.
As \new{coherent and convex dynamic risk measures are} translation invariant (see \cref{def:coherent_riskmeasure,def:convex_riskmeasure}), the learned \stratname{} is not affected by this price adjustment.

In \cref{fig:final-cost-hedge-dynamic}, we illustrate the P\&L distribution of the optimal strategies for three different confidence levels of dynamic \CVaR{}.
Analogous to the other experiments, the figure suggests that risk-sensitive \stratnames{} lead to reward distributions that become less variable as the confidence level of CVaR increases. 
%We exclude the expectation from our results, since such an \agentname{} would only be seeking gains without mitigating risk.
In \cref{fig:bank-vs-finalprice}, we generate a scatter plot of the underlying asset's price versus the \agentname{}'s terminal bank account (corresponding to their terminal wealth).
In a situation where the \agentname{} can perfectly hedge, we expect to see the ``hockey stick'' payoff.
In the current context, however, while we do see the general ``hockey stick'' shape, there is additional convexity induced by the \agentname{} aiming to hedge in a discrete trading environment with trading frictions and (unobserved) stochastic volatility -- which is far from being a complete market. 

%!TEX root = ../main.tex

% \begin{figure}[htbp]
%     \centering
%     \begin{subfigure}[b]{0.48\textwidth}
% 		\centering
% 		\includegraphics[width=0.95\textwidth]{figure-files/placeholder.png}
% 		\caption{Label 1}
% 		\label{subfig:final-cost-hedge-dynamic}
% 	\end{subfigure}
% 	\hfill
% 	\begin{subfigure}[b]{0.48\textwidth}
% 		\centering
% 		\includegraphics[width=0.95\textwidth]{figure-files/placeholder.png}
% 		\caption{Label 2}
% 		\label{subfig:bank-vs-finalprice}
% 	\end{subfigure}
% 	\caption{Estimated distribution of the terminal wealth when following the learned \stratname{} in the hedging problem over $30,000$ \episodenames{}. Vertical dashed lines indicate the $0.1$ and $0.9$ quantiles.}
% 	\label{fig:final-cost-dist-hedge}
% \end{figure}

\begin{figure}[htbp]
    \centering
    \begin{minipage}{0.48\textwidth}
		\centering
		\includegraphics[width=0.95\textwidth]{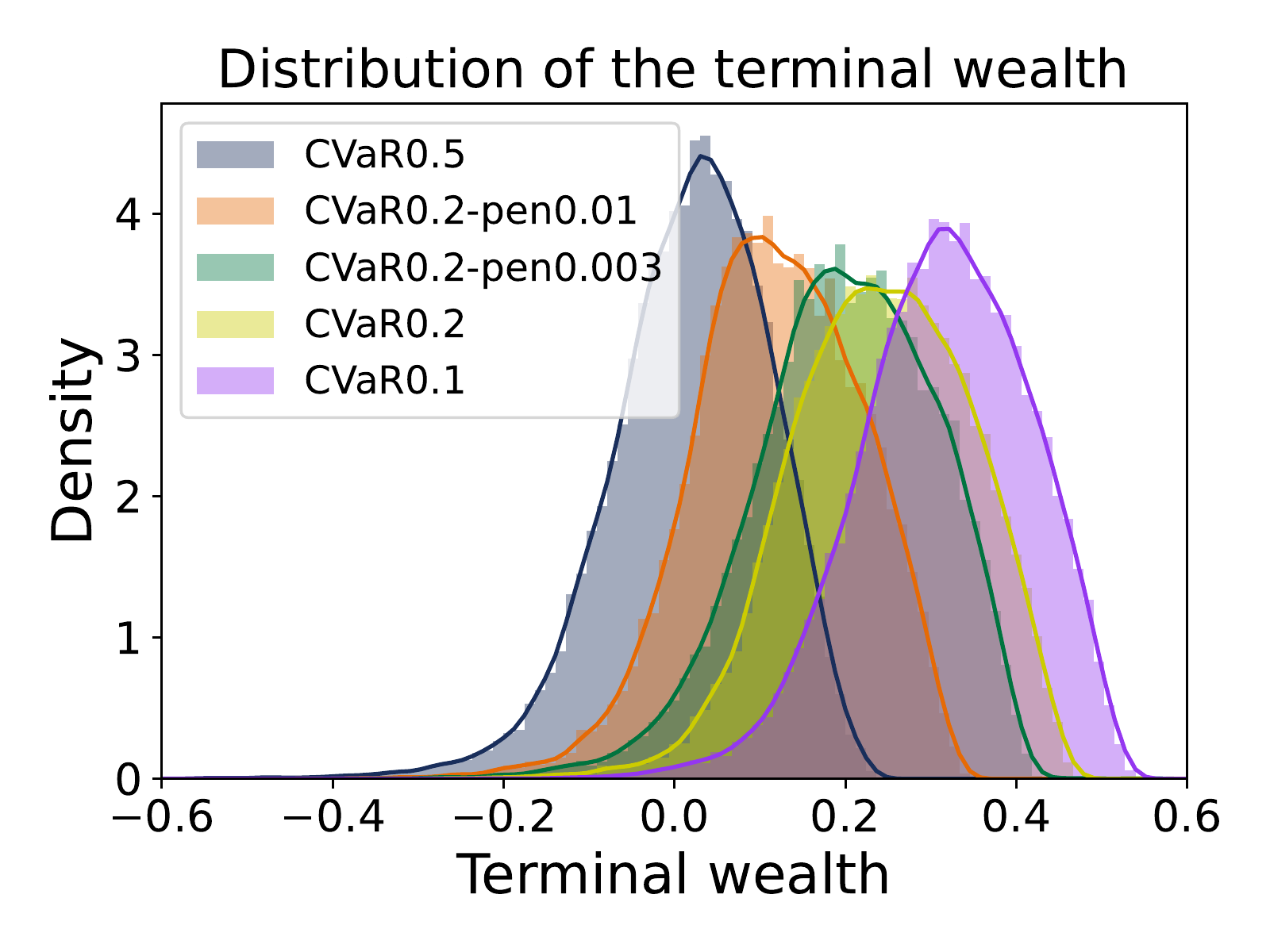}
	    \caption{Estimated distribution of the terminal wealth when following the learned \stratname{} in the hedging example over 30,000 \episodenames{}.}
		\label{fig:final-cost-hedge-dynamic}
	\end{minipage}
	\hfill
	\begin{minipage}{0.48\textwidth}
		\centering
		\includegraphics[width=0.95\textwidth]{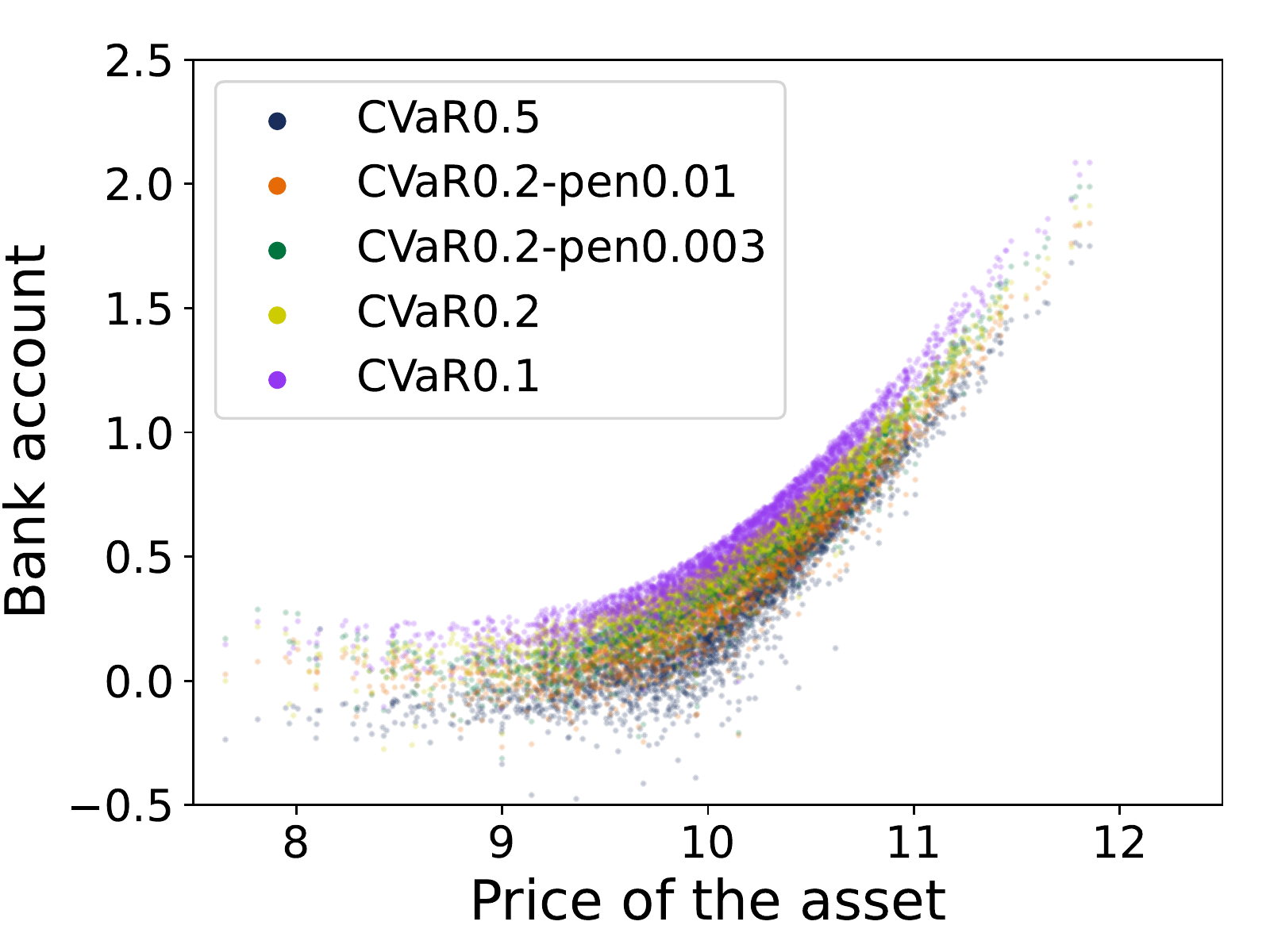}
		\caption{Scatter plot of the bank account before paying the call option against the terminal price of the asset when following the learned \stratname{} in the hedging example over 3,000 \episodenames{}.}
		\label{fig:bank-vs-finalprice}
	\end{minipage}
\end{figure}

\newpage
%!TEX root = ../main.tex

% --------------------------------------------------------------
%                         Cliff Walking
% --------------------------------------------------------------
\subsection{Cliff Walking Example}
\label{ssec:cliff-walking}

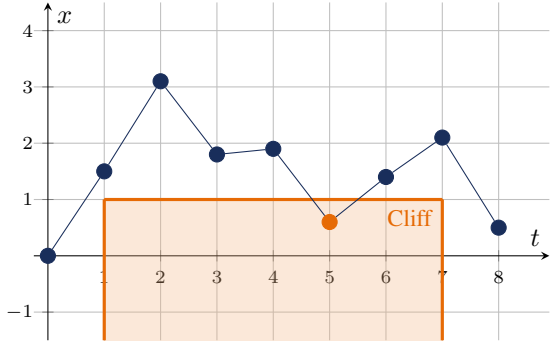
\begin{wrapfigure}{r}{0.5\textwidth}
% \begin{figure}[htbp]
\centering
\begin{tikzpicture}[
  dot/.style={
    circle,fill,draw,minimum size=2mm,inner sep=0
  }]

\begin{axis}[
    axis equal image,
    max space between ticks=20, % This is one way of getting a tick for every integer
    ticklabel style={font=\scriptsize},
    axis lines = middle,
    xmin=-0.25, xmax=8.9, % The range over which the x axis is drawn
    ymin=-1.5, ymax=4.5, % The range over which the y axis is drawn
    domain=1:7,         % The range over which the function is evaluated
    grid=both, 
    xlabel=$\timeidx$, ylabel=$x$     
]

 \shadedraw[left color=mred!50,right color=mred!50, draw=mred, fill opacity=0.3]
(1,-1.5) -- (1,1) -- (7,1) -- (7,-1.5);
 \addplot [very thick, mred] {-(0)*x+1} node[below left,font=\small]{Cliff};
 \addplot [very thick, mred,domain=1:-1.5] (1,{x});
 \addplot [very thick, mred,domain=1:-1.5] (7,{x});
 
 \node[dot, mblue, label={[font=\small]above left:$0$}] (a0) at (0,0){};
 \node[dot, mblue] (a1) at (1,1.5){};
%  \node[dot, mblue] (a2) at (2,2.2){};
 \node[dot, mblue] (a3) at (2,3.1){};
 \node[dot, mblue] (a4) at (3,1.8){};
 \node[dot, mblue] (a5) at (4,1.9){};
 \node[dot, mred] (a6) at (5,0.6){};
 \node[dot, mblue] (a7) at (6,1.4){};
 \node[dot, mblue] (a8) at (7,2.1){};
 \node[dot, mblue] (a9) at (8,0.5){};
%  \node[dot, mblue] (a10) at (10,0.5){};
 
 \draw[mblue] (a0) -- (a1);
%  \draw[mblue] (a1) -- (a2);
 \draw[mblue] (a1) -- (a3);
%  \draw[mblue] (a2) -- (a3);
 \draw[mblue] (a3) -- (a4);
 \draw[mblue] (a4) -- (a5);
 \draw[mblue] (a5) -- (a6);
 \draw[mblue] (a6) -- (a7);
 \draw[mblue] (a7) -- (a8);
 \draw[mblue] (a8) -- (a9);
%  \draw[mblue] (a9) -- (a10);
 
\end{axis}
\end{tikzpicture}

\caption{\new{Illustration of the modified cliff walking problem. One sample path (in blue) of the rover is drawn where it falls into the cliff (in orange) at $\timeidx=5$.}}
\label{fig:tikz-cliff}
% \end{figure}
\end{wrapfigure}
This set of experiments is performed on a \new{modified} version of the cliff walking problem \citep{sutton2018reinforcement} \new{with a continuous action space}, illustrated in \cref{fig:tikz-cliff}.
Consider an autonomous rover exploring the land of a new planet, represented as a
%$(G_1 \times G_2)$ grid.
Cartesian coordinate system.
The rover starts
%in the bottom left corner
at $(0,0)$ and aims to get
%in the bottom right corner
to $(\eplength,0)$, with $\eplength = 9$, while avoiding all
%cases of the bottom row
coordinates where $x \leq \clifflimit = 1.0$, illustrated as the cliff.
All allowed movements at any \timename{} $\timeidx \in \periodspace$, i.e. moving from $(\timeidx,x_1)$ to $(\timeidx+1, x_2)$, incur a cost of
%$1+|x_2-x_1|$.
$1+(x_2-x_1)^2$.
Stepping into the cliff region induces an additional cost of $100$, while landing further from the goal
%in the $G_2$-th column
at $(\eplength,x)$ induces a penalty of size
%$|x|$.
$x^2$.
%and the rover goes back to its previous state.
Actions taken by the rover are drawn from a Gaussian distribution $\action_{\timeidx}^{\policyparams} \sim \policy^{\policyparams} = \normaldist (\mu^{\policyparams}, \sigma)$, with $\mu^{\policyparams} \in (-\trademax,\trademax)$, $\trademax=4.0$ and $\sigma = 1.5$.
This represents the rover's desire to move in a certain direction $\mu^{\policyparams}$, but its movements are altered by the unknown terrain (e.g. slipping on sand, crossing shallow water, etc.).\footnote{The environment can be modified in order to place different obstacles on the coordinate system, and the policy parametrized with atypical distributions (e.g. one-sided or skewed) to illustrate different terrains.}
This introduces randomness in the \RL{} problem that the autonomous rover must account for while making decisions.
In our \RL{} notation, for all \timenames{} $\timeidx \in \periodspace$, the costs $\costfunc \in \costspace$ are determined by the rover's movements, and the states by the tuples $( \timeidx, x_{\timeidx})$.

We next explore the \emph{\agentname{}'s sensitivity to risk} -- whether the rover should reach the goal as quickly as possible by staying close to the cliff, or take a more circuitous route to avoid inadvertently falling into the cliff.
\cref{fig:preferred-path-cliff}  shows the $10\%$, $50\%$, and $90\%$ quantiles of the region visited by the rover when following the optimal strategy induced by the four dynamic risk measures. The figure illustrates that, indeed, the \agentname{}  takes into account the uncertainty of the unknown terrain by staying further and further away from the cliff as they become more risk-averse.
Notice, the optimal \stratname{} when optimizing the risk-neutral expectation induces the agent to  stay close to the cliff to avoid the costs of vertical movements.
Ultimately, using a risk-sensitive approach gives a reward distribution with mitigated tail risk for the autonomous rover, as shown in \cref{fig:final-cost-dist-cliff}.
\new{We remark that we must consider larger relative entropy penalty constants for the dynamic penalized \CVaR{} than in \cref{ssec:trading-problem,ssec:hedging-problem} due to the large magnitude of the costs in this example.}

%!TEX root = ../main.tex

\begin{figure}[htbp]
    \centering
    \begin{minipage}{0.48\textwidth}
      \centering
      \includegraphics[width=0.95\textwidth]{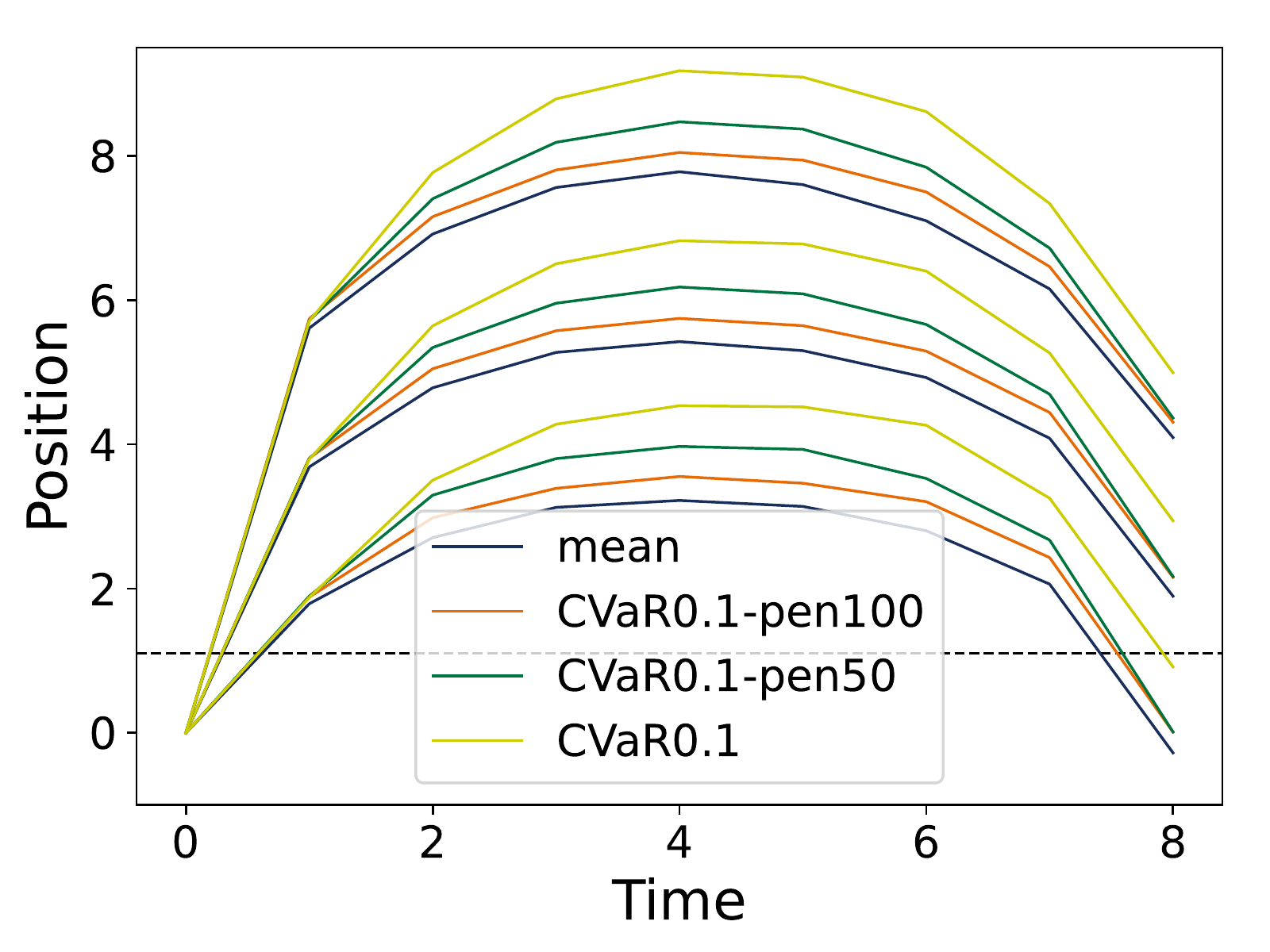}
      \caption{Estimated $10\%$, $50\%$ and $90\%$ quantiles of the region visited by the rover when following the learned policy over 30,000 \episodenames{}}
      \label{fig:preferred-path-cliff}
    \end{minipage}
    \hfill
    \begin{minipage}{0.48\textwidth}
      \centering
      \includegraphics[width=0.95\textwidth]{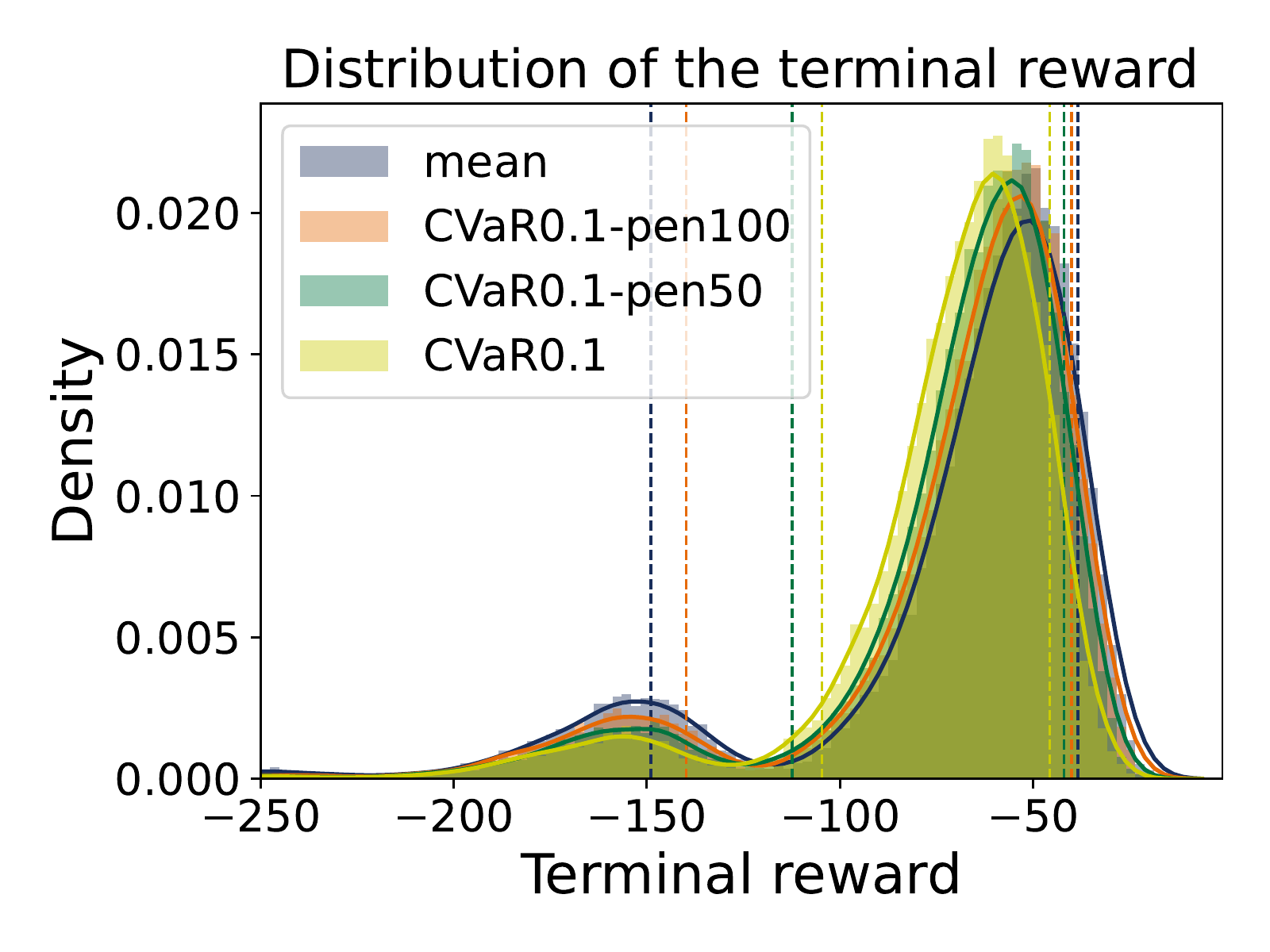}
      \caption{Estimated distribution of the terminal reward when following the learned policy in the cliff walking example for several risk measures over 30,000 \episodenames{}.}
      \label{fig:final-cost-dist-cliff}
    \end{minipage}
\end{figure}

\new{Risk-sensitivity in mathematical finance applications is of fundamental importance and requires special attention, which motivates the work provided in this paper. However, the proposed methodology can be applied more generally to other fields of study, as shown in this robot control example. One may see connections between the cliff walking problem and some financial applications: for instance, suppose that an \agentname{} aims to hedge a barrier option which is knocked-out (e.g. falling into the cliff) if the price of the underlying asset (e.g. position of the rover) goes beyond a predetermined price (e.g. position of the cliff). That is another related, yet distinct financial problem our algorithm could be used on.}
%\input{section-files/results}

% conclusion
%!TEX root = ../main.tex

% --------------------------------------------------------------
%                         Conclusion
% --------------------------------------------------------------
\section{Discussion}
\label{sec:conclusion}

In this work, we extend risk-aware \RL{} procedures found in the literature by developing a methodology for a wide class of sequential decision making problems.
Our approach broadens to the whole class of \emph{dynamic convex} risk measures and allows \emph{non-stationary behaviors} from the \agentname{}.
Our work performs well on three benchmark \RL{} and financial mathematical problems, which opens doors for several applications in credit risk, portfolio optimization and optimal control among others.

There are still some future directions that we could explore to
%solve potential shortcomings that remain.
provide even more flexibility to our proposed approach.
For instance, we could devise a \emph{computationally efficient} methodology for large-scale problems when
%outer and inner
simulations can be costly, propose a \emph{robust framework} for time-consistent dynamic risk measures,
%We could also propose a \emph{multi-agent system framework} to solve this class of problems,
and develop a generalization for deterministic \stratnames{} with dynamic risk measures in a similar manner to \emph{deep deterministic policy gradient}.
%We could apply our methodology on various financial mathematical problems, for instance in credit risk, portfolio optimization and systemic risk among others.
More recently, efforts were put to verify the theoretical \emph{convergence of policy gradient} methods.
Indeed, \cite{agarwal2021theory} show that the policy gradient approach with a risk-neutral expectation has global convergence guarantees.
On the other hand, it may not converge to a global optimum when the objective function is a dynamic risk measure \citep{huang2021convergence}.
%In fact, this suboptimality is exacerbated with more sensible risk measures, such as the conditional value-at-risk with a small threshold.
It thus remains an open challenge to prove that actor-critic algorithms with dynamic convex risk measures converge to an optimal \stratname{} when both the value function and \stratname{} are characterized by \ANN{}s.
%The paper from \citep{zhong2020risk} solves that issue when optimizing the expectation subject to a variance constraint.
%That encouraging result motivates work in that direction to find theoretical results for the wider class of risk-aware \RL{} problems.

% references
%\newpage
\bibliographystyle{abbrvnat}
\bibliography{bib-files/references}

\begin{thebibliography}{67}
\providecommand{\natexlab}[1]{#1}
\providecommand{\url}[1]{\texttt{#1}}
\expandafter\ifx\csname urlstyle\endcsname\relax
  \providecommand{\doi}[1]{doi: #1}\else
  \providecommand{\doi}{doi: \begingroup \urlstyle{rm}\Url}\fi

\bibitem[Acciaio and Penner(2011)]{acciaio2011dynamic}
B.~Acciaio and I.~Penner.
\newblock Dynamic risk measures.
\newblock In \emph{Advanced Mathematical Methods for Finance}, pages 1--34.
  Springer, 2011.

\bibitem[Agarwal et~al.(2021)Agarwal, Kakade, Lee, and
  Mahajan]{agarwal2021theory}
A.~Agarwal, S.~M. Kakade, J.~D. Lee, and G.~Mahajan.
\newblock On the theory of policy gradient methods: Optimality, approximation,
  and distribution shift.
\newblock \emph{Journal of Machine Learning Research}, 22\penalty0
  (98):\penalty0 1--76, 2021.

\bibitem[Ahmadi et~al.(2020)Ahmadi, Rosolia, Ingham, Murray, and
  Ames]{ahmadi2020constrained}
M.~Ahmadi, U.~Rosolia, M.~D. Ingham, R.~M. Murray, and A.~D. Ames.
\newblock Constrained risk-averse {Markov} decision processes.
\newblock \emph{arXiv preprint arXiv:2012.02423}, 2020.

\bibitem[Al-Aradi et~al.(2018)Al-Aradi, Correia, Naiff, Jardim, and
  Saporito]{al2018solving}
A.~Al-Aradi, A.~Correia, D.~Naiff, G.~Jardim, and Y.~Saporito.
\newblock Solving nonlinear and high-dimensional partial differential equations
  via deep learning.
\newblock \emph{arXiv preprint arXiv:1811.08782}, 2018.

\bibitem[Artzner et~al.(1999)Artzner, Delbaen, Eber, and
  Heath]{artzner1999coherent}
P.~Artzner, F.~Delbaen, J.-M. Eber, and D.~Heath.
\newblock Coherent measures of risk.
\newblock \emph{Mathematical Finance}, 9\penalty0 (3):\penalty0 203--228, 1999.

\bibitem[B{\"a}uerle and
  Glauner(2020{\natexlab{a}})]{bauerle2020distributionally}
N.~B{\"a}uerle and A.~Glauner.
\newblock Distributionally robust {Markov} decision processes and their
  connection to risk measures.
\newblock \emph{arXiv preprint arXiv:2007.13103}, 2020{\natexlab{a}}.

\bibitem[B{\"a}uerle and Glauner(2020{\natexlab{b}})]{bauerle2020minimizing}
N.~B{\"a}uerle and A.~Glauner.
\newblock Minimizing spectral risk measures applied to markov decision
  processes.
\newblock \emph{arXiv preprint arXiv:2012.04521}, 2020{\natexlab{b}}.

\bibitem[B{\"a}uerle and Glauner(2021)]{bauerle2021markov}
N.~B{\"a}uerle and A.~Glauner.
\newblock Markov decision processes with recursive risk measures.
\newblock \emph{European Journal of Operational Research}, 2021.

\bibitem[Bellemare et~al.(2017)Bellemare, Dabney, and
  Munos]{bellemare2017distributional}
M.~G. Bellemare, W.~Dabney, and R.~Munos.
\newblock A distributional perspective on reinforcement learning.
\newblock In \emph{International Conference on Machine Learning}, pages
  449--458. PMLR, 2017.

\bibitem[Bielecki et~al.(2016)Bielecki, Cialenco, Drapeau, and
  Karliczek]{bielecki2016dynamic}
T.~R. Bielecki, I.~Cialenco, S.~Drapeau, and M.~Karliczek.
\newblock Dynamic assessment indices.
\newblock \emph{Stochastics}, 88\penalty0 (1):\penalty0 1--44, 2016.

\bibitem[Campbell et~al.(2021)Campbell, Chen, Shrivats, and
  Jaimungal]{campbell2021deep}
S.~Campbell, Y.~Chen, A.~Shrivats, and S.~Jaimungal.
\newblock Deep learning for principal-agent mean field games.
\newblock \emph{arXiv preprint arXiv:2110.01127}, 2021.

\bibitem[Carmona and Lauri{\`e}re(2021)]{carmona2021deep}
R.~Carmona and M.~Lauri{\`e}re.
\newblock Deep learning for mean field games and mean field control with
  applications to finance.
\newblock \emph{arXiv preprint arXiv:2107.04568}, 2021.

\bibitem[Casgrain et~al.(2019)Casgrain, Ning, and Jaimungal]{casgrain2019deep}
P.~Casgrain, B.~Ning, and S.~Jaimungal.
\newblock Deep {Q}-learning for {Nash} equilibria: {Nash}-{DQN}.
\newblock \emph{arXiv preprint arXiv:1904.10554}, 2019.

\bibitem[Cheridito and Stadje(2009)]{cheridito2009time}
P.~Cheridito and M.~Stadje.
\newblock Time-inconsistency of var and time-consistent alternatives.
\newblock \emph{Finance Research Letters}, 6\penalty0 (1):\penalty0 40--46,
  2009.

\bibitem[Chow et~al.(2017)Chow, Ghavamzadeh, Janson, and Pavone]{chow2017risk}
Y.~Chow, M.~Ghavamzadeh, L.~Janson, and M.~Pavone.
\newblock Risk-constrained reinforcement learning with percentile risk
  criteria.
\newblock \emph{Journal of Machine Learning Research}, 18\penalty0
  (1):\penalty0 6070--6120, 2017.

\bibitem[Chu and Zhang(2014)]{chu2014markov}
S.~Chu and Y.~Zhang.
\newblock Markov decision processes with iterated coherent risk measures.
\newblock \emph{International Journal of Control}, 87\penalty0 (11):\penalty0
  2286--2293, 2014.

\bibitem[Cuchiero et~al.(2020)Cuchiero, Khosrawi, and
  Teichmann]{cuchiero2020generative}
C.~Cuchiero, W.~Khosrawi, and J.~Teichmann.
\newblock A generative adversarial network approach to calibration of local
  stochastic volatility models.
\newblock \emph{Risks}, 8\penalty0 (4):\penalty0 101, 2020.

\bibitem[Cybenko(1989)]{cybenko1989approximation}
G.~Cybenko.
\newblock Approximation by superpositions of a sigmoidal function.
\newblock \emph{Mathematics of Control, Signals and Systems}, 2\penalty0
  (4):\penalty0 303--314, 1989.

\bibitem[Delage and Mannor(2010)]{delage2010percentile}
E.~Delage and S.~Mannor.
\newblock Percentile optimization for {Markov} decision processes with
  parameter uncertainty.
\newblock \emph{Operations Research}, 58\penalty0 (1):\penalty0 203--213, 2010.

\bibitem[Detlefsen and Scandolo(2005)]{detlefsen2005conditional}
K.~Detlefsen and G.~Scandolo.
\newblock Conditional and dynamic convex risk measures.
\newblock \emph{Finance and Stochastics}, 9\penalty0 (4):\penalty0 539--561,
  2005.

\bibitem[Dhaene et~al.(2006)Dhaene, Vanduffel, Goovaerts, Kaas, Tang, and
  Vyncke]{dhaene2006risk}
J.~Dhaene, S.~Vanduffel, M.~J. Goovaerts, R.~Kaas, Q.~Tang, and D.~Vyncke.
\newblock Risk measures and comonotonicity: a review.
\newblock \emph{Stochastic Models}, 22\penalty0 (4):\penalty0 573--606, 2006.

\bibitem[Di~Castro et~al.(2019)Di~Castro, Oren, and Mannor]{di2019practical}
D.~Di~Castro, J.~Oren, and S.~Mannor.
\newblock Practical risk measures in reinforcement learning.
\newblock \emph{arXiv preprint arXiv:1908.08379}, 2019.

\bibitem[Drapeau et~al.(2016)Drapeau, Kupper, Gianin, and
  Tangpi]{drapeau2016dual}
S.~Drapeau, M.~Kupper, E.~R. Gianin, and L.~Tangpi.
\newblock Dual representation of minimal supersolutions of convex {BSDEs}.
\newblock In \emph{Annales de l'Institut Henri Poincar{\'e}, Probabilit{\'e}s
  et Statistiques}, volume~52, pages 868--887. Institut Henri Poincar{\'e},
  2016.

\bibitem[F{\"o}llmer and Schied(2002)]{follmer2002convex}
H.~F{\"o}llmer and A.~Schied.
\newblock Convex measures of risk and trading constraints.
\newblock \emph{Finance and Stochastics}, 6\penalty0 (4):\penalty0 429--447,
  2002.

\bibitem[F{\"o}llmer et~al.(2004)F{\"o}llmer, Schied, and
  Lyons]{follmer2004stochastic}
H.~F{\"o}llmer, A.~Schied, and T.~J. Lyons.
\newblock Stochastic finance. an introduction in discrete time.
\newblock \emph{The Mathematical Intelligencer}, 26\penalty0 (4):\penalty0
  67--68, 2004.

\bibitem[Frittelli and Gianin(2004)]{frittelli2004dynamic}
M.~Frittelli and E.~R. Gianin.
\newblock Dynamic convex risk measures.
\newblock \emph{Risk measures for the 21st century}, pages 227--248, 2004.

\bibitem[Galichet et~al.(2013)Galichet, Sebag, and
  Teytaud]{galichet2013exploration}
N.~Galichet, M.~Sebag, and O.~Teytaud.
\newblock Exploration vs exploitation vs safety: Risk-aware multi-armed
  bandits.
\newblock In \emph{Asian Conference on Machine Learning}, pages 245--260. PMLR,
  2013.

\bibitem[Goodfellow et~al.(2016)Goodfellow, Bengio, Courville, and
  Bengio]{goodfellow2016deep}
I.~Goodfellow, Y.~Bengio, A.~Courville, and Y.~Bengio.
\newblock \emph{Deep learning}, volume~1.
\newblock MIT press Cambridge, 2016.

\bibitem[Grondman et~al.(2012)Grondman, Busoniu, Lopes, and
  Babuska]{grondman2012survey}
I.~Grondman, L.~Busoniu, G.~A. Lopes, and R.~Babuska.
\newblock A survey of actor-critic reinforcement learning: Standard and natural
  policy gradients.
\newblock \emph{IEEE Transactions on Systems, Man, and Cybernetics, Part C
  (Applications and Reviews)}, 42\penalty0 (6):\penalty0 1291--1307, 2012.

\bibitem[Hambly et~al.(2021)Hambly, Xu, and Yang]{hambly2021policy}
B.~Hambly, R.~Xu, and H.~Yang.
\newblock Policy gradient methods for the noisy linear quadratic regulator over
  a finite horizon.
\newblock \emph{SIAM Journal on Control and Optimization}, 59\penalty0
  (5):\penalty0 3359--3391, 2021.

\bibitem[Han et~al.(2018)Han, Jentzen, and Weinan]{han2018solving}
J.~Han, A.~Jentzen, and E.~Weinan.
\newblock Solving high-dimensional partial differential equations using deep
  learning.
\newblock \emph{Proceedings of the National Academy of Sciences}, 115\penalty0
  (34):\penalty0 8505--8510, 2018.

\bibitem[Heston(1993)]{heston1993closed}
S.~L. Heston.
\newblock A closed-form solution for options with stochastic volatility with
  applications to bond and currency options.
\newblock \emph{The Review of Financial Studies}, 6\penalty0 (2):\penalty0
  327--343, 1993.

\bibitem[Hornik(1991)]{hornik1991approximation}
K.~Hornik.
\newblock Approximation capabilities of multilayer feedforward networks.
\newblock \emph{Neural Networks}, 4\penalty0 (2):\penalty0 251--257, 1991.

\bibitem[Horvath et~al.(2021)Horvath, Muguruza, and Tomas]{horvath2021deep}
B.~Horvath, A.~Muguruza, and M.~Tomas.
\newblock Deep learning volatility: a deep neural network perspective on
  pricing and calibration in (rough) volatility models.
\newblock \emph{Quantitative Finance}, 21\penalty0 (1):\penalty0 11--27, 2021.

\bibitem[Hu(2019)]{hu2019deep}
R.~Hu.
\newblock Deep fictitious play for stochastic differential games.
\newblock \emph{arXiv preprint arXiv:1903.09376}, 2019.

\bibitem[Huang et~al.(2021)Huang, Leqi, Lipton, and
  Azizzadenesheli]{huang2021convergence}
A.~Huang, L.~Leqi, Z.~C. Lipton, and K.~Azizzadenesheli.
\newblock On the convergence and optimality of policy gradient for {Markov}
  coherent risk.
\newblock \emph{arXiv preprint arXiv:2103.02827}, 2021.

\bibitem[Jaimungal(2022)]{jaimungal2022reinforcement}
S.~Jaimungal.
\newblock Reinforcement learning and stochastic optimisation.
\newblock \emph{Finance and Stochastics}, 26\penalty0 (1):\penalty0 103--129,
  2022.

\bibitem[Kalogerias et~al.(2020)Kalogerias, Chamon, Pappas, and
  Ribeiro]{kalogerias2020better}
D.~S. Kalogerias, L.~F. Chamon, G.~J. Pappas, and A.~Ribeiro.
\newblock Better safe than sorry: Risk-aware nonlinear bayesian estimation.
\newblock In \emph{ICASSP 2020-2020 IEEE International Conference on Acoustics,
  Speech and Signal Processing (ICASSP)}, pages 5480--5484. IEEE, 2020.

\bibitem[Kingma and Ba(2014)]{kingma2014adam}
D.~P. Kingma and J.~Ba.
\newblock Adam: A method for stochastic optimization.
\newblock \emph{arXiv preprint arXiv:1412.6980}, 2014.

\bibitem[Konda and Tsitsiklis(2000)]{konda2000actor}
V.~R. Konda and J.~N. Tsitsiklis.
\newblock Actor-critic algorithms.
\newblock In \emph{Advances in Neural Information Processing Systems}, pages
  1008--1014. Citeseer, 2000.

\bibitem[Kose and Ruszczynski(2021)]{kose2021risk}
U.~Kose and A.~Ruszczynski.
\newblock Risk-averse learning by temporal difference methods with {Markov}
  risk measures.
\newblock \emph{Journal of Machine Learning Research}, 22:\penalty0 38--1,
  2021.

\bibitem[LeCun et~al.(2015)LeCun, Bengio, and Hinton]{lecun2015deep}
Y.~LeCun, Y.~Bengio, and G.~Hinton.
\newblock Deep learning.
\newblock \emph{Nature}, 521\penalty0 (7553):\penalty0 436--444, 2015.

\bibitem[Leshno et~al.(1993)Leshno, Lin, Pinkus, and
  Schocken]{leshno1993multilayer}
M.~Leshno, V.~Y. Lin, A.~Pinkus, and S.~Schocken.
\newblock Multilayer feedforward networks with a nonpolynomial activation
  function can approximate any function.
\newblock \emph{Neural Networks}, 6\penalty0 (6):\penalty0 861--867, 1993.

\bibitem[Milgrom and Segal(2002)]{milgrom2002envelope}
P.~Milgrom and I.~Segal.
\newblock Envelope theorems for arbitrary choice sets.
\newblock \emph{Econometrica}, 70\penalty0 (2):\penalty0 583--601, 2002.

\bibitem[Mil’shtejn(1975)]{mil1975approximate}
G.~Mil’shtejn.
\newblock Approximate integration of stochastic differential equations.
\newblock \emph{Theory of Probability \& Its Applications}, 19\penalty0
  (3):\penalty0 557--562, 1975.

\bibitem[Nass et~al.(2019)Nass, Belousov, and Peters]{nass2019entropic}
D.~Nass, B.~Belousov, and J.~Peters.
\newblock Entropic risk measure in policy search.
\newblock \emph{arXiv preprint arXiv:1906.09090}, 2019.

\bibitem[Ning et~al.(2021)Ning, Jaimungal, Zhang, and
  Bergeron]{ning2021arbitrage}
B.~Ning, S.~Jaimungal, X.~Zhang, and M.~Bergeron.
\newblock Arbitrage-free implied volatility surface generation with variational
  autoencoders.
\newblock \emph{arXiv preprint arXiv:2108.04941}, 2021.

\bibitem[Osogami(2012)]{osogami2012robustness}
T.~Osogami.
\newblock Robustness and risk-sensitivity in {Markov} decision processes.
\newblock \emph{Advances in Neural Information Processing Systems},
  25:\penalty0 233--241, 2012.

\bibitem[Peng(1997)]{peng1997backward}
S.~Peng.
\newblock Backward {SDE} and related g-expectation.
\newblock \emph{Pitman Research Notes in Mathematics Series}, pages 141--160,
  1997.

\bibitem[Petrik and Subramanian(2012)]{petrik2012approximate}
M.~Petrik and D.~Subramanian.
\newblock An approximate solution method for large risk-averse {Markov}
  decision processes.
\newblock \emph{arXiv preprint arXiv:1210.4901}, 2012.

\bibitem[Pinkus(1999)]{pinkus1999approximation}
A.~Pinkus.
\newblock Approximation theory of the {MLP} model in neural networks.
\newblock \emph{Acta Numerica}, 8:\penalty0 143--195, 1999.

\bibitem[Prashanth and Ghavamzadeh(2013)]{prashanth2013actor}
L.~Prashanth and M.~Ghavamzadeh.
\newblock Actor-critic algorithms for risk-sensitive {MDPs}.
\newblock 2013.

\bibitem[Rahimian and Mehrotra(2019)]{rahimian2019distributionally}
H.~Rahimian and S.~Mehrotra.
\newblock Distributionally robust optimization: A review.
\newblock \emph{arXiv preprint arXiv:1908.05659}, 2019.

\bibitem[Riedel(2004)]{riedel2004dynamic}
F.~Riedel.
\newblock Dynamic coherent risk measures.
\newblock \emph{Stochastic Processes and their Applications}, 112\penalty0
  (2):\penalty0 185--200, 2004.

\bibitem[Rockafellar and Uryasev(2013)]{rockafellar2013fundamental}
R.~T. Rockafellar and S.~Uryasev.
\newblock The fundamental risk quadrangle in risk management, optimization and
  statistical estimation.
\newblock \emph{Surveys in Operations Research and Management Science},
  18\penalty0 (1-2):\penalty0 33--53, 2013.

\bibitem[Rockafellar et~al.(2000)Rockafellar, Uryasev,
  et~al.]{rockafellar2000optimization}
R.~T. Rockafellar, S.~Uryasev, et~al.
\newblock Optimization of conditional value-at-risk.
\newblock \emph{Journal of Risk}, 2:\penalty0 21--42, 2000.

\bibitem[Ruszczy{\'n}ski(2010)]{ruszczynski2010risk}
A.~Ruszczy{\'n}ski.
\newblock Risk-averse dynamic programming for {Markov} decision processes.
\newblock \emph{Mathematical Programming}, 125\penalty0 (2):\penalty0 235--261,
  2010.

\bibitem[Shapiro et~al.(2014)Shapiro, Dentcheva, and
  Ruszczy{\'n}ski]{shapiro2014lectures}
A.~Shapiro, D.~Dentcheva, and A.~Ruszczy{\'n}ski.
\newblock \emph{Lectures on Stochastic Programming: Modeling and Theory}.
\newblock SIAM, 2014.

\bibitem[Shen et~al.(2014)Shen, Tobia, Sommer, and Obermayer]{shen2014risk}
Y.~Shen, M.~J. Tobia, T.~Sommer, and K.~Obermayer.
\newblock Risk-sensitive reinforcement learning.
\newblock \emph{Neural Computation}, 26\penalty0 (7):\penalty0 1298--1328,
  2014.

\bibitem[Silver et~al.(2016)Silver, Huang, Maddison, Guez, Sifre, Van
  Den~Driessche, Schrittwieser, Antonoglou, Panneershelvam, Lanctot,
  et~al.]{silver2016mastering}
D.~Silver, A.~Huang, C.~J. Maddison, A.~Guez, L.~Sifre, G.~Van Den~Driessche,
  J.~Schrittwieser, I.~Antonoglou, V.~Panneershelvam, M.~Lanctot, et~al.
\newblock Mastering the game of go with deep neural networks and tree search.
\newblock \emph{Nature}, 529\penalty0 (7587):\penalty0 484--489, 2016.

\bibitem[Sutton and Barto(2018)]{sutton2018reinforcement}
R.~S. Sutton and A.~G. Barto.
\newblock \emph{Reinforcement Learning: An Introduction}.
\newblock MIT press, 2018.

\bibitem[Sutton et~al.(2000)Sutton, McAllester, Singh, and
  Mansour]{sutton2000policy}
R.~S. Sutton, D.~A. McAllester, S.~P. Singh, and Y.~Mansour.
\newblock Policy gradient methods for reinforcement learning with function
  approximation.
\newblock In \emph{Advances in Neural Information Processing Systems}, pages
  1057--1063, 2000.

\bibitem[Tamar et~al.(2015)Tamar, Chow, Ghavamzadeh, and
  Mannor]{tamar2015policy}
A.~Tamar, Y.~Chow, M.~Ghavamzadeh, and S.~Mannor.
\newblock Policy gradient for coherent risk measures.
\newblock \emph{Advances in Neural Information Processing Systems},
  28:\penalty0 1468--1476, 2015.

\bibitem[Tamar et~al.(2016)Tamar, Chow, Ghavamzadeh, and
  Mannor]{tamar2016sequential}
A.~Tamar, Y.~Chow, M.~Ghavamzadeh, and S.~Mannor.
\newblock Sequential decision making with coherent risk.
\newblock \emph{IEEE Transactions on Automatic Control}, 62\penalty0
  (7):\penalty0 3323--3338, 2016.

\bibitem[Weber(2006)]{weber2006distribution}
S.~Weber.
\newblock Distribution-invariant risk measures, information, and dynamic
  consistency.
\newblock \emph{Mathematical Finance: An International Journal of Mathematics,
  Statistics and Financial Economics}, 16\penalty0 (2):\penalty0 419--441,
  2006.

\bibitem[Weinan et~al.(2017)Weinan, Han, and Jentzen]{weinan2017deep}
E.~Weinan, J.~Han, and A.~Jentzen.
\newblock Deep learning-based numerical methods for high-dimensional parabolic
  partial differential equations and backward stochastic differential
  equations.
\newblock \emph{Communications in Mathematics and Statistics}, 5\penalty0
  (4):\penalty0 349--380, 2017.

\bibitem[Yu et~al.(2018)Yu, Haskell, and Xu]{yu2018approximate}
P.~Yu, W.~B. Haskell, and H.~Xu.
\newblock Approximate value iteration for risk-aware {Markov} decision
  processes.
\newblock \emph{IEEE Transactions on Automatic Control}, 63\penalty0
  (9):\penalty0 3135--3142, 2018.

\end{thebibliography}

% appendix
%\newpage
\appendix
%!TEX root = ../main.tex

% --------------------------------------------------------------
%                         Implementation
% --------------------------------------------------------------
\section{Implementation}
\label{sec:implementation}

In this section, we expand on the experimental setup by giving additional details on the implementation of algorithms given in \cref{sec:algorithm}.
All the code written in Python is available at \url{https://github.com/acoache/RL-DynamicConvexRisk}.

The structure of the Python files is similar between all sets of experiments.
The \texttt{envs.py} file contains the environment class for \RL{} problem, as well as functions to interact with it.
It has both the \href{https://pytorch.org/}{PyTorch} and \href{https://numpy.org/}{NumPy} versions of the simulation engine. 
The \texttt{risk\_measure.py} file has the class that creates an instance of a risk measure, with functions to compute the risk and calculate the gradient.
Risk measures currently implemented are the expectation, the \CVaR{}, the penalized \CVaR{}, the mean-semideviation and a linear combination between the mean and \CVaR{}.
There is also a \texttt{utils.py} file which contains some useful functions and variables, such as a function to create new directories and colors for the visualizations.

Models are regrouped under the \texttt{models.py} file with classes to build \ANN{} structures using the \href{https://pytorch.org/}{PyTorch} library.
In our experiments presented in \cref{sec:experiments}, the value function $\valuefunc^{\valueparams}$ has four layers of $16$ hidden nodes each with SiLU activation functions, and no activation function for its output layer.
The \ANN{} for the \stratname{} $\policy^{\policyparams}$ is composed of five layers with $16$ hidden nodes each with SiLU activation function, but with an output layer specific to the application -- e.g. linear transformation of a sigmoid activation function that maps to $(-\trademax, \trademax)$ for the set of experiments in \cref{ssec:trading-problem}.
The learning rates for $\valuefunc^{\valueparams}$ and $\policy^{\policyparams}$ are of the order of respectively $1 \times 10^{-3}$ and $5 \times 10^{-4}$.
Both $\valuefunc^{\valueparams}$ and $\policy^{\policyparams}$ are updated with mini batches of 300 (outer) \episodenames{} and 1,000 (inner) \transitionnames{} during the training phase.\footnote{Hyperparameters depend on the specific experiment and are chosen to accelerate the learning procedure.}

The whole algorithm is wrapped into a single class named \texttt{ActorCriticPG}, where input arguments specify which problem the \agentname{} faces.
The user needs to give an environment, a (convex) risk measure, as well as two neural network structures that play the role of the value function and \agentstratname{}.
Each instance of that class has functions to select actions from the \stratname{}, whether at random or using the best behavior found thus far, and give the set of invalid actions. 
There is also a function to simulate (outer) \episodenames{} and (inner) \transitionnames{} using the simulation upon simulation approach discussed in \cref{sec:algorithm}.
\cref{algo:estimate-value} is wrapped in a function which takes as inputs the mini-batch size $\NbatchsV$, number of epochs $\NepochsV$ and characteristics of the value function neural network structure, such as the learning rate and the number of hidden nodes.
Similarly, another function implements \cref{algo:update-pi} and takes as inputs the mini-batch size $\NbatchsPI$ and number of epochs $\NepochsPI$.

The \texttt{main.py} file contains the program to run the training phase.
The first part concerns the importation of libraries and initialization of all parameters, either for the environment, neural networks or risk measure.
Some notable parameters that need to be specified by the user in the \texttt{hyperparams.py} file are the numbers of epochs, learning rates, size of the neural networks and number of \episodenames{}/\transitionnames{} among others. 
The next section is the training phase and its skeleton is given in \cref{algo:main-steps}.
It uses mostly functions from the \texttt{actor\_critic.py} file.
Finally, the models for the policy and value function are saved in a folder, along with diagnostic plots.
Since the nested simulation approach is computationally expensive, running this Python program can take up to \new{24} hours depending on the application and the hyperparameters, especially the number of \timenames{}, \new{the complexity of the convex risk measures} and the number of (inner) transitions.

This \texttt{main\_plot.py} file contains the program to run the testing phase. The first part concerns the importation of libraries and initialization of all parameters.
Note that parameters must be identical to the ones used in \texttt{main.py}.
The next section evaluates the \stratname{} found by the algorithm.
It runs several simulations using the best behavior found by the actor-critic algorithm.
Finally it outputs graphics to assess the performance of the procedure, such as the preferred action in any possible state and the estimated distribution of the total cost when following the best \stratname{}.

\end{document}